%% file: main.tex
\documentclass{article}

\usepackage{microtype}      % microtypography
\usepackage{graphicx}
\usepackage{booktabs}       % professional-quality tables

\usepackage[accepted]{icml2021}
\usepackage[utf8]{inputenc} % allow utf-8 input
\usepackage[T1]{fontenc}    % use 8-bit T1 fonts
\usepackage{url}            % simple URL typesetting
\usepackage{amsfonts}       % blackboard math symbols
\usepackage{nicefrac}       % compact symbols for 1/2, etc.

\usepackage{amsmath}
\usepackage{amsfonts}
\usepackage{mathtools}
\usepackage{amssymb}

\usepackage{xcolor}
\usepackage{comment}
\usepackage{bbm}
\usepackage{url}

\usepackage{color} % use colored text in latex
\usepackage{amsthm}
\usepackage{amsmath}
\usepackage{amssymb}
\usepackage{bbm,amsfonts}
\usepackage{dsfont}  % for mathds and nice indicator function
\usepackage{setspace}
\usepackage{nomencl}
\usepackage{amstext} % Lots of math symbols and environments
\usepackage{comment}
\usepackage[colorlinks=true,citecolor={blue}]{hyperref}

\usepackage{dcolumn}% Align table columns on decimal point
\usepackage{bm}% bold math
\usepackage{tikz}
\usepackage{pgfplots}
\usepackage{float}

\usepackage{enumitem}

\usepackage{breqn}
\usepackage{verbatim}
\usepackage{caption}
\usepackage{multirow}

\usepackage[caption=false]{subfig}

\newcommand{\expectation}{\mathbb{E}}
\newcommand{\norm}[1]{\left\lVert#1\right\rVert}

\newtheorem{theorem}{\protect\theoremname}

\newtheorem{proposition}{\protect\propositionname}
\newtheorem{lemma}{\protect\lemmaname}
\newtheorem{condition_statement}{\protect\conditionname}

\providecommand{\definitionname}{Definition}
\providecommand{\lemmaname}{Lemma}
\providecommand{\propositionname}{Proposition}
\providecommand{\theoremname}{Theorem}
\providecommand{\conditionname}{Condition}

\newcommand{\treg}{T-regime}
\newcommand{\mreg}{M-regime}

\global\long\def\argmin{\operatornamewithlimits{argmin}}
\global\long\def\argmax{\operatornamewithlimits{argmax}}
\DeclareMathOperator{\sign}{sign}
\DeclareMathOperator{\sech}{sech}
\icmltitlerunning{Learning of Ising Model Dynamics}

\begin{document}



\twocolumn[

\icmltitle{Exponential Reduction in Sample Complexity with \\ Learning of Ising Model Dynamics}% Force line breaks with \\
% Exponential gain in sample complexity through adaptive learning of Ising models from time-correlated samples
% 

\icmlsetsymbol{equal}{*}

\begin{icmlauthorlist}
\icmlauthor{Arkopal Dutt}{mit}
\icmlauthor{Andrey Y. Lokhov}{lanl}
\icmlauthor{Marc Vuffray}{lanl}
\icmlauthor{Sidhant Misra}{lanl}
\end{icmlauthorlist}

\icmlaffiliation{mit}{Research Laboratory of Electronics, Massachusetts Institute of Technology, Cambridge, MA, USA}
\icmlaffiliation{lanl}{Theoretical Division, Los Alamos National Laboratory, Los Alamos, NM, USA}
\icmlcorrespondingauthor{Andrey Y. Lokhov}{lokhov@lanl.gov}
\icmlkeywords{Machine Learning, ICML}

\vskip 0.3in

]



\printAffiliationsAndNotice{}  % leave blank if no need to mention equal contribution

\begin{abstract}
    The usual setting for learning the structure and parameters of a graphical model assumes the availability of independent samples produced from the corresponding multivariate probability distribution. However, for many models the mixing time of the respective Markov chain can be very large and i.i.d. samples may not be obtained. We study the problem of reconstructing binary graphical models from correlated samples produced by a dynamical process, which is natural in many applications. We analyze the sample complexity of two estimators that are based on the interaction screening objective and the conditional likelihood loss. We observe that for samples coming from a dynamical process far from equilibrium, the sample complexity  reduces exponentially compared to a dynamical process that mixes quickly.
\end{abstract}


\vspace{-0.2in}
%%%%%%%%%%%%%%%%%%%%%%%%%%%%
\section{Introduction}
A graphical model (GM) is a convenient description of a probabilistic distribution which highlights the structure of the conditional dependencies existing between a set of random variables. We focus our attention on GMs that can be expressed as elements of an exponential family and are naturally associated with a graph that captures the underlying structure of the conditional dependencies. These GMs, sometimes referred as positive Markov random fields or Boltzmann distributions, are ubiquitous tools used to describe behaviors of random systems across a broad range of sciences such as physics \cite{chaves2015information}, biology \cite{Jansen2003}, medicine \cite{constantinou2016complex}, data mining \cite{Buczak2016datamining} and computer vision \cite{WANG20131610}. The expression of GMs can sometimes be deduced from first principles, but often it has to be learned from observed data accessible through measurements and experiments. As these samples are time-consuming or costly to produce, it is not surprising that \emph{efficient} GM learning methods play an important role in various fields such as in the study of gene expression \cite{MarbachCostelloKuffnerEtAl2012}, protein interactions \cite{MorcosPagnaniLuntEtAl2011}, neuroscience \cite{SchneidmanBerrySegevEtAl2006}, image processing \cite{RothBlack2005}, sociology \cite{EaglePentlandLazer2009} and even grid science \cite{HeZhang2011}.

The practical problem of learning a GM from observed data has a long-standing and rich history that can be traced back to the seminal work of Chow-Liu \cite{chowliu68}. However, it wasn't until recently and after further developments that a body of work showed one can efficiently reconstruct GMs from independent and identically distributed (i.i.d.) samples \cite{Bresler2015, vuffray2016interaction, Ankur2017nips, Klivans2017, lokhov2018optimal, vuffray2019efficient}. In these papers, two methods stand out for being essentially optimal in the number of samples that they require \cite{lokhov2018optimal}. These methods are named Regularized Interaction Screening Estimator (RISE) and Regularized Pseudo-likelihood Estimator (RPLE) and both rely on the minimization of a convex loss function. The sample complexity of these estimators scales exponentially with a quantity named $\beta$ that represents the maximum magnitude of the parameters in the GM. This exponential dependence in $\beta$ is a fundamental limit of GM learning from i.i.d. samples \cite{santhanam2012information} with heavy practical consequences as it restricts the possibility of learning GMs when data is scarce. However, the assumption of having access to independent samples is a modeling hypothesis that is convenient in many ways, but for which we can challenge the limits of its validity as it is known that sampling from arbitrary GMs is an NP-hard task. In most of the experimental settings mentioned earlier, the samples are actually obtained from a dynamic process whose stationary distribution is captured by a GM. Even the state of the art sampling techniques for GMs are implemented through Markov Chain Monte-Carlo (MCMC) dynamics \cite{levin2017markov,Hassani2015sampling}. It is therefore natural to wonder if learning a graphical model from a dynamical process can be beneficial from a sample complexity standpoint.

Surprisingly, GM learning from dynamics has been rigorously studied very little with the notable exception of the paper of Bresler, Gamarnik and Shah \cite{bresler2017learning}. In an attempt to demonstrate that learning GMs from non-i.i.d. samples can be tractable, a question that was still widely debated at the time of the paper's initial release, Bresler, Gamarnik and Shah proved that one can efficiently learn GMs using samples coming from Glauber dynamics \cite{glauber1963time}, an iconic MCMC sampling dynamics. This result was regrettably overshadowed by the progress made in the following years in GM learning and their algorithm suffers from an impractical scaling much worse compared to what one could obtain with RISE or RPLE in the i.i.d. sample setting. The question of whether correlated samples from dynamics can improve the sample complexity of GM learning remains unanswered.

A hint on the fact that such a reduction in sample complexity is possible is provided by a number of empirical studies in the statistical physics literature that considered reconstruction using mean-field methods \cite{roudi2011mean, mezard2011exact, zeng2011network, zhang2012inference, bachschmid2015learning} and using pseudo-likelihood \cite{besag1975statistical} based estimators \cite{zeng2013maximum, decelle2015inference, decelle2016data} in various settings, although most of these studies focus on a simpler setting of asymmetric couplings known as kinetic Ising model that does not contain the GM as its equilibrium state. We do not consider mean-field based methods here because these methods are not exact and typically only work for high-temperature (weakly coupled) models, see \cite{lokhov2018optimal} for an extensive discussion on the value of exact algorithms. Existing studies of pseudo-likelihood based estimators have been mostly conducted in a setting of reconstruction of single instances, and with a focus on parameter estimation (instead of structure learning for which the sample complexity bounds are known); hence, it is hard to extract the sample complexity scalings with model parameters such as $\beta$ from these works. Still, single-instance reconstruction results indicate that in practice the number of samples required for an accurate model learning in the dynamic case seems to be significantly smaller compared to the i.i.d. learning setting. 

In this work, we quantify through a carefully designed set of experiments and a rigorous mathematical analysis the reduction in sample complexity that one can achieve using samples from Glauber dynamics. We focus our attention on Ising models, the celebrated class of pairwise and binary GMs for which information-theoretic lower-bounds on sample complexity exist both for i.i.d. samples \cite{santhanam2012information} and samples coming from Glauber dynamics \cite{bresler2017learning}. We propose an adaptation of the efficient learning algorithms RISE and RPLE for learning GMs with dynamical samples; Interaction Screening method has never been previously considered for learning in the dynamic setting. We extract the $\beta$ scaling of the sample complexity for different instances of Ising models in two different dynamical regimes. The first, denoted as \treg{}, consists in learning an Ising model from a single Glauber dynamic trajectory that mixes quickly toward its stationary distribution. The second, referred to as \mreg{}, consists in learning an Ising model from a series of one step evolutions of the Glauber dynamics from an initial distribution thus mimicking the trajectory of a system far from its mixed state. A similar setting of learning from a number of short trajectories starting with uniformly sampled configurations instead of one long trajectory has been considered in \cite{decelle2016data}. We find that the $\beta$ scaling in the \treg{} is similar to the one obtained from learning GMs with i.i.d. samples, an expected result since the Glauber dynamics produces i.i.d. samples once it has mixed. However, our main finding is that in the \mreg{} the $\beta$ scaling depends crucially on the initial distribution, and for dynamics far from equilibrium we achieve a $\beta$ exponent scaling up to ten times better than in the i.i.d. case. This exponential improvement in the sample complexity concretely translates into a reduction in sample requirements by a factor $10^4 - 10^5$ in typical regimes where variables of the GMs display non-trivial correlations. Our results also have a deep theoretical implication as we  show that samples acquired far from the equilibrium carry more information about the structure of the problem. Based on this intuition, we design an active learning algorithm that modifies the trajectory of the dynamics on the fly to optimize the sample complexity of the learning task.

The paper is organized as follows. In Sec.~\ref{sec:problem_statement}, we define the problem of learning an Ising model from Glauber dynamics and describe two different regimes under which learning can take place. In Sec.~\ref{sec:learning_algo}, we present our learning algorithms and a theoretical analysis of their scaling properties. Additionally, we assess their performance experimentally on a variety of Ising models of different topologies and interaction strengths. In Sec.~\ref{sec:applications}, we illustrate a real world application of our algorithms and present how active learning can be used to gain further advantage in learning from dynamics. The conclusion can be found in Sec.~\ref{sec:conclusions}.

\section{Problem statement} \label{sec:problem_statement}
% What is the graphical model we are interested in?
\subsection{Ising model}
Consider the Ising model on a graph $G=(V,E)$ with $n$ nodes where $V=[n]$ is the set of nodes and $E\subset V \times V$ is the set of undirected edges. Each node $i \in V$ is associated with a spin which we will denote by $\sigma_i$ and is a binary random variable taking values in $\{-1,+1\}$. The neighborhood of a node $i$ is denoted by $\partial i = \{j \in V \mid (i,j) \in E\}$.  The probability measure of a particular configuration of spins $\underline{\sigma} \in \{-1,+1\}^n$ is given by the Gibbs distribution
\vspace{-0.1in}
\begin{equation}
    p(\underline{\sigma}) = \frac{1}{Z} \exp \left( \sum_{(i,j) \in E} J^\star_{ij} \sigma_i \sigma_j + \sum_{i \in V} H^\star_i \sigma_i \right),
    \label{eq:gibbs_distribution}
\end{equation}
where $\underline{J}^\star = \{J^\star_{ij}\}_{(i,j) \in E}$ is the vector of non-zero interactions associated with each edge, and $\underline{H}^\star = \{H^\star_i\}_{i \in V}$ is the vector of magnetic fields associated with each node. The normalization factor $Z = \sum_{\underline{\sigma}} \exp \left( \sum_{(i,j) \in E} J^\star_{ij} \sigma_i \sigma_j + \sum_{i \in V} H^\star_i \sigma_i \right)$ is referred to as the partition function and is in general NP-hard to compute \cite{sly2012computational}.

% What is Glauber dynamics?
\subsection{Glauber dynamics and observations}
Glauber dynamics is a reversible Markov chain that was originally introduced in \cite{glauber1963time} for Ising models and can be generalized for any Markov random field. The Glauber dynamics is specified by the update rule that determines its transition probabilities. The spin configuration at any time $t$ is denoted by $\underline{\sigma}^{t}$ with the initial configuration being $\underline{\sigma}^{0}$. At each time step $t$, a node is chosen uniformly at random. The corresponding random variable is given by $I^{t+1}$. Conditioned on $I^{t+1} = i$, the spin $\sigma_i$ is updated according to the following conditional distribution:
\begin{equation}
    p(\sigma_i^{t+1}|\underline{\sigma}^t) = \frac{\exp \left[\sigma_i^{t+1} (\sum_{j \in \partial i} J^\star_{ij} \sigma_j^t + H^\star_i) \right]}{2 \cosh \left[ \sum_{j \in \partial i} J^\star_{ij} \sigma_j^t + H^\star_i \right]}.
    \label{eq:glauber_dynamics_conditional_prob_update}
\end{equation}
The initial configuration $\underline{\sigma}^{0}$ is assumed to be drawn from some distribution $p_0(\underline{\sigma}^{0})$. Executing $m$ steps of the Glauber dynamics yields the samples $\underline{\sigma}^{1},\underline{\sigma}^{2},...,\underline{\sigma}^{m}$ and the corresponding sequence of node identities is then $I^{1}, I^{2}, ..., I^{m}$. It can be used to draw i.i.d. samples from the Gibbs distribution in \eqref{eq:gibbs_distribution} when run long enough to allow for mixing. However, for a large class of models this mixing time is exponentially high \cite{martinelli1994approach}, limiting its computational tractability. At the same time, many out-of-equilibrium natural systems such as biological neural networks naturally generate temporally correlated spike train data \cite{berry1997structure, pillow2008spatio} that is well described and is modeled by the Glauber dynamics \cite{marre2009prediction,tyrcha2013effect}. This raises the problem of learning the graphical model associated with the sequence of time-correlated samples produced by Glauber dynamics, with the goals of inferring the connectivity of the system, predicting the final state of the dynamics, or for building a reliable model that can be used to simulate and predict the dynamics starting from other configurations. 

% What are the two different regimes that can be seen from this?
\subsection{Glauber dynamics with multi-start and the model selection problem}
Suppose that the Glauber dynamics is run in batches of size $m_r$ for $r = 1,\ldots, R$ with total number of samples $\sum_{r} m_r = m$. For each $r$, the initial configuration is picked according to the probability distribution $p_0$ and a sequence of $m_r$ steps of the Glauber dynamics are executed to obtain $m_r$ samples. In this paper, we consider two extreme cases.  
The \textbf{\treg{}} corresponds to $R=1$ and $m_1 = m$ where starting from an initial configuration, one batch of size $m$ is executed to obtain $m$ samples. The \textbf{\mreg{}} corresponds to $m_r = 1 \  \forall r$ and executing one step of the Glauber dynamics $m$ times, each time starting from a new initial configuration. The two regimes are designed to emulate the behaviour of the Markov chain close to the equilibrium distribution (\treg{}), and far from the equilibrium distribution (\mreg{}).

\paragraph{Notation:} We will denote a sample by a tuple of the input spin configuration to a step of Glauber dynamics, the resulting spin configuration of Glauber dynamics and the updated node identity. The sample produced in the $(t+1)$-th step of \treg{} is given by $(\underline{\sigma}^{t}, \underline{\sigma}^{t+1}, I^{t+1})$ and the $t$-th sample produced in \mreg{} will be given by $\{ {\underline{\sigma}^0}^{(t)},{\underline{\sigma}^1}^{(t)}, {I^1}^{(t)})$. Note the difference in the superscripts. For convenience, we will denote the set $\{1,2,...,k\}=[k]$ for $k \in \mathbb{Z}^+$. For stating sample complexity results, it is convenient to define a minimum non-zero coupling  $\alpha = \min_{(i,j) \in E} |J^\star_{ij}|$, the maximum coupling strength $\beta = \max_{(i,j) \in E} |J^\star_{ij}|$, and a maximum nodal degree $d$ for the Ising model in \eqref{eq:gibbs_distribution}. % Check if these notational changes have been pushed to the appendix e.g., in description of compressed sensing

\paragraph{The dynamic model selection problem:}
Given $m$ samples of  $\{(\underline{\sigma}^{t},\underline{\sigma}^{t+1},I^{t+1})\}_{t = 0, ..., m-1}$ or $\{ {\underline{\sigma}^0}^{(t)},{\underline{\sigma}^1}^{(t)}, {I^1}^{(t)}\}_{t \in [m]}$  observed from the Glauber dynamics from either the T-regime or M-regime, the model selection problem consists of two parts:
\vspace{-0.05in}
\begin{enumerate}
    \item \emph{Parameter estimation:} Compute estimates $\hat J_{ij}$ of $J^\star_{ij}$ such that for all $1\leq i,j \leq n$, we have $\|\hat J_{ij} - J^\star_{ij}\|\leq \tilde \alpha /2$, where $\|.\|$ is the norm of interest and $\tilde \alpha$ is the required precision. 
    \item \emph{Structure reconstruction:} Compute an estimate $\hat E$ of $E$ such that the probability of perfect reconstruction satisfies $p[\hat{E} = E] \geq 1-\delta$ where $\delta$ defines the confidence.
\end{enumerate}
\vspace{-0.1in}
Most existing methods in the literature \cite{vuffray2016interaction, vuffray2019efficient, Klivans2017} use parameter estimation to perform structure reconstruction. It is evident that whenever the parameters can be estimated with precision $\tilde \alpha  = \alpha$, the structure estimated by the thresholding procedure given by
\begin{align}
    \hat E\left(\alpha \right) = \{1 \leq (i,j) \leq n \mid  |\hat{J}_{ij}| \geq \alpha/2\},
\end{align}
results in a perfect reconstruction with high probability. An information theoretic lower bound for the dynamic model selection problem from a single trajectory was derived in \cite{bresler2017learning} and is given by 
\begin{equation}
    m \geq \frac{e^{2\beta d/3}}{32 d \alpha e^{d + 3\beta + 6}} n \log n
    \label{eq:info_theoretic_sample_complexity_glauber}
\end{equation}
In comparison, the information-theoretic sample complexity for learning from i.i.d. samples \cite{santhanam2012information} scales as $\exp(\beta d)$.  It still remains unclear if either of the information theoretic lower bounds are tight. Current evidence or constructive proofs show a scaling of $\exp(4 \beta d)$ for the i.i.d. case \cite{lokhov2018optimal} and  $\exp(20\beta d)$ for the general dynamic case \cite{bresler2017learning}. 

\section{Learning Ising models from dynamics} \label{sec:learning_algo}
\subsection{Learning algorithms}
We now describe how to adapt RISE and RPLE into computationally efficient estimators for learning Ising models from Glauber dynamics. Both algorithms minimize a convex loss function that relies on the properties of the conditional distributions rather than the full probability distribution. These methods reconstruct the neighborhoods of each node independently and are, therefore, fully parallelizable. Moreover, the minimization procedure can be implemented in $\tilde{O}(n^2)$ using entropic gradient descent for both RISE \cite{vuffray2019efficient} and RPLE (see \cite{Klivans2017} for a related stochastic first-order method with multiplicative updates).

Unlike the i.i.d. sample setting, the Glauber dynamics naturally takes the form of a local neighborhood update rule conditioned on the event that the spin $i$ is updated, see Eq.~\eqref{eq:glauber_dynamics_conditional_prob_update}. Following the construction of the RISE estimator in \cite{vuffray2016interaction}, we define the Dynamic Interaction Screening Objective for each node $i \in V$ as being the inverse of the exponent of the conditional distribution, 
\vspace{-0.05in}
\begin{align}
    &\textbf{D-ISO:} \, \mathcal{S}_m(\underline{J}_{i}, H_i) \nonumber
    \\
    & = \frac{1}{m_i} \sum_{t=1}^{m} \exp \left[-\sigma_i^{t+1} \left(\sum_{j \neq i} J_{ij} \sigma_j^t + H_i\right) \right] \delta_{i,{I^{t+1}}}, \label{eq:gd_iso}
\end{align}
\vspace{-0.05in}
where $\underline{J}_{i}:=\left\{ J_{ij} \mid j\neq i\right\} \in \mathbb{R}^{n-1}$ denotes the vector of pairwise interactions around a node $i$, and $m_{i} = \sum_{t=1}^m \delta_{i,{I^{t+1}}}$. The Kronecker delta $\delta_{i,{I^{t+1}}}$ in Eq.~\eqref{eq:gd_iso} is used to keep samples for which updates happened at $i$.
% The D-ISO objective has been stated considering the \treg{}. The similar \mreg{} version can be found in Appendix~~\ref{sec:analysis_gd_estimators}.

The estimators' objectives have been stated considering the samples come from the \treg{}. Similar expressions are stated for the \mreg{} in Appendix~~\ref{sec:analysis_gd_estimators}.

We call the corresponding estimator which uses D-ISO as the \textit{ Dynamic Regularized Interaction Screening Estimator (D-RISE)} and is defined in the spirit of \cite{vuffray2016interaction} as the following convex program,
\begin{align}
    \textbf{D-RISE:} \, (\hat{\underline{J}}_i, \hat{H}_i) = \argmin_{(\underline{J}_i, H_i)} \mathcal{S}_m (\underline{J}_i, H_i) + \lambda ||\underline{J}_i||_1,\label{eq:gd_rise}
\end{align}
where the $\ell_1$-regularization promotes sparsity and $\lambda$ is a tunable parameter controlling the amount of sparsity enforced.

The pseudo-likelihood based estimator can be understood as the (negative) conditional likelihood \cite{ravikumar2010high} of an update at node $i$ and takes the following form in the case of Glauber dynamics,
\begin{align}
    \notag
    & \textbf{D-PL:} \, \mathcal{L}_m(\underline{J}_{i}, H_i)
    \\
    & = - \frac{1}{m_i} \sum_{t=1}^{m} \ln\left[1 + \sigma_i^{t+1} \tanh\left(\sum_{j \neq i} J_{ij} \sigma_j^t + H_i\right) \right] \delta_{i,{I^{t+1}}}. \label{eq:gd_pl}
\end{align}
Analogous to D-RISE, the \textit{Dynamic Regularized Pseudo-Likelihood Estimator (D-RPLE)} takes the form of an $\ell_1$-regularized convex program,
\begin{align}
    \textbf{D-RPLE:} \, (\hat{\underline{J}}_i, \hat{H}_i) = \argmin_{(\underline{J}_i, H_i)} \mathcal{L}_m (\underline{J}_i, H_i) + \lambda ||\underline{J}_i||_1\label{eq:gd_rple}
\end{align}

% The version of D-RPLE for the \mreg{}
% can be found in Appendix~~\ref{sec:analysis_gd_estimators}.
The performance of the learning algorithms depends on the regularization parameter $\lambda$. Setting it too high encourages the interaction parameters to drop out and setting it too low can make the estimation sensitive to noise. Following theoretical considerations explained further, a good choice for successfully reconstructing the local neighborhood of $i$ with probability $1-\delta'$ (where $\delta'=\delta/n$ and $1-\delta$ is the success of the whole graph reconstruction) is to set $\lambda = c_\lambda \sqrt{\log (n^2/\delta') / m_i}$ where the intensity of the penalty increases in a logarithmic fashion with the size of the system $n$ and decreases with the number of spin updates observed $m_{i}$. The parameter $c_\lambda>0$ is a numerical constant independent of the problem parameters such as $m_{i}$ and $n$.

One of our main contributions is the following theorem which quantifies the sample complexity required for structure learning from Glauber dynamics in the \mreg{}.
\begin{theorem}[\mreg: Structure Learning of Ising Model Dynamics] \label{thm:informal_structure_learning_gdrise_gdrple} Let $\{ {\underline{\sigma}^0}^{(t)},{\underline{\sigma}^1}^{(t)}, {I^1}^{(t)}\}_{t \in [m]}$ be $m$ samples of spin configurations and corresponding node identities drawn through Glauber dynamics (Eq.~\ref{eq:glauber_dynamics_conditional_prob_update}), and define $m_{i} = \sum_{t=1}^m \delta_{i,{I^1}^{(t)}}$ as the number of updates per spin $i$. Consider ~\mreg~on an Ising model with maximum degree $d$, maximum coupling intensity $\beta$, minimum coupling intensity $\alpha$, and for simplicity assume $H^\star_i = 0$ $\forall i$. Then for any $\delta > 0$, the following estimators with penalty parameter of form $\lambda \propto \sqrt{\log(3n^3/\delta)/m_i}$ reconstruct the edge-set perfectly with probability $p(\hat{E}(\lambda,\alpha)=E) \geq 1-\delta$ if the number of samples satisfies
\vspace{-0.05in}
\begin{enumerate}[label=\roman*)]
    \item D-RPLE: $m_i \geq C_d \max(1,\alpha^{-2}) \exp(4 \beta d) \ln (3n^3/\delta)$,
    \item D-RISE: $m_i \geq C_d^\prime \max(1,\alpha^{-2}) \exp(2 \beta d) \ln (3n^3/\delta)$,
\end{enumerate}
where $C_d$ and $C_d^\prime$ depend only polynomially on $d$.
\end{theorem}
A more precise statement and proof of Theorem~\ref{thm:informal_structure_learning_gdrise_gdrple} is given in Appendix~\ref{sec:analysis_gd_estimators}. Notice that given the choice of the initial distribution $p(\underline{\sigma}_0)$ to be the uniform distribution, the total number of samples $m$ required to get the number of samples $m_i$ that satisfies Theorem~\ref{thm:informal_structure_learning_gdrise_gdrple} is $m = O(n m_i)$. Notice that unlike the i.i.d. case where the entire sample may be distinct, in the dynamic case only a single spin is updated while the values of other variables are kept fixed; hence, perhaps a more natural quantity for comparison with the i.i.d. case is the number of updates per spin $m_i$ instead of $m$.

We expect the worst-case scalings of learning from dynamics in the \treg~to be similar to the i.i.d. setting as it includes the fully mixed setting as a particular case.
The main contribution of Theorem~\ref{thm:informal_structure_learning_gdrise_gdrple} lies in that it establishes with certainty that the scaling of D-RISE in the \mreg{} is strictly better than in the i.i.d. setting where it was found experimentally to be at least $\exp(4 \beta d)$ in the worst case \cite{lokhov2018optimal}. It is also interesting to compare the above results to the theoretical analysis of the i.i.d. setting for which RPLE and RISE have scalings upper-bounded by $\exp(8 \beta d)$ and $\exp(6 \beta d)$ respectively \cite{lokhov2018optimal}.
However, these theoretical upper-bounds tend to be loose and this motivates us to quantify the scalings achieved in practice in the dynamic case experimentally. 

\subsection{Empirical $\beta$ scaling of the sample complexity} \label{subsec:empirical_beta_studies}
Our main goal is to assess the empirical sample complexity of our learning algorithms in both the \treg{} and \mreg{}. In particular, we want to extract the exponential scaling of the sample complexity for successful structure reconstruction with respect to $\beta$, the magnitude of the largest coupling. The sample complexity of D-RPLE and D-RISE are tested over Ising models of different topologies and interaction strengths. We do not include a comparison to the algorithm of \cite{bresler2017learning} due to its high computational complexity, and to heuristic mean-field methods \cite{roudi2011mean, mezard2011exact, zeng2011network, zhang2012inference, bachschmid2015learning} since most of them are derived for asymteric kinetic Ising model, and are not guaranteed to reconstruct arbitrary strongly interacting models; instead, we focus on studying the performance of two exact methods that can be efficiently implemented through convex optimization.   

We denote the minimal number of samples required for perfect structure reconstruction with probability $1-\delta \geq 0.95$ (for $\delta=0.05$) as $m^\star$. Our experimental protocol to find $m^\star$ (sample complexity of the structure learning problem) is similar to the one from Supplementary material of \cite{lokhov2018optimal}. For each graphical model topology and coupling values, we determined $m^\star$ by finding the minimal value of $m$ samples required to successively reconstruct the structure $45$ times in a row from $45$ independent sets of $m$ samples, which guarantees a $90\%$ confidence for $\delta=0.05$.

We reconstruct the topology by first solving the optimization problems in Eq.~\eqref{eq:gd_rise} for D-RISE and in Eq.~\eqref{eq:gd_rple} for D-RPLE with appropriate $\ell_1$ regularization to obtain estimates of $(\hat{\underline{J}}_i,\hat{H}_i)$ at each node $i \in V$. We create a consensus of the estimated couplings by averaging the reconstruction from two nodes i.e. $\hat{J}_{ij}^{\rm avg} = (\hat{J}_{ij} + \hat{J}_{ji})/2$. The set of pairwise interactions $\underline{\hat{J}}^{\rm avg}$ which are higher than $\alpha/2$ are defined as the estimate of the edge set $\hat{E}$. The structure is declared to be successfully reconstructed when the edge set is perfectly recovered $\hat{E}=E$.

We perform an extensive set of numerical experiments to empirically obtain $m^\star$ for a variety of graphs in both the \treg{} and the \mreg{}. We consider two different topologies: the periodic two-dimensional lattice and the random 3-regular graph. Each of these two topologies are subdivided into three categories according to the signs of the interaction couplings. This includes the ferromagnetic case with positive couplings, the spin glass case with couplings taking random signs and the ferromagnetic case with a weak anti-ferromagnetic impurity (i.e., weak negative coupling). For each of these six cases, all the couplings' magnitudes are set to $|J^\star_{ij}| = \beta$ with exception of one or two couplings which are set to $|J^\star_{ij}| = \alpha$. In our experiments, we fixed the value of $\alpha=0.4$ and varied $\beta$ from $\alpha$ to a value ranging from $1$ to $4$ depending on the model. All models have their magnetic fields at each node $H_i^\star$ set to zero. These cases are identical to the test cases used in \cite{lokhov2018optimal} in the i.i.d. learning setting which enables us to draw a comparison between the dynamic and i.i.d. settings.

In deciding the $\ell_1$-regularization to be used, optimal values of $c_\lambda$ which are unknown apriori were determined through detailed numerical simulations on different Ising model topologies as described in Appendix~\ref{sec:emp_selection_lambda}. The determined optimal values of $c_{\lambda}$ are summarized in Table~\ref{tab:optimal_lambda} on lattices and random regular (RR) graphs for the two different regimes.
\vspace{-0.1in}
\begin{table}[ht!]
\small
\begin{tabular}{ll|l|l|l}
    & \multicolumn{2}{c|}{\textbf{T-regime}} & \multicolumn{2}{c}{\textbf{M-regime}} \\ \cline{2-5} 
    & \multicolumn{1}{c|}{\textbf{D-RISE}} & \multicolumn{1}{c|}{\textbf{D-RPLE}} & \multicolumn{1}{c|}{\textbf{D-RISE}} & \multicolumn{1}{c}{\textbf{D-RPLE}} \\ \hline
    \multicolumn{1}{l|}{Lattices} & 0.1 & 0.05 & 0.1 & \multicolumn{1}{l}{0.05} \\ \hline
    \multicolumn{1}{l|}{RR Graphs} & 0.45 & 0.1  & 0.7 & \multicolumn{1}{l}{0.3} \\ \hline
\end{tabular}
\caption{Optimal values of $c_\lambda$ for D-RISE and D-RPLE.}
\label{tab:optimal_lambda}
\end{table}
\vspace{-0.1in}

\begin{figure*}[ht!]
\centering
\subfloat[Scaling of $m^\star$ with $\beta$ in \treg.]{\includegraphics[scale=0.28]{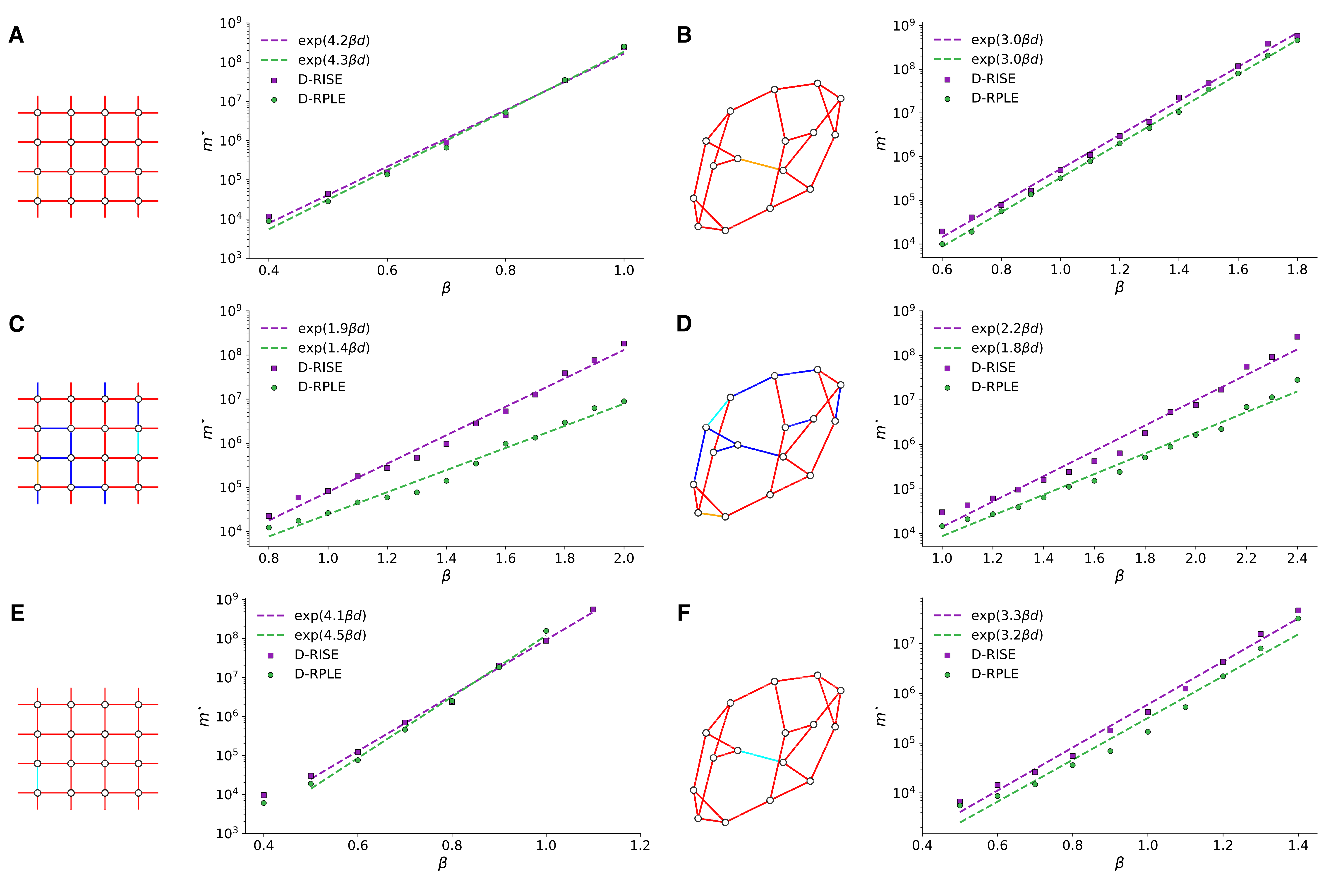}\label{fig:sample_complexity_beta_scaling_T_regime}\vspace{-0.03in}}
\\
\subfloat[Scaling of $m^\star$ with $\beta$ in \mreg.]{\includegraphics[scale=0.28]{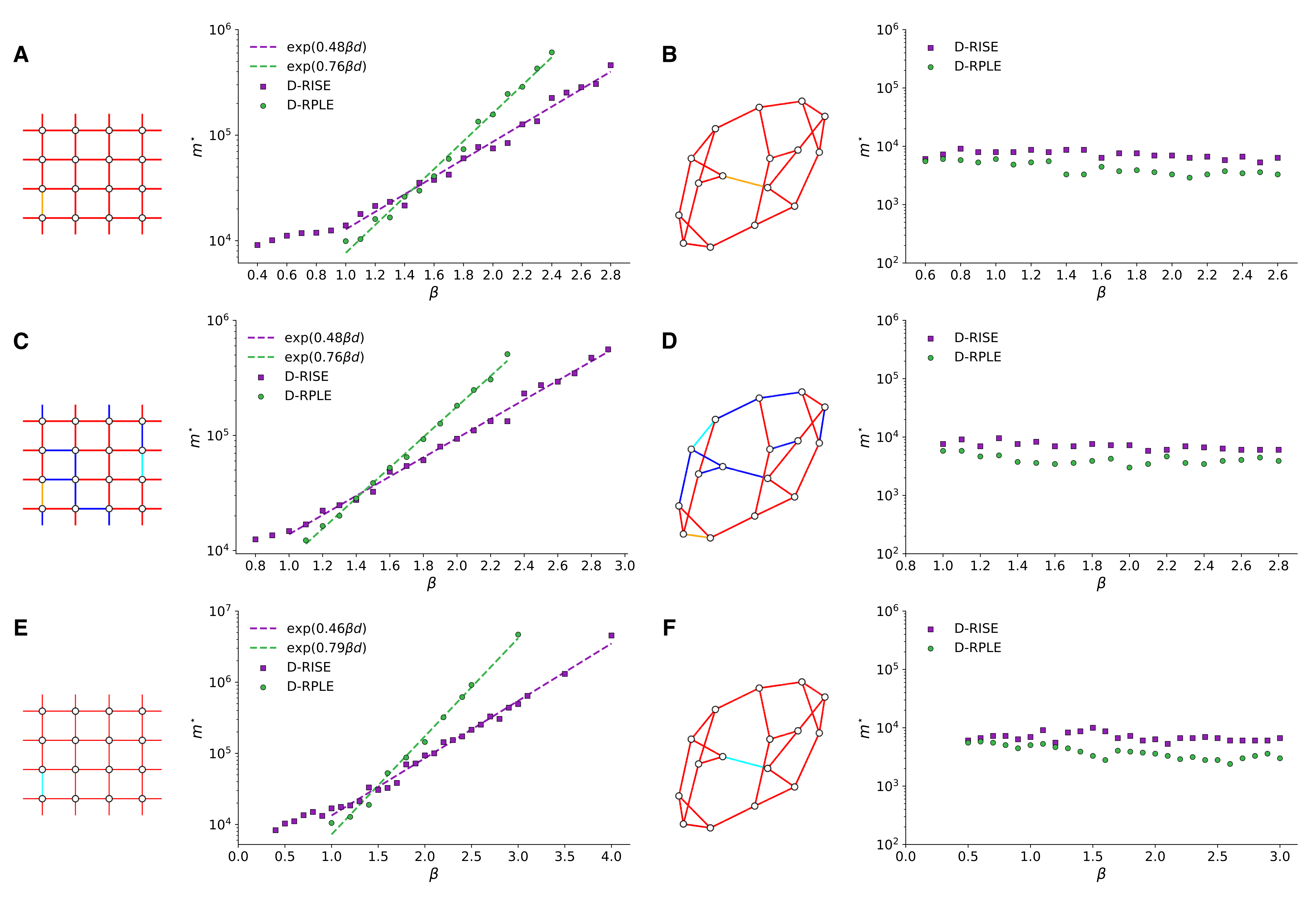}\label{fig:sample_complexity_beta_scaling_M_regime}\vspace{-0.03in}}
\caption{We assess the performance of D-RISE and D-RPLE in reconstructing Ising models of size $n=16$ from samples generated from Glauber dynamics. The different Ising model topologies with their corresponding pictorial representations on the left-hand side of each plot are: (A) ferromagnetic model on a periodic lattice, (B) spin glass model on a periodic lattice, and (C) ferromagnetic model on a periodic lattice with weak antiferromagnetic impurity. Edges in the pictorial representations of the models are colored red ($\beta$), orange ($\alpha$), cyan ($-\alpha$) and blue ($-\beta$). Value of $\alpha=0.4$ for all the graphs.}
\end{figure*}

Our sample complexity results for the \treg{} and \mreg{} are shown in Figure~\ref{fig:sample_complexity_beta_scaling_T_regime} and Figure~\ref{fig:sample_complexity_beta_scaling_M_regime} respectively. In the \treg{}, the worst scalings of sample complexity are observed for both D-RISE and D-RPLE in lattices with purely ferromagnetic interactions (Fig.~\ref{fig:sample_complexity_beta_scaling_T_regime}A) and those with a weak anti-ferromagnetic interaction (Fig.~\ref{fig:sample_complexity_beta_scaling_T_regime}E). The worst cases are $\exp(4.2 \beta d)$ and $\exp(4.5 \beta d)$ for D-RISE and D-RPLE respectively. The scaling of sample complexity results are similar to the i.i.d. setting for the ferromagnetic models on lattices and random regular graphs. However, compared to the i.i.d. case, we obtain a $35\%$ reduction in the $\beta$ scaling for D-RISE for the spin glass model on a lattice and around $27\%$ reduction for the spin glass model on a random regular graph. For the systems studied here, the Glauber dynamics mixes rapidly in the case of ferromagnetic models as the number of variables is small compared to the number of samples required to learn the structure. Therefore the dynamics in the \treg{} quickly produces samples that behave like i.i.d. samples in these cases and we see similar scalings. This mixing may not be as rapid in the case of spin glass models, yielding us savings in number of samples when learning from Glauber dynamics.

The picture drastically changes as we move to the \mreg{} for which the scaling results are shown in Fig.~\ref{fig:sample_complexity_beta_scaling_M_regime}. Among the numerical examples that we consider, the worst-case scaling for D-RISE and D-RPLE observed are $\exp(0.48 \beta d)$ and $\exp(0.79 \beta d)$ respectively, with the scalings being very similar for all the lattices. Compared to the \treg{} there is a reduction in the $\beta$ scaling between $80\%$ and $90\%$, translating to a reduction in the sample requirement of orders of magnitude. Interestingly, we observe that a constant number of samples independent of $\beta$ is required for learning the random regular graphs. Thus, it is possible to beat exponential scalings for special topologies in the case of \mreg{} which would not be possible in the i.i.d. setting. The details of this behavior is described in Appendix~\ref{sec:mreg_learning_rr_graphs}.

The fundamental difference between the two regimes is that the Glauber dynamics comes effectively from two different initial distributions. In the \mreg{}, the samples are produced from a one step Glauber dynamics initialized with a uniform distribution. In the \treg{}, however, the samples are \emph{effectively} produced from a one step Glauber dynamics that is initialized from a distribution which is close to the equilibrium one, as the actual dynamics mixes more rapidly. This shows that the dynamical samples acquired far from the equilibrium carry more information about the structure of graphical models. A natural question to ask is then how to empirically find such an initial distribution for the Glauber dynamics that improves the sample complexity. We propose a possible solution to this issue by introducing an active learner in Section~\ref{sec:active_learning}.

\section{Applications}\label{sec:applications}
In this section, we illustrate how D-RISE/D-RPLE can be applied to a real world system and extended to improve their performance. In Section~\ref{sec:application_learning_from_neural_data}, we consider a multi-neuron spike trains' data set to learn the structure of a network of neurons. This can be used to understand the dynamics of the network and how it implements computations. In Section~\ref{sec:active_learning}, we highlight an approach to modify the initial distribution to the \mreg~to improve sample complexity.

\subsection{Learning Ising models from neural data} \label{sec:application_learning_from_neural_data}
Due to the high dimensionality of the space of spike patterns and lack of enough data to build an exact statistical description, it has become popular to use parametric models such as Ising models \cite{schneidman2006weak}. In the corresponding Ising model \cite{rieke1999spikes}, the spin $\sigma_i$ of neuron/node $i$ can be interpreted as spiking/firing ($\sigma_i=1$) or not ($\sigma_i = -1$). Studies on using Ising models for spike trains include application to larger populations of neurons \cite{cocco2009neuronal,nirenberg2007analyzing}, conditions under which Ising models are good approximations \cite{roudi2009pairwise,tkacik2009spin}, development of faster learning methods \cite{broderick2007faster} and comparisons of different learning methods \cite{roudi2009ising}. Most of previous studies consider the i.i.d. setting. However, \cite{hertz2011ising} showed that respecting correlations in time and the dynamics can lead to better Ising model fits to the data. This motivates us to investigate the underlying Ising model for spike trains considering Glauber dynamics.

% What is the dataset we consider? 
We consider the dataset from \cite{prentice2016error} containing spike trains from $152$ salamander retinal ganglion cells in response to a non-repeated natural movie stimulus, of which we select spike trains over $n=42$ neurons over $24$s for our application. To use D-RISE/D-RPLE, the dataset is first converted into a sequence of $1.2 \times 10^5$ spin configurations, a segment of which is shown as a spike raster in Figure~\ref{fig:spike_train_neural_data}. Time series of spin configurations along with updated node identities are then extracted from this sequence and $3.2 \times 10^4$ samples corresponding to the \mreg{} with an unknown initial distribution are obtained. Details of this procedure is given in Appendix~\ref{sec:neural_dataset_prep}. 
\begin{figure}[h!]
    \centering
    \includegraphics[scale=0.43]{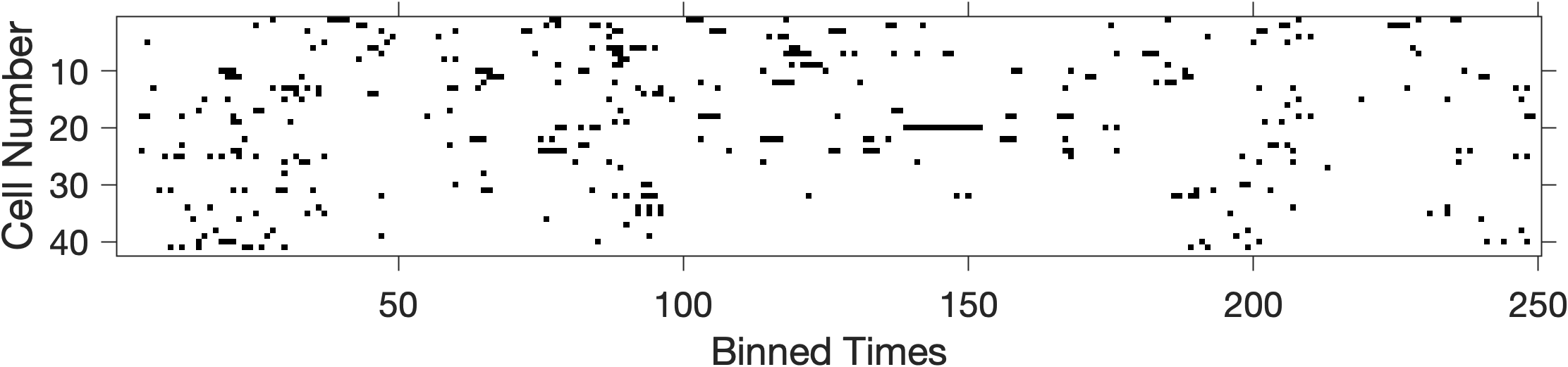}
	\caption{Spike raster from the first $5$ sec of the data over $42$ neurons. Each column indicates the spiking pattern of the neurons over a 20 ms time bin.}
	\label{fig:spike_train_neural_data}
	\vspace{-0.1in}
\end{figure}

% What were our results? -- Working in M-regime, effective glauber dynamics, etc.
We discuss the statistics of Ising model parameters learned using D-RISE on this set of samples in Appendix~\ref{sec:neural_dataset_prep}, where we also compare the recovered parameters with those obtained by RISE assuming the samples are i.i.d. Correlations computed from data assuming the samples are i.i.d. (Figure~\ref{fig:correlations_iid}) and that respecting time (Figure~\ref{fig:correlations_dynamics}) are very different. This difference strengthens the importance of respecting dynamics when learning an effective Ising model if one would like to capture time correlations present in the data. The correlation matrix predicted using the model learned with D-RISE is presented in Figure~\ref{fig:predicted_correlations_dynamics}, and is within $\sim 10\%$ of that computed from data under the Frobenius norm (see Figure~\ref{fig:diff_pred_data_correlations_dynamics}), indicating a good model fit that respects the time correlations present in the data. Details of correlation computation can be found in the Appendix~\ref{sec:neural_dataset_prep}.

As we have control in such biological systems through external stimuli, learning an effective Ising model could be accelerated using an active learner which we discuss next.
%
% \begin{figure}[h!]
%     \centering
%     \includegraphics[scale=0.45]{Figures/neural_spikes/neural_histogram_J_H_inset.png}
% 	\caption{Ising model parameters learned from spike trains over $42$ neurons using $3.2 \times 10^4$ samples. Significant couplings are in blue and thresholded couplings are in red. Reconstructed fields are in green in a separated histogram.}
% 	\label{fig:learned_model_neural_data}
% 	\vspace{-0.1in}
% \end{figure}
% %
% \begin{figure}[h!]
%     \centering
%     \includegraphics[scale=0.35]{Figures/neural_spikes/neural_comparison_estimates_DRISE_RISE.png}
% 	\caption{Comparison of Ising model parameter estimates $\hat{J}_{ij}$ obtained through RISE and D-RISE}
% 	\label{fig:neural_comparison_drise_rise_estimate}
% 	\vspace{-0.2in}
% \end{figure}
%
\begin{figure}[h!]
\centering
\subfloat[$\mathrm{Corr}(\sigma_i, \sigma_j)$]{\includegraphics[scale=0.35]{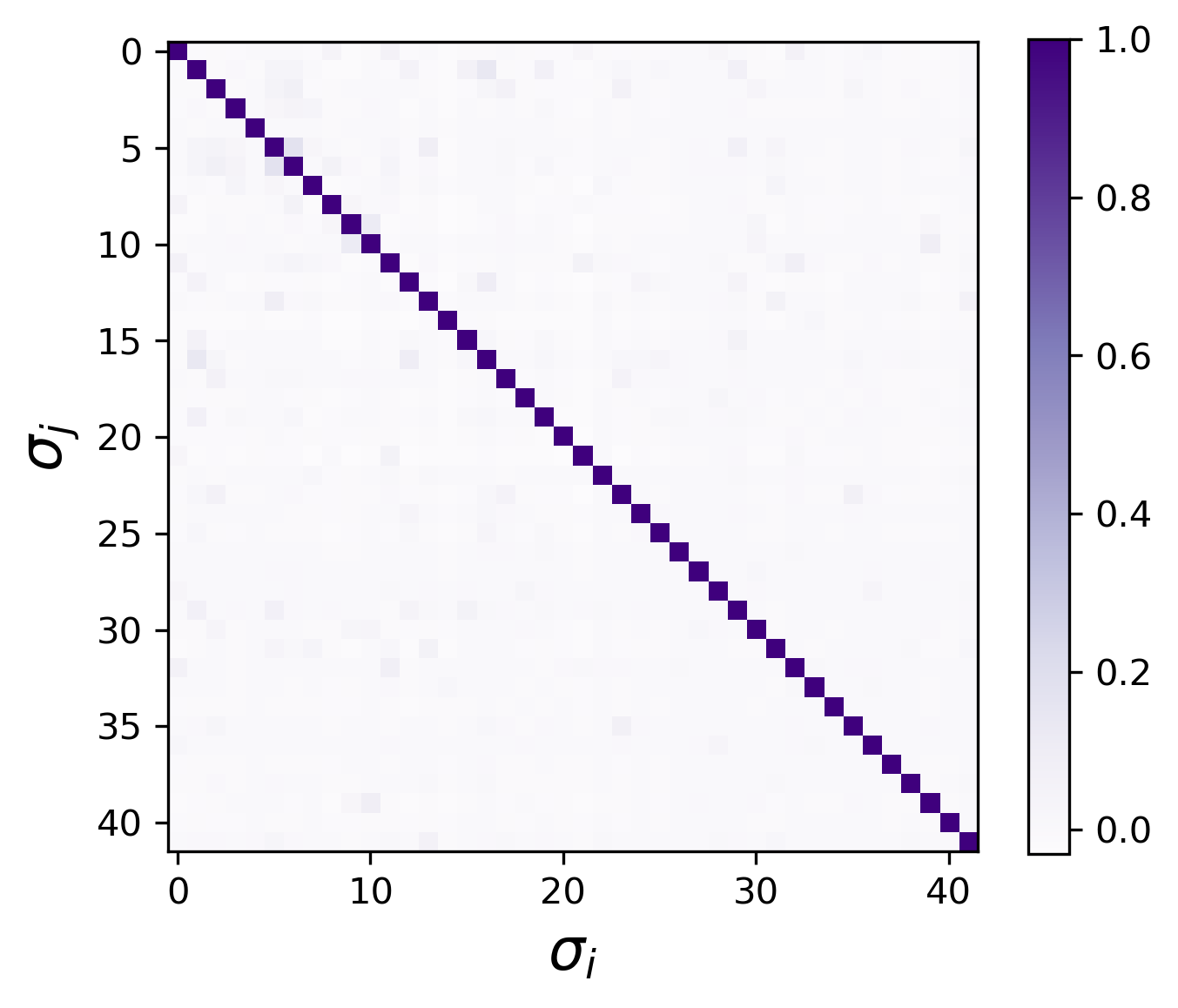}\label{fig:correlations_iid}\vspace{-0.03in}}
\subfloat[$\mathrm{Corr}(\sigma_i^0, \sigma_j^1)$]{\includegraphics[scale=0.35]{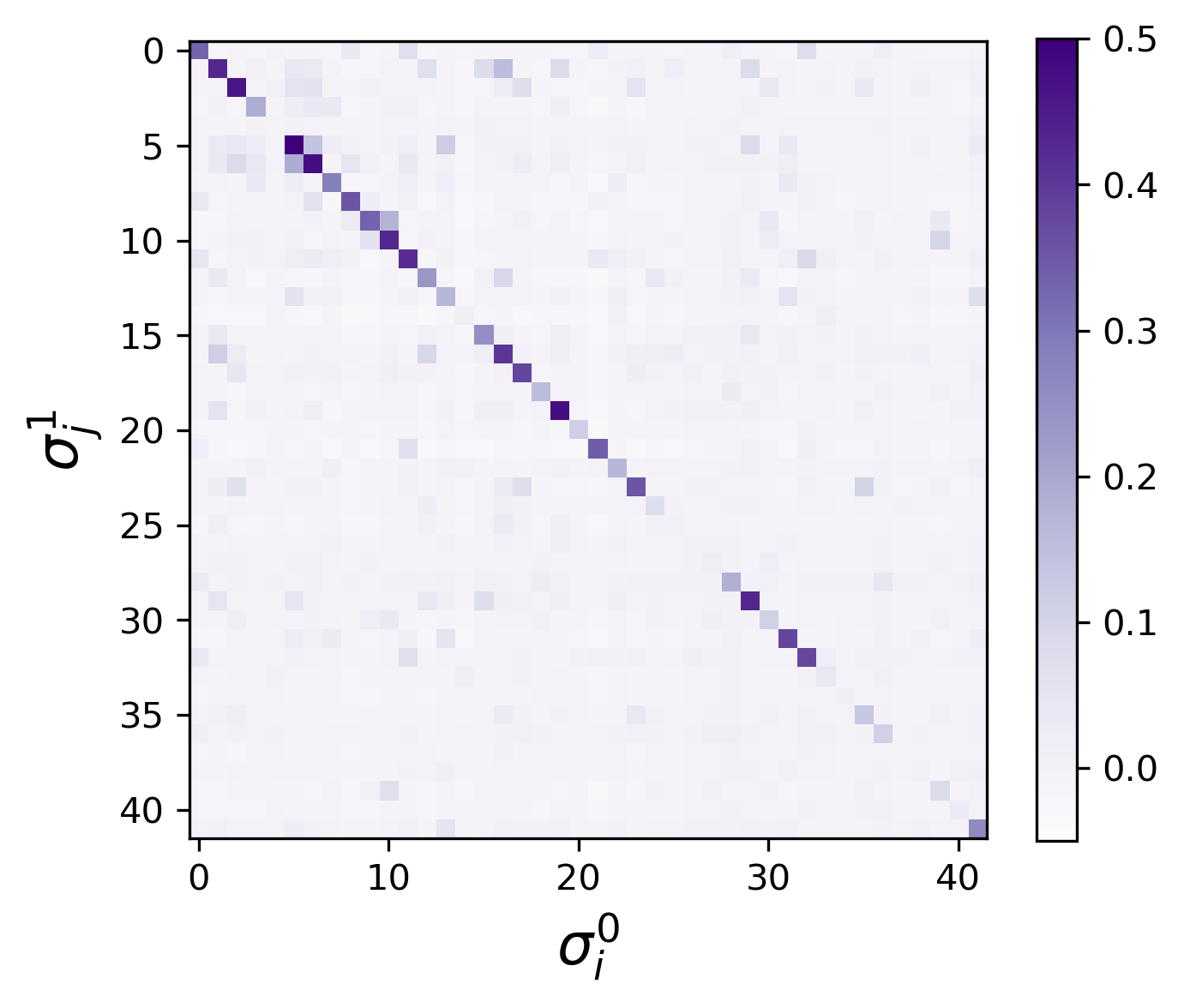}\label{fig:correlations_dynamics}\vspace{-0.03in}}
\\
\subfloat[Predicted $\mathrm{Corr}(\sigma_i^0, \sigma_j^1)$]{\includegraphics[scale=0.35]{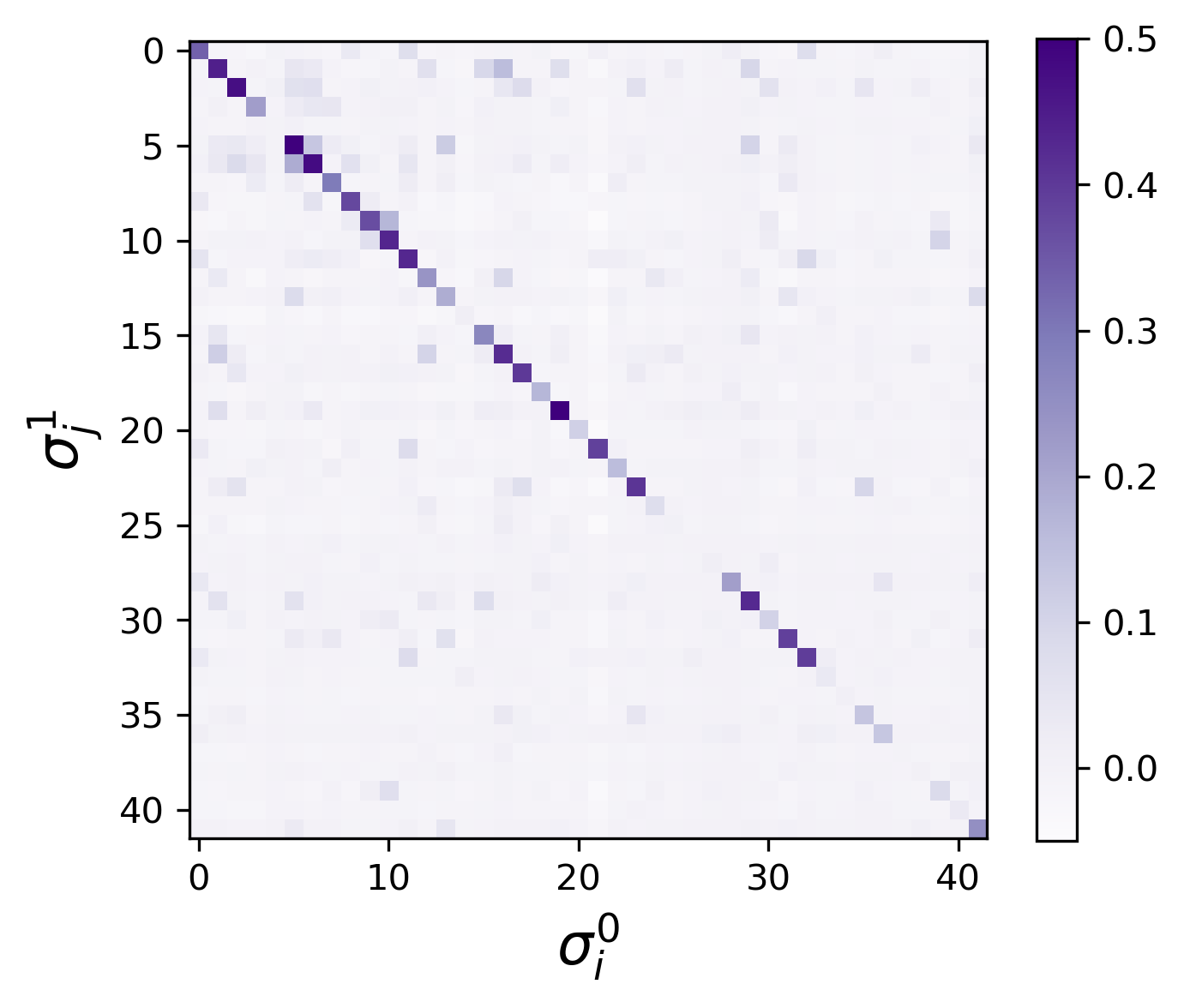}\label{fig:predicted_correlations_dynamics}\vspace{-0.03in}}
\subfloat[Difference in predicted and data correlations]{\includegraphics[scale=0.35]{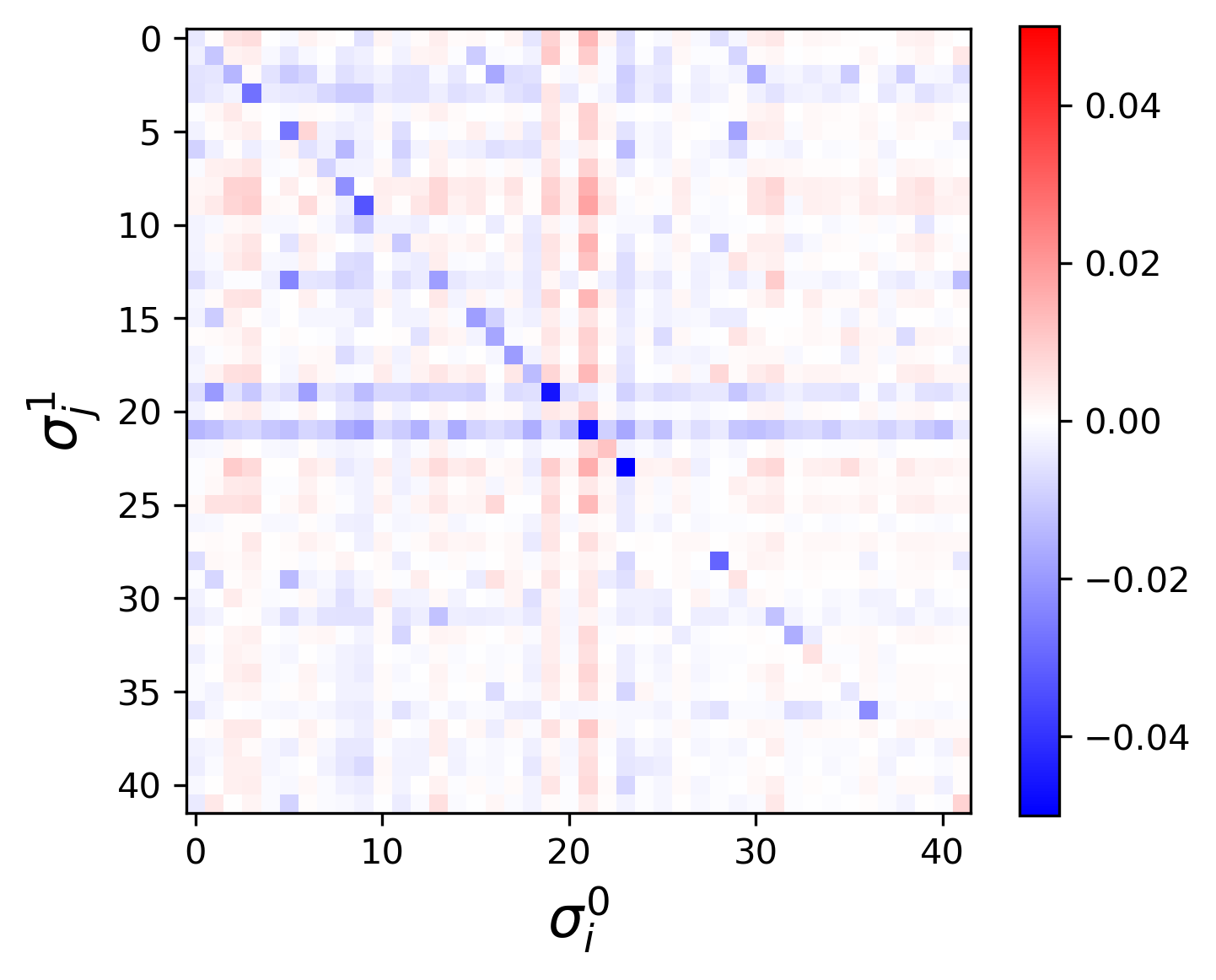}\label{fig:diff_pred_data_correlations_dynamics}\vspace{-0.03in}}
\caption{Correlation matrices are computed from data in (a) assuming the spin configurations are i.i.d., and (b) respecting time ordering (see definition in Appendix~\ref{sec:neural_dataset_prep}). We show the predicted correlation matrix using D-RISE estimates in (c) and its difference from that computed from data in (d).}\label{fig:neural_spikes_correlations}
\end{figure}

\subsection{Active learning in M-regime} \label{sec:active_learning}
% Remark on  Initial Distribution in M-Regime
Motivated by the previous section, we develop an active learning algorithm that optimizes on the fly over the initial distribution in the \mreg{}. In the \mreg{}, the initial spin configuration $\underline{\sigma}^0$ can be viewed as a query to the Glauber dynamics which returns the output of $(\underline{\sigma}^1,I^1)$. Clearly, the observations generated have a dependence on the query distribution $p_{\underline{\sigma}^0}$ from which the queries $\underline{\sigma}^0$ are simulated. If no prior information about the parameter set $(\underline{J},\underline{H})$ is known, then a suitable choice of $p_{\underline{\sigma}^0}$ is the uniform distribution over $\{-1,+1\}^n$ as we had chosen for our numerical experiments in the previous section. 

% What is an active learner and what are the different active learning strategies in literature?
In \textit{active learning}, the learner has the ability to select queries during model learning that would be most informative. Here, we consider a \textit{mini-batch adaptive active learning} \cite{wei2015submodularity} scheme where in each round a mini-batch of initial spin configurations $\underline{\sigma}^0$ are selected to be queried and then the samples obtained are combined with those from before to produce estimates of $(\hat{\underline{J}},\hat{\underline{H}})$. These estimates are then used to determine the next mini-batch of queries. To select these queries we use the informative measure of entropy as in uncertainty sampling. An alternate criteria that can be used is Fisher information but the computational effort of the resulting query optimization typically grows exponentially with $n$ \cite{sourati2017probabilistic}. In uncertainty sampling \cite{settles2009active}, one query $\underline{\sigma}^0$ is chosen at a time 
\begin{equation}
    \hat{\underline{\sigma}}^0 = \argmax_{\underline{\sigma}^0 \in \{-1,+1\}^n} S(\underline{\sigma}^1|\underline{\sigma}^0;\hat{\underline{J}},\hat{\underline{H}})
\end{equation}
where $S$ is the entropy measure of the probability $p(\underline{\sigma}^1|\underline{\sigma}^0;\hat{\underline{J}},\hat{\underline{H}})$ based on current parameter estimates. The entropy in the case of Glauber dynamics is 
\begin{equation}
    S(\underline{\sigma}^1|\underline{\sigma}^0) = \sum \limits_{k \in [n]} \log(2 \cosh(A_k)) - A_k \tanh(A_k)
\end{equation}
where $A_k = \sum \limits_{l \in \partial k} J_{kl} \sigma_l^0 + H_k$. Here, we issue mini-batches of queries sampled from distribution $q$ that is proportional to the the entropy $S$. Our algorithm is given in Algorithm~\ref{algo:AL_glauber_dynamics_M_regime}.
%
% Q: What is the algorithm we have applied? Discuss pool-based active learning based on entropy
%

\begin{algorithm}[ht!]
	\caption{\small Active Learning of Glauber Dynamics}\label{algo:AL_glauber_dynamics_M_regime}
	\footnotesize
	\textbf{Input}:  Initial set of samples $X^{(0)}$, number of mini-batches $i_{max}$, size of mini-batch $m_b$
 	\begin{algorithmic}
	    \STATE $(\hat{\underline{J}},\hat{\underline{H}}) \leftarrow \text{D-RISE}(X^{(0)})$
		\FOR{$i=1:i_{max}$}
            \STATE Compute entropy $S(\underline{\sigma}^1|\underline{\sigma}^0;\hat{\underline{J}},\hat{\underline{H}}) \, \forall \underline{\sigma}^0$
            \STATE Set $q(\underline{\sigma}^0)~\propto~S(\underline{\sigma}^1|\underline{\sigma}^0;\hat{\underline{J}},\hat{\underline{H}})$
            \STATE Modify distribution: $q = \mu q + (1-\mu)p_U$
            \STATE Sample $m_{b}$ queries from $\{-1,+1\}^n$ w.p. $q$
            \STATE Obtain corresponding samples $X_b$ by running Glauber dynamics in \mreg{}
            \STATE Set $X^{(i)}=X^{(i-1)} \bigcup X_b$
            \STATE $(\hat{\underline{J}},\hat{\underline{H}}) \leftarrow \text{D-RISE}(X^{(i)})$
        \ENDFOR
 	\end{algorithmic}
 	\textbf{Output}: $(\hat{\underline{J}},\hat{\underline{H}})$
\end{algorithm}

\begin{figure}[h]
    \centering
    \includegraphics[scale=0.35]{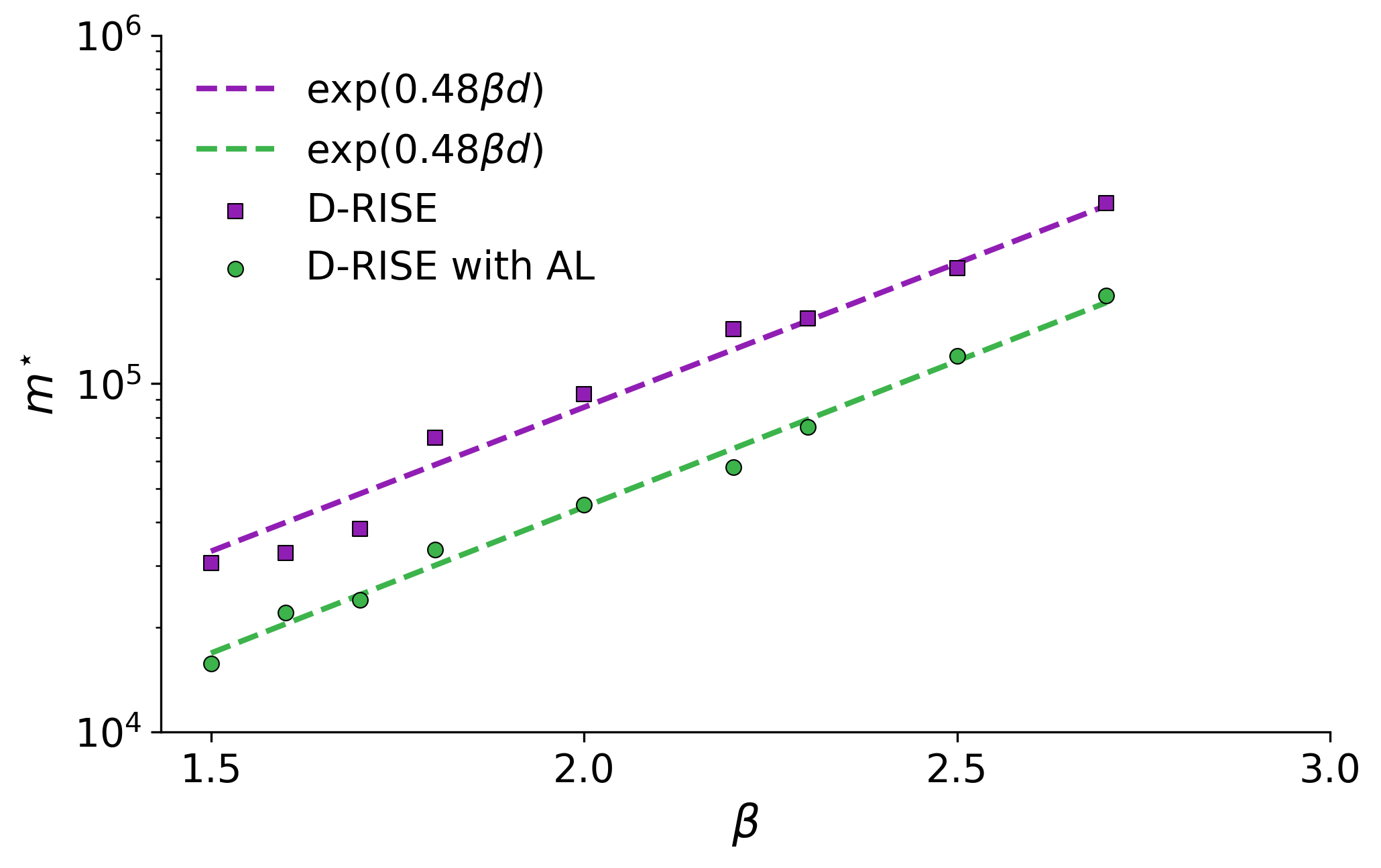}
	\caption{Scaling of $m^\star$ with $\beta$ in \mreg{}. Performance comparison of D-RISE with active learning against a vanilla D-RISE in reconstructing a ferromagnetic model with weak anti-ferromagnetic impurity of size $n=16$ (see Fig.~\ref{fig:sample_complexity_beta_scaling_M_regime}E).}
	\label{fig:sample_complexity_beta_scaling_GDRISE_AL}
	\vspace{-0.1in}
\end{figure}

Note that we slightly modify the query distribution $q$ for each mini-batch by mixing it with the uniform distribution $p_U$. The mixing coefficient $0 \leq \mu \leq 1$ typically depends on the number of samples so far. We set it to $\mu = 1 - 1/|X^{(i)}|^{1/6}$ which is often used for such active learning algorithms \cite{sourati2017asymptotic,chaudhuri2015convergence}.

% Q: How do we test m^\star here?
To determine the sample complexity, the minimal number of samples required for successful structure reconstruction $m^\star$ is determined as described in Sec.~\ref{subsec:empirical_beta_studies}. However, here each trial corresponds to an independent active learning run. In each trial for a given value of $m$, we consider the size of initial set of samples to be $\lfloor m/3 \rfloor$ and size of batches $m_b$ such that there are a total of $15$ mini-batches.

% Q: What are the results for Fig.E?
In Figure~\ref{fig:sample_complexity_beta_scaling_GDRISE_AL}, we compare the sample complexity of D-RISE with and without active learning (AL) on the challenging case of a ferromagnetic periodic lattice with a weak anti-ferromagnetic impurity (Fig.~\ref{fig:sample_complexity_beta_scaling_M_regime}E). We consider values of $\beta$ between $1.5$ and $2.7$. While the scaling of sample complexity with $\beta$ remains unchanged (within experimental error), there is about $47\%$ reduction in the number of samples required for successful graph reconstruction when using AL. 

\section{Conclusions and future work} \label{sec:conclusions}
% Summary
In this paper, we theoretically and empirically establish a fundamental difference between learning graphical models in the traditional framework of i.i.d. samples and samples obtained from out of equilibrium dynamics. We show that in the latter understudied setting, there is considerable improvement in sample complexity which has both theoretical and practical consequences. In future work, it would be interesting to further investigate general Markov Random Fields and other Markov chain dynamics.

\section*{Code and data availability}
The code for the learning algorithms, active learner and data in this work is available on GitHub\footnote{\url{https://github.com/lanl-ansi/learning-ising-dynamics}}.

\section*{Acknowledgements}
A.D. was partially supported by the Applied Machine Learning Fellowship at LANL where most of the work was carried out. A.Y.L., M.V., and S.M. acknowledge support from the Laboratory Directed Research and Development program of Los Alamos National Laboratory under project numbers 20190059DR, 20200121ER, and 20210078DR.

\bibliographystyle{icml2021}
\bibliography{references}% Produces the bibliography via BibTeX.

\clearpage

\appendix

\onecolumn

\renewcommand{\theequation}{S\arabic{equation}}

\setcounter{equation}{0}

\renewcommand{\thefigure}{S\arabic{figure}}

\setcounter{figure}{0}

\renewcommand{\thesection}{S\arabic{section}}

\setcounter{section}{0}

\renewcommand\appendixname{}

\vspace{1cm}

\include{supplement}

\end{document}

%% file: supplement.tex
\section{Analysis of the estimators D-RISE and D-RPLE} \label{sec:analysis_gd_estimators}
In this section, we provide a rigorous analysis of the sample complexity of learning from Glauber dynamics in the ~\mreg~ on Ising models using the estimators of D-RISE and D-RPLE. Theorems~\ref{thm:estimate_drise_drple} and \ref{thm:structure_learning_drise_drple} are established.

To simplify the analysis, we consider the case of Ising models with zero local magnetic fields (i.e., $H^\star_i=0$). The probability measure of a particular configuration of spins $\underline{\sigma} \in \{-1,+1\}^n$ on the resulting Ising model with $n$ nodes is given by
\begin{equation}
    p(\underline{\sigma}) = \frac{1}{Z} \exp \left( \sum_{(i,j) \in E} J^\star_{ij} \sigma_i \sigma_j \right).
    \label{eq:gibbs_distribution_sim}
\end{equation}

As noted earlier in Section~\ref{sec:active_learning}, it is useful to view the initial spin configuration $\underline{\sigma}^0$ as a query to the Glauber dynamics which returns the output of $(\underline{\sigma}^1,I^1)$ in the \mreg{}, where $I^1$ is the identity of the spin being updated at time $t=1$. The conditional probability of updating node $i$ through Glauber dynamics in the \mreg~is then given by
\begin{equation}
    p(\sigma_i^1|\underline{\sigma}^0, \delta_{i,I^1}=1) = \frac{\exp \left[\sigma_i^1 (\sum_{j \in \partial i} J^\star_{ij} \sigma_j^0 ) \right]}{2 \cosh \left[ \sum_{j \in \partial i} J^\star_{ij} \sigma_j^0 \right]}.
    \label{eq:glauber_dynamics_conditional_prob_update_sim}
\end{equation}
where the initial probability distribution over $\underline{\sigma}^0$ is the uniform distribution over all possible spin configurations:
\begin{equation}
    p(\underline{\sigma}^0 = \underline{\sigma}) = \frac{1}{2^n}, \, \forall \underline{\sigma} \in \{-1,+1\}^n.
    \label{eq:initial_pdf_GD_Ising}
\end{equation}

Using $m$ samples $\{ ({\underline{\sigma}^0}^{(t)}, {\underline{\sigma}^1}^{(t)}, {I^1}^{(t)}) \}_{t \in [m]}$ generated through Glauber dynamics (Eq.~\ref{eq:glauber_dynamics_conditional_prob_update_sim}) in the \mreg{}, the Ising model is learned using the D-RISE and D-RPLE estimators which are based on the following (simplified) objectives:
\begin{equation}
    \text{D-ISO:} \quad \mathcal{S}_m(\underline{J}_{u}) 
     = \frac{1}{m_u} \sum_{t=1}^{m} \exp \left[-{\sigma_u^1}^{(t)} \left(\sum_{i \neq u} J_{ui} {\sigma_i^0}^{(t)}\right) \right] \delta_{u,{{I^1}^{(t)}}}. \label{eq:gd_iso_sim}
\end{equation}
\begin{equation}
    \text{D-PL:} \quad \mathcal{L}_m(\underline{J}_{u}) = - \frac{1}{m_u} \sum_{t=1}^{m} \ln\left[1 + {\sigma_u^1}^{(t)} \tanh\left(\sum_{i \neq u} J_{ui} {\sigma_u^0}^{(t)} \right) \right] \delta_{u,{{I^1}^{(t)}}}.  \label{eq:gd_pl_sim}
\end{equation}
The above objective functions are given for recovering the neighborhood around node $u$. The analysis of the D-RISE and D-RPLE estimators closely follows the work of \cite{vuffray2016interaction} which analyzed the case of learning from i.i.d. samples. We are now in a position to state formal theorems regarding estimation error and sample complexity of structure learning of the estimators D-RISE and D-RPLE.

\begin{theorem}[\mreg: Error Bound on Estimates] \label{thm:estimate_drise_drple} Let $\{ {\underline{\sigma}^0}^{(t)},{\underline{\sigma}^1}^{(t)}, {I^1}^{(t)}\}_{t \in [m]}$ be $m$ samples of spin configurations and corresponding node identities drawn through Glauber dynamics (Eq.~\ref{eq:glauber_dynamics_conditional_prob_update}), and define $m_{i} = \sum_{t=1}^m \delta_{i,{I^1}^{(t)}}$ as the number of updates per spin $i$. Considering ~\mreg~on an Ising model with maximum degree $d$, maximum coupling intensity $\beta$, and assume $H^\star_i = 0$ $\forall i$. Then for any $\delta > 0$, the square error of the following estimators with the following choices of penalty parameter $\lambda \propto \sqrt{\frac{ \ln (3n^3/\delta)}{m_u}}$ is bounded with probability at least $1-\delta$ for all nodes $u \in V$ if the number of samples satisfies
\begin{enumerate}[label=\roman*)]
    \item D-RPLE: If $m_u~\geq~2^{17} d^2 \exp(4\beta d)\ln\frac{3n^3}{\delta}$ for $\lambda = 4\sqrt{2} \sqrt{\frac{ \ln(3n^3/\delta)}{m_u}}$ then $\norm{\hat{\underline{J}}_u - \underline{J}_u^\star}_2 \leq 240\sqrt{2d} \exp(2\beta d) \sqrt{\frac{\ln \frac{3n^3}{\delta}}{m_u}}$,
    \item D-RISE: If $m_u~\geq~2^{14} d^2 \exp(2\beta d)\ln\frac{3n^3}{\delta}$ for $\lambda = 4 \sqrt{\frac{ \ln(3n^3/\delta)}{m_u}}$ then $\norm{\hat{\underline{J}}_u - \underline{J}_u^\star}_2 \leq 240 \sqrt{d} \exp(\beta d) \sqrt{\frac{\ln \frac{3n^3}{\delta}}{m_u}}$. 
\end{enumerate}
\end{theorem}
The following theorem quantifies the sample complexity required for structure learning for the different estimators.
\begin{theorem}[\mreg: Structure Learning of Ising Model Dynamics] \label{thm:structure_learning_drise_drple} Let $\{ {\underline{\sigma}^0}^{(t)},{\underline{\sigma}^1}^{(t)}, {I^1}^{(t)}\}_{t \in [m]}$ be $m$ samples of spin configurations and corresponding node identities drawn through Glauber dynamics (Eq.~\ref{eq:glauber_dynamics_conditional_prob_update}), and define $m_{i} = \sum_{t=1}^m \delta_{i,{I^1}^{(t)}}$ as the number of updates per spin $i$. Consider ~\mreg~on an Ising model with maximum degree $d$, maximum coupling intensity $\beta$, minimum coupling intensity $\alpha$, and assume $H^\star_i = 0$ $\forall i$. Then for any $\delta > 0$, the following estimators with specified penalty parameter of form $\lambda \propto \sqrt{\frac{ \ln(3n^3/\delta)}{m_u}}$ reconstructs the edge-set perfectly with probability $p(\hat{E}(\lambda,\alpha)=E) \geq 1-\delta$ if the number of samples satisfies
\begin{enumerate}[label=\roman*)]
    \item D-RPLE: $m_u \geq   \max(d, \alpha^{-2}) 2^{19} d \exp(4 \beta d) \ln \frac{3n^3}{\delta}$ for $\lambda = 4\sqrt{2} \sqrt{\frac{ \ln(3n^3/\delta)}{m_u}}$,
    \item D-RISE: $m_u \geq  \max(d, \alpha^{-2})  2^{18} d \exp(2 \beta d) \ln \frac{3n^3}{\delta}$ for $\lambda = 4 \sqrt{\frac{ \ln(3n^3/\delta)}{m_u}}$.
\end{enumerate}
\end{theorem}
\textbf{Remark:} Given the choice of the initial distribution $p(\underline{\sigma}_0)$ to be the uniform distribution, the total number of samples $m$ required to get the number of samples $m_u$ that satisfy Theorems~\ref{thm:estimate_drise_drple} and \ref{thm:structure_learning_drise_drple} is $m = O(n m_u)$. 

\subsection{Conditions for controlling estimation error}
In order to control the error of the D-RISE and D-RPLE estimators, we enforce conditions shown to be sufficient in \cite{negahban2009unified} for such $\ell_1$-regularized M-estimators. These conditions are similar to the ones shown in \cite{vuffray2016interaction} and are restated here for completeness. We state them considering the D-RISE estimator but they hold for the D-RPLE estimator as well.

\begin{condition_statement}\label{cond:gradient_concentration} The $\ell_1$-penalty parameter strongly enforces regularization if it is greater than any partial derivatives of the objective function at $\underline{J}_u = \underline{J}_u^\star$
\begin{equation}
    \norm{\nabla S_m(\underline{J}_u^\star)}_\infty \leq \frac{\lambda}{2}.
    \label{eq:condition1}
\end{equation}
\end{condition_statement}
The above condition ensures that $\underline{J}_u^\star$ has at most $d$ non-zero components and then the difference of the estimates $\Delta_u = \hat{\underline{J}}_u - \underline{J}_u^\star$ lies within the set
\begin{equation}
    K := \left\{\Delta_u \in \mathbb{R}^{n-1} | \norm{\Delta_u}_1 \leq 4\sqrt{d} \norm{\Delta_u}_2 \right\}.
\end{equation}

Denoting the residual of the first order Taylor expansion of the objective function of the estimator:
\begin{equation}
    \delta S_m (\Delta, \underline{J}_u^\star) = S_m(\underline{J}_u^\star + \Delta) - S_m(\underline{J}_u^\star) - \langle \nabla S_m(\underline{J}_u^\star), \Delta \rangle.
    \label{eq:first_order_taylor_expansion_GD-RISE}
\end{equation}

\begin{condition_statement} \label{cond:strong_convexity} The objective function is restricted strongly convex with respect to $K$ on a ball of radius $R$ centered at $\underline{J}_u = \underline{J}_u^\star$, if for all $\Delta_u \in K$ such that $\norm{\Delta_u}_2 \leq R$, there exists a constant $\kappa > 0$ such that
\begin{equation}
    \delta S_m (\Delta_u, \underline{J}_u^\star) \geq \kappa \norm{\Delta_u}_2^2.
    \label{eq:condition2}
\end{equation}
\end{condition_statement}
The above condition ensures that the objective function is strongly convex in a restricted subset of $\mathbb{R}^{n-1}$. The following proposition shows that estimation error can be controlled if the above two conditions are satisfied. 

\begin{proposition}[\cite{vuffray2016interaction}]\label{prop:m_estimator_error_bound} 
If an $\ell_1$-regularized M-estimator satisfies Condition \ref{cond:gradient_concentration} and Condition \ref{cond:strong_convexity} with $R> 3 \sqrt{d}\frac{\lambda}{\kappa}$ then the error is bounded by
\begin{equation}
    \norm{\hat{\underline{J}}_u - \underline{J}_u^\star}_2 \leq 3 \sqrt{d} \frac{\lambda}{\kappa}
\end{equation}
\end{proposition}

\subsection{Proof of the error bound estimation}

For the presentation convenience, we first state Propositions that are useful for the proof. These proofs for these Propositions are given further below. 

\subsubsection{D-RISE Estimator}
\begin{proposition} [Gradient concentration of \textsc{D-RISE}] \label{prop:concentration_gradient_GD-ISO} 
For some node $u \in V$, let $m_u~\geq~\exp(2\beta~d)\ln\frac{2n}{\delta_1}$, then with probability at least $1-\delta_1$, the components of the gradient of the D-ISO are bounded from above as
\begin{equation}
    \norm{\nabla S_m(\underline{J}_u^\star)}_\infty \leq \epsilon_1
\end{equation}
where $\epsilon_1 = 2 \sqrt{\frac{\ln \frac{2n}{\delta_1}}{m_u}}$. 
\end{proposition}

\begin{proposition}[Restricted Strong Convexity for \textsc{D-RISE}] \label{prop:restricted_strong_convexity_GD-ISO} 
For some node $u \in V$, let $m_u~\geq~2^{11}~d^2~\ln\frac{n^2}{\delta_2}$, then with probability at least $1-\delta_2$, the residual of the first order Taylor expansion of D-ISO satisfies
\begin{equation}
    \delta S_m (\Delta_u, \underline{J}_u^\star) \geq \exp(-\beta d) \frac{\norm{\Delta_u}_2^2}{4(1 + 2\sqrt{d} R)},
\end{equation}
for all $\norm{\Delta_u}_2 \leq R$.
\end{proposition}

\begin{proof}[Proof of Theorem~\ref{thm:estimate_drise_drple}(ii)]
 Error bound on D-RISE: Let $\delta_1 = \frac{2\delta}{3n}$ and $\delta_2 = \frac{\delta}{3n}$ and $m_u~\geq~2^{14} d^2 \exp(2\beta~d)\ln\frac{3n^3}{\delta}$. Consider any node $u \in V$ and let $\hat{\underline{J}}_u$ be an optimal point of the D-ISO and $\Delta = \hat{\underline{J}}_u - \underline{J}_u^\star$.
 Using the values of $\lambda$ and $\kappa$ found in Proposition~\ref{prop:concentration_gradient_GD-ISO} and Proposition~\ref{prop:restricted_strong_convexity_GD-ISO}, we will look for values of $R$ that satisfy
 \begin{align}
     R > 3 \sqrt{d}\frac{\lambda}{\kappa} = 12 \sqrt{d} \lambda (1 + 2 \sqrt{d} R) \exp(\beta d).
     \label{R_equation_SGD}
 \end{align}
The above inequality is satisfied for $R=2/\sqrt{d}$. Therefore, we can apply Proposition~\ref{prop:m_estimator_error_bound} for each node $u$ and using the union bound, we find that the error is bounded in $\ell_2$-norm with probability at least $1-\delta$ for all nodes by the following quantity,
\begin{align}
    \norm{\hat{\underline{J}}_u - \underline{J}_u^\star}_2 \leq 240 \sqrt{d} \exp(\beta d) \sqrt{\frac{\ln \frac{3n^3}{\delta}}{m_u}}.
\end{align}

% Error bound on D-RISE: Consider node $u \in V$. Estimate $\hat{\underline{J}}_u$ is an $\epsilon$-optimal point of the D-ISO and noting $\Delta = \hat{\underline{J}}_u - \underline{J}_u^\star$, we have
% %
% \begin{align}
%     \epsilon &\geq S_m(\hat{\underline{J}}_u) - S_m(\underline{J}_u^\star) \\
%     &= \langle \nabla S_m(\underline{J}_u^\star), \Delta_u \rangle + \delta S_m (\Delta_u, \underline{J}_u^\star) \\
%     & \geq -|| \nabla S_m(\underline{J}_u^\star)||_\infty ||\Delta_u||_1 + \delta S_m (\Delta_u, \underline{J}_u^\star)
% \end{align}
% %
% Using the union bound on Proposition~\ref{prop:concentration_gradient_GD-ISO} and Proposition~\ref{prop:restricted_strong_convexity_GD-ISO} with $\delta_1 = \frac{\delta}{3}$ and $\delta_2 = \frac{2 \delta}{3}$, we get the following inequality with probability at least $1-\delta$
% %
% \begin{align}
%     \epsilon & \geq - \epsilon_1 ||\Delta_u||_1 + \exp(-\beta d) \frac{\norm{\Delta_u}_2^2}{4(1 + 2\sqrt{d} R)}
% \end{align}
% %
% To use Proposition~\ref{prop:m_estimator_error_bound}, we require
% %
% \begin{equation}
%     12 \sqrt{d} \lambda (1 + 2 \sqrt{d} R) \exp(\beta d) < R
% \end{equation}
% %
% which is satisfied for $R=2/\sqrt{d}$ and $m_u \geq 2^{14} d^2 \exp(2 \beta d) \ln \frac{3n^2}{\delta}$. This completes the proof.
 \end{proof}

\subsubsection{D-RPLE Estimator}
\begin{proposition} [Gradient concentration of \textsc{D-RPLE}] \label{prop:concentration_gradient_D-PL} 
For some node $u \in V$, let $m_u~\geq~\exp(2\beta~d)\ln\frac{2n}{\delta_1}$, then with probability at least $1-\delta_1$, the components of the gradient of the D-PL are bounded from above as
\begin{equation}
    \norm{\nabla \mathcal{L}_m(\underline{J}_u^\star)}_\infty \leq \epsilon_1
\end{equation}
where $\epsilon_1 = 2\sqrt{2} \sqrt{\frac{\ln \frac{2n}{\delta_1}}{m_u}}$. 
%\tcblue{$\epsilon_1$ depends on $\delta_1$ and $m$?}
\end{proposition}

\begin{proposition}[Restricted Strong Convexity for \textsc{D-RPLE}] \label{prop:restricted_strong_convexity_D-PL} 
For some node $u \in V$, let $m_u~\geq~2^{11}~d^2~\ln\frac{n^2}{\delta_2}$, then with probability at least $1-\delta_2$, the residual of the first order Taylor expansion of D-PL satisfies
\begin{equation}
    \delta \mathcal{L}_m (\Delta_u, \underline{J}_u^\star) \geq \exp(-2 \beta d) \frac{\norm{\Delta_u}_2^2}{4(1 + 4\sqrt{d} R)}
\end{equation}
\end{proposition}

 \begin{proof}[Proof of Theorem~\ref{thm:estimate_drise_drple}(i)]
 Error bound on D-RPLE: Let $\delta_1 = \frac{2 \delta}{3n}$ and $\delta_2 = \frac{\delta}{3n}$ and $m_u~\geq~2^{17} d^2 \exp(4\beta~d)\ln\frac{3n^3}{\delta}$. Consider any node $u \in V$ and let $\hat{\underline{J}}_u$ be an optimal point of the D-PL and $\Delta = \hat{\underline{J}}_u - \underline{J}_u^\star$.
 Using the values of $\lambda$ and $\kappa$ found in Proposition~\ref{prop:concentration_gradient_D-PL} and Proposition~\ref{prop:restricted_strong_convexity_D-PL}, we will look for values of $R$ that satisfy
 \begin{align}
     R > 3 \sqrt{d}\frac{\lambda}{\kappa} = 12 \sqrt{d} \lambda (1 + 4 \sqrt{d} R) \exp(2 \beta d).
     \label{R_equation_SGD}
 \end{align}
The above inequality is satisfied for $R=1/\sqrt{d}$. Therefore, we can apply Proposition~\ref{prop:m_estimator_error_bound} and find that the error is bounded in $\ell_2$-norm by the following quantity,
\begin{align}
    \norm{\hat{\underline{J}}_u - \underline{J}_u^\star}_2 \leq 240\sqrt{2} \sqrt{d} \exp(2\beta d) \sqrt{\frac{\ln \frac{3n^3}{\delta}}{m_u}}.
\end{align}
The theorem follows by application of the union bound over all nodes.
\end{proof}

% \begin{proof}[Proof of Theorem~\ref{thm:estimate_drise_drple}(i)]
% Error bound on D-RPLE: Consider node $u \in V$. Estimate $\hat{\underline{J}}_u$ is an $\epsilon$-optimal point of the D-PL and noting $\Delta = \hat{\underline{J}}_u - \underline{J}_u^\star$, we have
% %
% \begin{align}
%     \epsilon &\geq \mathcal{L}_m(\hat{\underline{J}}_u) - \mathcal{L}_m(\underline{J}_u^\star) \\
%     &= \langle \nabla \mathcal{L}_m(\underline{J}_u^\star), \Delta_u \rangle + \delta \mathcal{L}_m (\Delta_u, \underline{J}_u^\star) \\
%     & \geq -|| \nabla \mathcal{L}_m(\underline{J}_u^\star)||_\infty ||\Delta_u||_1 + \delta \mathcal{L}_m (\Delta_u, \underline{J}_u^\star)
% \end{align}
% %
% Using the union bound on Proposition~\ref{prop:concentration_gradient_D-PL} and Proposition~\ref{prop:restricted_strong_convexity_D-PL} with $\delta_1 = \frac{\delta}{3}$ and $\delta_2 = \frac{2 \delta}{3}$, we get the following inequality
% %
% \begin{align}
%     \epsilon & \geq - \epsilon_1 ||\Delta_u||_1 + \exp(-2 \beta d) \frac{\norm{\Delta_u}_2^2}{4(1 + 4\sqrt{d} R)}
% \end{align}
% %
% To use Proposition~\ref{prop:m_estimator_error_bound}, we require
% %
% \begin{equation}
%     12 \sqrt{d} \lambda (1 + 4 \sqrt{d} R) \exp(2 \beta d) < R
% \end{equation}
% %
% which is satisfied for $R=4/\sqrt{d}$ and $m_u \geq 2^{18} d^2 \exp(4 \beta d) \ln \frac{3n^2}{\delta}$. This completes the proof.
% \end{proof}

\subsection{Proof of structure learning theorem}
\begin{proof}[Proof of Theorem~\ref{thm:structure_learning_drise_drple}] 
It is a simple application of Theorem~\ref{thm:estimate_drise_drple} for an error equal to $\alpha/2$.
% According to Theorem~\ref{thm:estimate_drise_drple}, the minimal number of samples required to achieve an error of $\alpha/2$ on every coupling around some node $u \in V$ with probability $1 - \delta'$ is given by
% %
% \begin{align}
%     \text{D-RPLE: }& m_u \geq \max \left( \frac{d}{16}, \alpha^{-2} \right) 2^{22} d^2 \exp(4 \beta d) \ln \frac{3n^2}{\delta'} \\
%     \text{D-RISE: }& m_u \geq \max \left( \frac{d}{16}, \alpha^{-2} \right) 2^{18} d^2 \exp(2 \beta d) \ln \frac{3n^2}{\delta'}
% \end{align}
% %
% To ensure that the set of edges is recovered perfectly with probability $1 - \delta$ after using the union bound over all nodes, set $\delta' = \delta/n$. We then require the number of samples $m_u$ to satisfy
% %
% \begin{align}
%     \text{D-RPLE: }& m_u \geq \max \left( \frac{d}{16}, \alpha^{-2} \right) 2^{22} d^2 \exp(4 \beta d) \ln \frac{3n^3}{\delta} \\
%     \text{D-RISE: }& m_u \geq \max \left( \frac{d}{16}, \alpha^{-2} \right) 2^{18} d^2 \exp(2 \beta d) \ln \frac{3n^3}{\delta}
% \end{align}
%
\end{proof}

\subsection{Gradient Concentration}
\subsubsection{D-RISE Estimator}
Gradient of D-ISO (Eq.~\ref{eq:gd_iso_sim}) is given by:
\begin{equation}
    \frac{\partial}{\partial J_{uk}} S_m(\underline{J}_u) = \frac{1}{m_u} \sum \limits_{t=1}^{m} -{\sigma_{u}^1}^{(t)} {\sigma_k^0}^{(t)} \exp \left[ - \sum \limits_{i \in V \setminus u} J_{ui} {\sigma_{u}^1}^{(t)} {\sigma_i^0}^{(t)} \right] \delta_{u, {I^1}^{(t)}}
    \label{eq:gradient_GD_Ising}
\end{equation}
where $m_u = \sum \limits_{t=1}^m \delta_{u, {I^1}^{(t)}}$. Let us denote the term in the above summation as the following random variable
\begin{equation}
    X_{uk}(\underline{J}_u) = -{\sigma_{u}^1} {\sigma_k^0} \exp \left[ - \sum \limits_{i \in V \setminus u} J_{ui} {\sigma_{u}^1} {\sigma_i^0} \right] \, \forall k \in \partial u
    \label{eq:X_uk_GD_Ising}
\end{equation}
The lemma below indicates that the D-RISE estimator is consistent and unbiased regardless of the choice of $p(\underline{\sigma}^0)$. This will also be useful for the concentration inequality to come after.
\begin{lemma} \label{lem:consistency_GD-ISO} 
For any $u \in V$ and $k \in V \setminus u$, we have
\begin{equation}
    \expectation[X_{uk}(\underline{J}_u^\star)] = 0
\end{equation}
\end{lemma}
\begin{proof}[Proof of Lemma~\ref{lem:consistency_GD-ISO}]
Let us note the probability distribution with respect to which we take the expectation.
\begin{align}
    & \expectation_{p(\sigma_u^{1}, \underline{\sigma}^0| \delta_{u,I^{1}} = 1)}[X_{uk} (\underline{J}_u^\star)] \\ = & \expectation_{p(\sigma_u^{1}| \underline{\sigma}^0, \delta_{u,I^{1}}=1)p(\underline{\sigma}^0)}[X_{uk} (\underline{J}_u^\star) ] \\
    = & \sum \limits_{\underline{\sigma}^0} \left[ \sum_{\sigma_u^{1}} p(\sigma_u^{1}| \underline{\sigma}^0, \delta_{u,I^{1}}=1)p(\underline{\sigma}^0) X_{uk}(\underline{J}_u^\star) \right]\\
    = & \sum \limits_{\underline{\sigma}^0 } p(\underline{\sigma}^0) \left[ \sum_{\sigma_u^{1}} p(\sigma_u^{1}| \underline{\sigma}^0, \delta_{u,I^{1}}=1) X_{uk}(\underline{J}_u^\star) \right] \\
    = &  \sum \limits_{\underline{\sigma}^0} \frac{p(\underline{\sigma}^0)}{2 \cosh \left( \sum \limits_{j \in \partial u} J_{uj}^\star \sigma_j^0 \right)} \left[ \sum_{\sigma_u^{1}} -\sigma_u^1 \sigma_k^0 \exp \left( \sigma_u^1 (\sum_{j \in \partial u} J_{uj}^\star \sigma_j^0 - \sum_{l \in \partial u} J_{ul}^\star \sigma_l^0 ) \right) \right] \\
    =& 0
\end{align}
where in the first step, we used the law of total expectations. In the second to last step, we used the definition of $X_{uk}$ from Eq.~\ref{eq:X_uk_GD_Ising} and the conditional probability from Eq.~\ref{eq:glauber_dynamics_conditional_prob_update_sim}.
\end{proof}

\begin{lemma} \label{lem:variance_GD-ISO} 
For any Ising model with $n$ spins considering $u \in V$ and $k \in V \setminus u$, we have
\begin{equation}
    \expectation[X_{uk}(\underline{J}_u^\star)^2] = 1
\end{equation}
\end{lemma}
\begin{proof}[Proof of Lemma~\ref{lem:variance_GD-ISO}] From direct computation, we have
\begin{align}
    & \expectation_{p(\sigma_u^{1}, \underline{\sigma}^0| \delta_{u,I^{1}} = 1)}[X_{uk}(\underline{J}_u^\star)^2 ] \\ 
    =& \sum \limits_{\underline{\sigma}^0 } p(\underline{\sigma}^0) \left[ \sum_{\sigma_u^{1}} p(\sigma_u^{1}| \underline{\sigma}^0, \delta_{u,I^{1}}=1) X_{uk}(\underline{J}_u^\star)^2 \right] \\
    =&  \sum \limits_{\underline{\sigma}^0} \frac{p(\underline{\sigma}^0)}{2 \cosh \left( \sum \limits_{j \in \partial u} J_{uj}^\star \sigma_j^0 \right)} \left[ \sum_{\sigma_u^{1}} \exp \left( - \sigma_u^1 \sum_{j \in \partial u} J_{uj}^\star \sigma_j^0 \right) \right] \\
    =& \sum \limits_{\underline{\sigma}^0} p(\underline{\sigma}^0) \\
    =& 1
\end{align}
where in the second step we noted that $(\sigma_i^1)^2=(\sigma_i^0)^2=1$ when substituting for $X_{uk}$ from Eq.~\ref{eq:X_uk_GD_Ising}.
\end{proof}

\begin{lemma} \label{lem:bound_Xuk_GD_Ising} 
For any Ising model with $n$ spins with maximum degree $d$ and maximum interaction strength $\beta$, we have for $k \neq u \in V $, we have
\begin{equation}
    |X_{uk}(\underline{J}_u^\star)| \leq \exp(\beta d)
\end{equation}
\end{lemma}
\begin{proof}[Proof of Lemma~\ref{lem:bound_Xuk_GD_Ising}]
\begin{align}
    |X_{uk}(\underline{J}_u^\star)| &= \left|-{\sigma_{u}^1} {\sigma_k^0} \exp \left( - \sum \limits_{i \in \partial u} J_{ui}^\star {\sigma_{u}^1} {\sigma_i^0} \right) \right| \\
    &= \exp \left( - \sum \limits_{i \in \partial u} J_{ui}^\star {\sigma_{u}^1} {\sigma_i^0} \right) \\
    &\leq \exp(\beta d)
\end{align}
where we firstly noted that $\delta_{u,I^1} \in \{0,1\}$. In the second step, we noted that the components of $\underline{J}^\star_u$ have a maximum value of $\beta$ and at most $d$ of them are non-zero. Further $|\sigma_u^1 \sigma_i^0|=1$ as spins take values in $\{-1, +1\}$.
\end{proof}

In the M-regime, the different tuples of realizations of $(\underline{\sigma}^1, \underline{\sigma}^0, I^1)$ are independent of each other. This allows us to use the lemmas obtained above and obtain a concentration inequality in the following proof.

\begin{proof}[Proof of Proposition~\ref{prop:concentration_gradient_GD-ISO}] Utilizing Lemmas~\ref{lem:consistency_GD-ISO},~\ref{lem:variance_GD-ISO} and \ref{lem:bound_Xuk_GD_Ising} combined with Bernstein's inequality, we have
\begin{align}
    p \left[\left| \frac{\partial}{\partial J_{uk}} S_m(\underline{J}_u^\star) \right| > a \right] \leq 2 \exp \left(-\frac{\frac{1}{2}a^2 m_u}{1 + \frac{1}{3}\exp(\beta d) a} \right)
\end{align}
Setting 
\begin{equation}
    s=\frac{\frac{1}{2}a^2 m_u}{1 + \frac{1}{3}\exp(\beta d) a},
\end{equation}
and considering $z = \frac{s}{m_u} \exp(\beta d)$, we can write
\begin{equation}
    p \left[\left| \frac{\partial}{\partial J_{uk}} S_m(\underline{J}_u^\star) \right| > \frac{1}{3} \left( z + \sqrt{ \frac{18z}{\exp(\beta d)} + z^2} \right) \right] \leq 2 \exp(-s)
    \label{eq:interim_prop_concentration_grad_Ising}
\end{equation}
For $m_u \geq s \exp(2 \beta d)$, we have $z^2 = \frac{s^2}{m_u^2}\exp(2 \beta d) \leq \frac{s}{m_u}$ and that
\begin{align}
    \frac{1}{3} \left( z + \sqrt{ \frac{18z}{ \exp(\beta d)} + z^2} \right) &\leq \frac{1}{3} \left(\sqrt{\frac{s}{m_u}} + \sqrt{\frac{18s}{m_u} + \frac{s}{m_u}} \right) \\
    &\leq \frac{\sqrt{19} + 1}{3} \sqrt{\frac{s}{m_u}} \\
    &\leq 2 \sqrt{\frac{s}{m_u}}
\end{align}
which allows us to simplify Eq.~\ref{eq:interim_prop_concentration_grad_Ising}:
\begin{equation}
    p \left[\left| \frac{\partial}{\partial J_{uk}} S_m(\underline{J}_u^\star) \right| > 2 \sqrt{\frac{s}{m_u}} \right] \leq 2 \exp(-s)
    \label{eq:interim2_prop_concentration_grad_Ising}
\end{equation}
Setting $s = \ln \frac{2n}{\delta_1}$ and taking the union bound over every component of the gradient gives us the desired result.
\end{proof}

\subsubsection{D-RPLE Estimator}
Gradient of D-PL (Eq.~\ref{eq:gd_pl_sim}) is given by:
\begin{equation}
    \frac{\partial}{\partial J_{uk}} \mathcal{L}_m(\underline{J}_u) = \frac{1}{m_u} \sum_{t=1}^{m} {\sigma_k^0}^{(t)}\left[\tanh\left(\sum_{i \neq u} J_{ui} {\sigma_i^0}^{(t)} \right) - {\sigma_u^1}^{(t)} \right] \delta_{u,{{I^1}^{(t)}}}.  \label{eq:gradient_DPL}
\end{equation}

Let us denote the term in the above summation as the following random variable
\begin{equation}
    Z_{uk}(\underline{J}_u) = {\sigma_k^0}\left[\tanh\left(\sum_{i \neq u} J_{ui} {\sigma_i^0} \right) - {\sigma_u^1} \right] \, \forall k \in \partial u
    \label{eq:Z_uk_DPL_Ising}
\end{equation}
The lemma below indicates that the D-RPLE estimator is consistent and unbiased regardless of the choice of $p(\underline{\sigma}^0)$. This will also be useful for the concentration inequality to come after.
\begin{lemma} \label{lem:consistency_D-PL} 
For any $u \in V$ and $k \in V \setminus u$, we have
\begin{equation}
    \expectation[Z_{uk}(\underline{J}_u^\star)] = 0
\end{equation}
\end{lemma}
\begin{proof}[Proof of Lemma~\ref{lem:consistency_D-PL}]
Let us note the probability distribution with respect to which we take the expectation.
\begin{align}
    & \expectation_{p(\sigma_u^{1}, \underline{\sigma}^0| \delta_{u,I^{1}} = 1)}[Z_{uk} (\underline{J}_u^\star) ] \\
    = & \sum \limits_{\underline{\sigma}^0} \left[ \sum_{\sigma_u^{1}} p(\sigma_u^{1}| \underline{\sigma}^0, \delta_{u,I^{1}}=1)p(\underline{\sigma}^0) Z_{uk}(\underline{J}_u^\star) \right]\\
    = & \sum \limits_{\underline{\sigma}^0 } p(\underline{\sigma}^0) \left[ \sum_{\sigma_u^{1}} p(\sigma_u^{1}| \underline{\sigma}^0, \delta_{u,I^{1}}=1) Z_{uk}(\underline{J}_u^\star) \right] \\
    = &  \sum \limits_{\underline{\sigma}^0} p(\underline{\sigma}^0)\left[ \sum_{\sigma_u^{1}} \frac{\exp \left[\sigma_i^1 (\sum_{j \in \partial u} J^\star_{uj} \sigma_j^0 ) \right]}{2 \cosh \left[ \sum_{j \in \partial u} J^\star_{uj} \sigma_j^0 \right]} {\sigma_k^0} \left(\tanh\left(\sum_{i \neq u} J_{ui} {\sigma_i^0} \right) - {\sigma_u^1} \right) \right] \\
    =& \sum \limits_{\underline{\sigma}^0} p(\underline{\sigma}^0) {\sigma_k^0} \left[ \tanh\left(\sum_{i \neq u} J_{ui} {\sigma_i^0}\right) - \sum_{\sigma_u^{1}} \frac{\exp \left[\sigma_i^1 (\sum_{j \in \partial u} J^\star_{uj} \sigma_j^0 ) \right]}{2 \cosh \left[ \sum_{j \in \partial u} J^\star_{uj} \sigma_j^0 \right]} {\sigma_u^1}  \right] \\
    =& \sum \limits_{\underline{\sigma}^0} p(\underline{\sigma}^0) {\sigma_k^0} \left[ \tanh\left(\sum_{i \neq u} J_{ui} {\sigma_i^0} \right) - \tanh\left(\sum_{i \neq u} J_{ui} {\sigma_i^0} \right)  \right] \\
    =& 0
\end{align}
where in the first step, we used the law of total expectations. In the third to last step, we used the definition of $Z_{uk}$ from Eq.~\ref{eq:Z_uk_DPL_Ising} and the conditional probability from Eq.~\ref{eq:glauber_dynamics_conditional_prob_update_sim}.
\end{proof}

\begin{lemma} \label{lem:bound_Zuk_D-PL_Ising} 
For any Ising model with $n$ spins with maximum degree $d$ and maximum interaction strength $\beta$, we have for $k \neq u \in V $, we have
\begin{equation}
    |Z_{uk}(\underline{J}_u^\star)| \leq 2
\end{equation}
\end{lemma}
\begin{proof}[Proof of Lemma~\ref{lem:bound_Zuk_D-PL_Ising}]
\begin{align}
    |Z_{uk}(\underline{J}_u^\star)| &= \left|{\sigma_k^0} \left[\tanh\left(\sum_{i \neq u} J_{ui} {\sigma_i^0} \right) - {\sigma_u^1} \right]\right| \\
    &\leq \left|\tanh\left(\sum_{i \neq u} J_{ui} {\sigma_i^0} \right)\right| + \left| {\sigma_u^1} \right| \\
    &\leq 2
\end{align}
where in the last step, we noted that $|\tanh(\cdot)| \leq 1$ and $|\sigma_u^1|=1$ as spins take values in $\{-1, +1\}$.
\end{proof}

As noted before, the different tuples of realizations of $(\underline{\sigma}^1, \underline{\sigma}^0, I^1)$ are independent of each other in the M-regime. We use the lemmas obtained above and obtain a concentration inequality in the following proof.

\begin{proof}[Proof of Proposition~\ref{prop:concentration_gradient_D-PL}] Utilizing Lemmas~\ref{lem:consistency_D-PL} and \ref{lem:bound_Zuk_D-PL_Ising} combined with Hoeffding's inequality, we have
\begin{align}
    p \left[\left| \frac{\partial}{\partial J_{uk}} \mathcal{L}_m(\underline{J}_u^\star) \right| > a \right] \leq 2 \exp \left(-\frac{m_u a^2}{8} \right)
\end{align}
Let $a = 2\sqrt{2} \sqrt{\frac{s}{m_u}}$. We then have
\begin{equation}
    p \left[\left| \frac{\partial}{\partial J_{uk}} \mathcal{L}_m(\underline{J}_u^\star) \right| > 2\sqrt{2} \sqrt{\frac{s}{m_u}} \right] \leq 2 \exp \left(-s\right)
\end{equation}
Setting $s = \ln \frac{2n}{\delta_1}$ and taking the union bound over every component of the gradient gives us the desired result.
\end{proof}

\subsection{Restricted Strong Convexity}
The following deterministic functional inequality derived in \cite{vuffray2016interaction} will be useful for latter results.
\begin{lemma} \label{lem:functional_ineq}
The following inequality holds for all $z \in \mathbb{R}$.
\begin{align}
    e^{-z} - 1 + z \geq \frac{z^2}{2 + |z|}.
\end{align}
\end{lemma}

We will also find the following deterministic functional inequality useful.
\begin{lemma} \label{lem:functional_ineq_log_tanh}
The following inequality holds for all $z \in \mathbb{R}$ and for any $x \in \mathbb{R}$.
\begin{align}
    -\ln \left(1 + \tanh(x+z) \right) + \ln \left(1 + x \right) + z\left( 1 -\tanh(x) \right) \geq \frac{z^2}{2(1 + |z|)} \sech^2(x).
\end{align}
\end{lemma}
\begin{proof}[Proof of Lemma~\ref{lem:functional_ineq_log_tanh}]
Let us denote the function
\begin{equation}
    f(x,z) := -\ln \left(1 + \tanh(x+z) \right) + \ln \left(1 + x \right) + z\left( 1 -\tanh(x) \right)
\end{equation}
and the auxillary function
\begin{equation}
    g(x,z) := 2(1+|z|)f(x,z) - \sech^2(x) z^2.
\end{equation}
We show that for fixed $x \in \mathbb{R}$ and $z=0$, $g(x,z)$ achieves its minimum at $g(x,0)=0$. Observe the first partial derivative of $g(x,z)$ wrt $z$ is given by
\begin{equation}
    \frac{\partial}{\partial z}g(x,z) = 
    \begin{cases}
        2f(x,z) + 2(1+z) \left(\tanh(x+z) - \tanh(x)\right) -2z \sech^2(x), \, \quad z > 0 \\
        -2f(x,z) + 2(1-z) \left(\tanh(x+z) - \tanh(x)\right) -2z \sech^2(x), \, \quad z < 0
    \end{cases}
\end{equation} 
We note that the partial derivatives vanishes at zero from both the negative and positive directions of $z$:
\begin{equation}
    \lim_{z \rightarrow 0_{+}} \frac{\partial}{\partial z}g(x,z) = \lim_{z \rightarrow 0_{-}} \frac{\partial}{\partial z}g(x,z) = 0
    \label{eq:g_grad_zero}
\end{equation}
The second partial derivatives of $g(x,z)$ are given by
\begin{equation}
    \frac{\partial^2}{\partial z^2}g(x,z) = 
    \begin{cases}
        4 \left(\tanh(x+z) - \tanh(x)\right) + 2 (1 + z) \sech^2(x+z) -2 \sech^2(x), \, \quad z > 0 \\
        4 \left(\tanh(x) - \tanh(x-z)\right) + 2 (1 - z) \sech^2(x+z) -2 \sech^2(x), \, \quad z < 0
    \end{cases}
\end{equation} 
We note that for $z > 0$, $\frac{\partial^2}{\partial z^2}g(x,z)$ is non-negative
\begin{align}
    \frac{\partial^2}{\partial z^2}g(x,z) &= 4 \left(\tanh(x+z) - \tanh(x)\right) + 2 (1 + z) \sech^2(x+z) -2 \sech^2(x) \\
    & > 4 \left(\tanh(x+z) - \tanh(x)\right) + 2 \sech^2(x+z) -2 \sech^2(x) \\
    & = 4 \left(\tanh(x+z) - \tanh(x)\right) + 2\left(\tanh^2(x) -2 \tanh^2(x+z)\right) \\
    & = 2 \left(\tanh(x+z) - \tanh(x)\right)\left(2 - \tanh(x) - \tanh(x+z)\right) \\
    & > 0
\end{align}
where in the last step, we noted that $\tanh(x)$ is a monotonically increasing function and that $\tanh(x) < 1$. A similar result can be shown for $z<0$. Combining this with Eq.~\ref{eq:g_grad_zero}, we prove that for all $z$ and any $x$, $g(x,z) \geq g(x,0) = 0$ which gives us our desired result.
\end{proof}

Let $H_{ij}$ denote the correlation matrix elements
\begin{equation}
    H_{ij} = \expectation_{p(\sigma_u^1 , \underline{\sigma}^0 |\delta_{u,I^{1}}=1)}[\sigma_i^{0} \sigma_j^{0}]
    \label{eq:correlation_GD_Ising}
\end{equation}
and we denote the corresponding matrix as $H = [H_{ij}] \in \mathbb{R}^{|\partial u| \times |\partial u|}$. Let the empirical estimate of the correlation matrix be denoted by $\hat{H}$ and the matrix elements be given by
\begin{equation}
    \hat{H}_{ij} = \frac{1}{m_u} \sum \limits_{t=1}^{m} {\sigma_i^0}^{(t)} {\sigma_j^0}^{(t)} \delta_{u,{I^{1}}^{(t)}}
    \label{eq:empirical_correlation_GD_Ising}
\end{equation}

\begin{lemma}\label{lem:ineq_hessian_GD-ISO}
Consider some node $u \in V$. With probability at least $1 - n^2 \exp \left(- \frac{m_u \epsilon_2^2}{2}\right)$, we have
\begin{equation}
    | \hat{H}_{ij} -  H_{ij} | \leq \epsilon_2
\end{equation}
for all $i,j \in \partial u$ and $\epsilon_2 > 0$.
\end{lemma}
\begin{proof}[Proof of Lemma~\ref{lem:ineq_hessian_GD-ISO}]
Fix $i,j \in \partial u$. Noting that $|\sigma_i^0 \sigma_j^0 \delta_{u, I^1}| \leq 1$ and combining with Hoeffding's inequality:
\begin{equation}
    p(| \hat{H}_{ij} -  H_{ij} | > \epsilon_2) \leq 2 \exp \left( - \frac{m_u \epsilon_2^2}{2} \right)
\end{equation}
The proof follows by noting that the matrix $H$ is symmetric allowing us to take the union bound over all $i < j \in V \setminus u$. 
\end{proof}

\begin{lemma} \label{lem:hessian_GD-ISO} 
Consider an Ising model with $n$ spins and some node $u \in V$, then the following holds for the Hessian
\begin{align}
    \Delta_u^T H \Delta_u = \norm{\Delta_u}_2^2
\end{align}
\end{lemma}
\begin{proof}[Proof of Lemma~\ref{lem:hessian_GD-ISO}] From direct evaluation of $H$, we have
\begin{align}
    H_{ij} &= \expectation[\sigma_i^0 \sigma_j^0] \\ 
    &= \sum \limits_{\underline{\sigma}^0} p(\underline{\sigma}^0) \sum \limits_{\sigma_u^1} p(\sigma_u^1 | \underline{\sigma}^0, \delta_{u, I^1} = 1) \sigma_i^0 \sigma_j^0 \\
    &= \sum \limits_{\underline{\sigma}^0} p(\underline{\sigma}^0) \sigma_i^0 \sigma_j^0 \\
    &= \delta_{ij}
\end{align}
where $\delta_{ij}=1$ iff $i=j$. Thus the correlation matrix $H=I$ where $I$ is an identity matrix of size $\mathbb{R}^{|\partial u| \times |\partial u|}$. We assumed that the initial distribution of $p(\underline{\sigma}^0)$ is given by Eq.~\ref{eq:initial_pdf_GD_Ising}. The proof follows from computation of $\Delta_u^T H \Delta_u$.
\end{proof}

\subsubsection{D-RISE Estimator}
\begin{lemma} \label{lem:residual_GD-ISO} 
The residual of the first order Taylor expansion of the \textsc{D-ISO} satisfies
\begin{align}
    \delta S_m (\Delta_u, \underline{J}_u^\star) \geq \exp(-\beta d)\frac{\Delta_u^T \hat{H} \Delta_u}{2 + \norm{\Delta_u}_1}.
\end{align}
\end{lemma}
\begin{proof}[Proof of Lemma~\ref{lem:residual_GD-ISO}] Noting the expression of the residual from Eq.~\ref{eq:first_order_taylor_expansion_GD-RISE} and that $\langle \nabla S_m(\underline{J}_u^\star), \Delta_u \rangle = \sum \limits_{k \in \partial u} \left[ \frac{\partial}{\partial J_{uk}} S_m ( \underline{J}_u^\star) \right] \Delta_{uk}$, we have that
\begin{align}
    \nonumber
    \delta S_m (\Delta_u, \underline{J}_u^\star) &= \frac{1}{m_u}\sum_{t=1}^m \exp \left(-\sum_{k \in \partial u} J_{uk}^\star {\sigma_u^1}^{(t)} {\sigma_k^0}^{(t)}\right) \delta_{u, {I^1}^{(t)}} \times \\ 
    & \qquad
    \left[ \exp\left(-\sum_{k \in \partial u}\Delta_{uk} {\sigma_u^1}^{(t)} {\sigma_k^0}^{(t)} \right) - 1 + \sum_{k \in \partial u} \Delta_{uk} {\sigma_u^1}^{(t)} {\sigma_k^0}^{(t)} \right] \\
    & \geq \exp(-\beta d) \frac{1}{m_u} \sum_{t=1}^m \frac{ \left(\sum \limits_{k \in \partial u}\Delta_{uk} {\sigma_u^1}^{(t)} {\sigma_k^0}^{(t)}\right)^2}{2 + |\sum \limits_{k \in \partial u}\Delta_{uk} {\sigma_u^1}^{(t)} {\sigma_k^0}^{(t)}|} \delta_{u,{I^{1}}^{(t)}} \\
    & \geq \exp(-\beta d) \frac{\Delta^T \hat{H} \Delta}{2 + \|\Delta_u\|_1}.
    \label{eq:residual_inequality_GD-ISO}
\end{align}
where in the second step we used Lemma~\ref{lem:functional_ineq} considering $z=\sum \limits_{k \in \partial u} \Delta_{uk} {\sigma_u^1}^{(t)} {\sigma_k^0}^{(t)}$ and noted $|\sum \limits_{i \in \partial u} J_{ui} {\sigma_i^0}^{(t)} {\sigma_u^1}^{(t)}| \leq \beta d$. The proof follows from using the definition of $\hat{H}$ and observing that $|\sum \limits_{k \in \partial u} \Delta_{uk} {\sigma_u^1}^{(t)} {\sigma_k^0}^{(t)}| \leq \|\Delta_u\|_1$.
\end{proof}

\begin{proof}[Proof of Proposition~\ref{prop:restricted_strong_convexity_GD-ISO}] Using Lemma~\ref{lem:residual_GD-ISO} and for $m_u~\geq~\frac{2}{\epsilon_2^2}~\ln\frac{n^2}{\delta_2}$, the residual satisfies the following with probability at least $1-\delta_2$:
\begin{align}
    \delta S_m (\Delta_u, \underline{J}_u^\star) &\geq \exp(-\beta d)\frac{\Delta_u^T \hat{H} \Delta_u}{2 + \norm{\Delta_u}_1} \\
    &= \exp(-\beta d)\frac{\Delta_u^T H \Delta_u + \Delta_u^T (H - \hat{H}) \Delta_u}{2 + \norm{\Delta_u}_1} \\
    &\overset{(a)}{\geq} \exp(-\beta d) \frac{ \Delta_u^T H \Delta_u - \epsilon_2 \norm{\Delta_u}_1^2}{2 + \norm{\Delta_u}_1} \\
    &\overset{(b)}{\geq} \exp(-\beta d) \frac{ \norm{\Delta_u}_2^2 - \epsilon_2 \norm{\Delta_u}_1^2}{2 + \norm{\Delta_u}_1}
\end{align}
where we used Lemma~\ref{lem:ineq_hessian_GD-ISO} in $(a)$ and Lemma~\ref{lem:hessian_GD-ISO} in $(b)$. Setting $\epsilon_2 = \frac{1}{32d}$, we note that
\begin{align}
    -\epsilon_2 \norm{\Delta_u}_1^2 \geq - \frac{1}{2} \norm{\Delta_u}_2^2
\end{align}
where we have used Condition~\ref{eq:condition2} that $\norm{\Delta_u}_1 \leq 4 \sqrt{d} \norm{\Delta_u}_2$. Combining this with $\norm{\Delta_u}_2 \leq R$, we obtain the desired result.
\end{proof}

\subsubsection{D-RPLE Estimator}
Consider the first order Taylor expansion:
\begin{equation}
    \delta \mathcal{L}_m (\Delta, \underline{J}_u^\star) = \mathcal{L}_m(\underline{J}_u^\star + \Delta) - \mathcal{L}_m(\underline{J}_u^\star) - \langle \nabla \mathcal{L}_m(\underline{J}_u^\star), \Delta \rangle.
    \label{eq:first_order_taylor_expansion_D-RPLE}
\end{equation}

\begin{lemma} \label{lem:residual_D-PL} 
The residual of the first order Taylor expansion of the \textsc{D-PL} satisfies
\begin{align}
    \delta \mathcal{L}_m (\Delta_u, \underline{J}_u^\star) \geq \exp(- 2 \beta d)\frac{\Delta_u^T \hat{H} \Delta_u}{2(1 + \norm{\Delta_u}_1)}.
\end{align}
\end{lemma}
\begin{proof}[Proof of Lemma~\ref{lem:residual_D-PL}] Noting the expression of the residual from Eq.~\ref{eq:first_order_taylor_expansion_D-RPLE} and that $\langle \nabla \mathcal{L}_m(\underline{J}_u^\star), \Delta_u \rangle = \sum \limits_{k \in \partial u} \left[ \frac{\partial}{\partial J_{uk}} \mathcal{L}_m ( \underline{J}_u^\star) \right] \Delta_{uk}$, we have that
\begin{align}
    \nonumber
    \delta \mathcal{L}_m (\Delta_u, \underline{J}_u^\star) &= \frac{1}{m_u} \sum_{t=1}^m \left[ - \ln \left[ 1 + {\sigma_u^1}^{(t)} \tanh\left(\sum_{i \neq u} (J_{ui}+\Delta_{ui}) {\sigma_i^0}^{(t)} \right) \right] + \right. \\ & \left. \ln \left[1 + {\sigma_u^1}^{(t)} \tanh\left(\sum_{i \neq u} J_{ui} {\sigma_i^0}^{(t)} \right) \right] + \sum_{k \in \partial u} \Delta_{uk} {\sigma_k}^{(0)}{\sigma_u^1}^{(t)} \left[1 - {\sigma_u^1}^{(t)}\tanh\left(\sum_{i \neq u} J_{ui} {\sigma_i^0}^{(t)} \right) \right] \right] \delta_{u, {I^1}^{(t)}} \\
    & \geq \frac{1}{m_u} \sum_{t=1}^m \frac{ \left(\sum \limits_{k \in \partial u}\Delta_{uk} {\sigma_u^1}^{(t)} {\sigma_k^0}^{(t)}\right)^2}{2 \left(1 + |\sum \limits_{k \in \partial u}\Delta_{uk} {\sigma_u^1}^{(t)} {\sigma_k^0}^{(t)}| \right)} \sech^2 \left( \sum_{i \neq u} J_{ui} {\sigma_i^0}^{(t)} {\sigma_u^1}^{(t)} \right) \delta_{u, {I^1}^{(t)}} \\
    & \geq \exp(-2 \beta d) \frac{1}{m_u} \sum_{t=1}^m \frac{ \left(\sum \limits_{k \in \partial u}\Delta_{uk} {\sigma_u^1}^{(t)} {\sigma_k^0}^{(t)}\right)^2}{2 \left(1 + |\sum \limits_{k \in \partial u}\Delta_{uk} {\sigma_u^1}^{(t)} {\sigma_k^0}^{(t)}| \right)} \delta_{u, {I^1}^{(t)}} \\
    & \geq \exp(- 2 \beta d)\frac{\Delta_u^T \hat{H} \Delta_u}{2(1 + \norm{\Delta_u}_1)}
    \label{eq:residual_inequality_D-PL}
\end{align}
where in the second step, we used Lemma~\ref{lem:functional_ineq_log_tanh} combined with the fact that $\sigma_i \tanh(x) = \tanh(\sigma_i x)$ as $\sigma_i \in \{-1,+1\}$. In the third step, we noted that $\sech^2(x) = 1 - \tanh^2(x) \geq \exp(-2|x|)$ and $|\sum_{i \in \partial u} J_{ui} {\sigma_i^0}^{(t)} {\sigma_u^1}^{(t)}| \leq \beta d$. In the final step, we used the definition of $\hat{H}$ and that $|\sum_{k \in \partial u} \Delta_{uk} {\sigma_u^1}^{(t)} {\sigma_k^0}^{(t)}| \leq \|\Delta_u\|_1$ to complete the proof.
\end{proof}

\begin{proof}[Proof of Proposition~\ref{prop:restricted_strong_convexity_D-PL}] Using Lemma~\ref{lem:residual_D-PL} and for $m_u~\geq~\frac{2}{\epsilon_2^2}~\ln\frac{n^2}{\delta_2}$, the residual satisfies the following with probability at least $1-\delta_2$:
%N
\begin{align}
    \delta \mathcal{L}_m (\Delta_u, \underline{J}_u^\star) &\geq \exp(-2 \beta d) \frac{\Delta_u^T \hat{H} \Delta_u}{2(1 + \norm{\Delta_u}_1)} \\
    &= \exp(-2 \beta d) k_1 \frac{\Delta_u^T H \Delta_u + \Delta_u^T (H - \hat{H}) \Delta_u}{2(1 + \norm{\Delta_u}_1)} \\
    &\overset{(a)}{\geq} \exp(-2 \beta d) \frac{ \Delta_u^T H \Delta_u - \epsilon_2 \norm{\Delta_1}_u^2}{2(1 + \norm{\Delta_u}_1)} \\
    &\overset{(b)}{\geq} \exp(-2 \beta d) \frac{ \norm{\Delta_u}_2^2 - \epsilon_2 \norm{\Delta_u}_1^2}{2(1 + \norm{\Delta_u}_1)}
\end{align}
where we used Lemma~\ref{lem:ineq_hessian_GD-ISO} in $(a)$ and Lemma~\ref{lem:hessian_GD-ISO} in $(b)$. Setting $\epsilon_2 = \frac{1}{32d}$, we note that
\begin{align}
    -\epsilon_2 \norm{\Delta_u}_1^2 \geq - \frac{1}{2} \norm{\Delta_u}_2^2
\end{align}
where we have used Condition~\ref{eq:condition2} that $\norm{\Delta_u}_1 \leq 4 \sqrt{d} \norm{\Delta_u}_2$. Combining this with $\norm{\Delta_u}_2 \leq R$, we obtain the desired result.
\end{proof}

\section{Setup for numerical experiments}
\label{sec:opt_gd_estimators}
In this section, we describe the optimization techniques that we used in the implementation of D-RISE/D-RPLE estimators and the computer infrastructure that was used for running our numerical experiments.

\subsection{Optimization tools}
The D-RISE and D-RPLE estimators involve optimization of a convex objective function. There are a variety of optimization techniques that can be used for this purpose including gradient descent type methods and interior point methods. For our numerical experiments, we used the interior point Ipopt \cite{biegler2009large} optimization package within the JuMP modeling framework for mathematical optimization in Julia. 

% General convex solvers might use matrix inversion in which case the complexity for the reconstruction of a node $i$ might grow as $O(n^{3} m_u)$ (resulting in an algorithmic complexity for the whole problem $O(n^{4}m_u)$). This makes first-order composite gradient descent methods \cite{nesterov2013gradient, agarwal2012fast} preferable which can achieve a complexity of $O(n m_u)$ for single node reconstruction and $O(n^2 m_u)$ for the entire structure learning problem.

% Coordinate gradient descent method in theory also gives a complexity of $O(n m_u)$ for reconstruction of one node.
As an alternative to the Ipopt software for optimization, we also implemented the coordinate descent (CD) method for D-RISE and D-RPLE as a part of our package. The idea is to perform gradient descent only along one coordinate at a time, and then cycle through coordinates until convergence. For example, the optimization problem for each node $u$ in D-RISE is:
\begin{equation}
    (\widehat{\underline{J}}_u,\widehat{H}_u) = \argmin_{(\underline{J}_u,H_i)} \left[\frac{1}{m_u} \sum_{t=1}^{m} \exp \left(-\sigma_u^{t+1} \left(\sum_{i \neq u} J_{ui} \sigma_i^t + H_u\right) \right)\delta_{u,{I^{t+1}}}  + \lambda \Vert \underline{J}_u \Vert_{1} \right],
\end{equation}
In coordinate descent, we optimize over one variable only at a time, e.g. over $J_{uk}$ for some $k$. Once the minimum is found, we cycle through other components, and repeat until we reach the global minimum of the entire convex function. Interestingly, each iteration step is a solution of a one-dimensional optimization problem:
\begin{equation}
    \hat{x} = \argmin_{x} \cosh x - \kappa \sinh x + \mu \vert x \vert, 
\end{equation}
where the constant $\kappa$ and the regularization parameter $\mu$ is given by
\begin{align}
    \kappa = \frac{b}{a}, \quad \mu = 
    \begin{cases}
        0, & x = H_u \\
        \lambda/a, & \text{otherwise}.
    \end{cases}
\end{align}
with
\begin{align}
    a &= \begin{cases}
        \frac{1}{m_u} \sum \limits_{t=1}^{m} \exp \left(-\sigma_u^{t+1} \sum \limits_{i \neq u} J_{ui} \sigma_i^t \right) \delta_{u,{I^{t+1}}}, & x=H_u \\
        \frac{1}{m_u} \sum \limits_{t=1}^{m} \exp \left(-\sigma_u^{t+1} \left(\sum \limits_{i \neq u,k} J_{ui} \sigma_i^t + H_u\right) \right) \delta_{u,{I^{t+1}}}, & x=J_{uk} \forall k
    \end{cases}\\
    b &= \begin{cases}
        \frac{1}{m_u} \sum \limits_{t=1}^{m} \sigma_u^{t+1} \exp \left(-\sigma_u^{t+1} \sum \limits_{i \neq u} J_{ui} \sigma_i^t \right) \delta_{u,{I^{t+1}}}, & x=H_u \\
        \frac{1}{m_u} \sum \limits_{t=1}^{m} \sigma_u^{t+1} \sigma_k^{t} \exp \left(-\sigma_u^{t+1} \left(\sum \limits_{i \neq u,k} J_{ui} \sigma_i^t + H_u\right) \right) \delta_{u,{I^{t+1}}}, & x=J_{uk} \forall k
    \end{cases}
\end{align}
% there is no actual ``descent'' because one can write the expression for the minimum explicitly and directly ``jump'' into it:
and the minimum can be found explicitly (after soft-thresholding):
\begin{equation}
    \hat{x} = 
    \begin{cases}
        \log \left( \frac{\sqrt{1-\kappa^2 + \mu^2} - \mu \text{ sign}(\kappa)}{1 - \kappa}  \right), & \mu < |\kappa| \\
        0, & \text{otherwise.}
    \end{cases}
\end{equation}
Note that there is a slight difference in the optimization problem and solution depending on if the coordinate randomly chosen is a coupling parameter $J_{ui}$ or magnetic field $H_u$. As a result of having access to the above analytical solution, coordinate descent for D-RISE does not require the choice of the step in the gradient descent.
% , because the one-dimensional optimization can be solved explicitly.
This simplification however does not occur for D-RPLE.
% There is a general consensus in the literature that randomly choosing the coordinates at each step is beneficial for the convergence of the algorithm.
At each step of the descent, the updated coordinate is chosen at random.
% The computational complexity of this randomized coordinate descent method for the reconstruction of one node is $O(n m_u)$ \cite{shalev2011stochastic,nesterov2012efficiency,wright2015coordinate}. It achieves a computational complexity of $O(n^2 m_u)$ for the entire structure learning problem.

% Another optimization method that may preferred in terms of computational complexity is stochastic gradient descent (SGD) method.
% In theory, D-RISE/D-RPLE would have a complexity scaling of $O(n \ln n)$ (independent of $m_u$ for a given accuracy) \cite{shalev2011stochastic} for one node.
Another popular optimization method that may be applied to D-RISE and D-RPLE is stochastic gradient descent (SGD) method, although SGD requires tuning hyper-parameters such as learning rate and batch size of samples.

% The lower computational complexity of CD and SGD methods can make them preferable to convex solvers
% % that may rely on matrix inversion
% which may have a complexity of up to $O(n^3 m_u)$ for a single node reconstruction.
We note that other possible choices of optimization techniques include the entropic gradient descent method \cite{beck2003mirror} (see \cite{vuffray2019efficient} for a description of this method used in estimator RISE for the case of i.i.d. samples) and mirror gradient descent method of \cite{ben2001ordered}.

\subsection{Computational resources}
The numerical experiments were run on a cluster. Each node of the cluster has a Intel Xeon Gold 6248 processor with $2 \times 20$ cores with $32$GB RAM. We were able to take advantage of multiple cores on each node since learning local neighborhoods of each node in our algorithm is done independently and hence can be carried out in parallel. Jobs for different graphical model instances were distributed on the cluster. Parallelization is implemented in our package and example scripts for distributing jobs are included for completeness.

\section{Empirical selection of the regularization parameter $c_\lambda$} \label{sec:emp_selection_lambda}
In this section, we describe the procedure that we used for selecting the values of the coefficient of the regularization parameter $c_\lambda$ for the estimators of D-RPLE and D-RISE that we used in  different regimes of Glauber dynamics and on different Ising model topologies.

The regularization parameter $\lambda$ has the following functional form 
\begin{equation}
    \lambda = c_\lambda \sqrt{\frac{\log(n^2/\delta')}{m_u}}
\end{equation}
where $1-\delta'$ is the probability with which we wish to successfully reconstruct the local neighborhood of $u$. Note that $\delta'$ should be related to $\delta$ used in the quantifying the success of the whole graph recovery as $\delta'=\delta/n$. We follow an approach similar to that followed in Supplementary material of \cite{lokhov2018optimal} to determine the optimal values of $c_\lambda$. Our experimental protocol for selecting an optimal value of $c_\lambda$ on a given Ising model topology is as follows. For a fixed typical values of $\alpha$ and $\beta$, we determine $m^\star$ for different values of $c_\lambda$. The optimal value of $c_\lambda$ is then defined as the one for which the lowest $m^\star$ was obtained. To determine consensus values across topologies, this procedure is repeated on different types of lattices and random regular graphs. Further, as we expect different optimal values of $c_\lambda$, the studies are repeated for both the \treg{} and \mreg{}.

The results for the \treg{} and \mreg{} are shown in Figure~\ref{fig:emp_selection_clambda_DRISE_DRPLE_Tregime} and Figure~\ref{fig:emp_selection_clambda_DRISE_DRPLE_Mregime}. We observe that a consensus optimal value of $c_\lambda$ can be selected across lattices or random regular graphs but not across both topologies. In the \treg{}, optimal choice of $c_\lambda$ for D-RPLE is $\approx 0.05$ on lattices and $\approx 0.1$ on random regular graphs. In the \treg{}, optimal choice of $c_\lambda$ for D-RISE is $\approx 0.1$ on lattices and $\approx 0.45$ on random regular graphs. In the \mreg{}, optimal choice of $c_\lambda$ for D-RPLE is $\approx 0.05$ on lattices and $\approx 0.3$ on random regular graphs. In the \mreg{}, optimal choice of $c_\lambda$ for D-RISE is $\approx 0.1$ on lattices and $\approx 0.7$ on random regular graphs. We use these values for producing the scaling results in the Main Text.

\begin{figure*}[!h]
\centering
\subfloat[D-RISE: Value of $\alpha=0.4$ for all the graphs. Value of $\beta=0.8$ on Lattice-F and $\beta=1.5$ on Lattice-SG. Value of $\beta=1.2$ on RR-F and $\beta=1.8$ on RR-SG.]{
\includegraphics[scale=0.3]{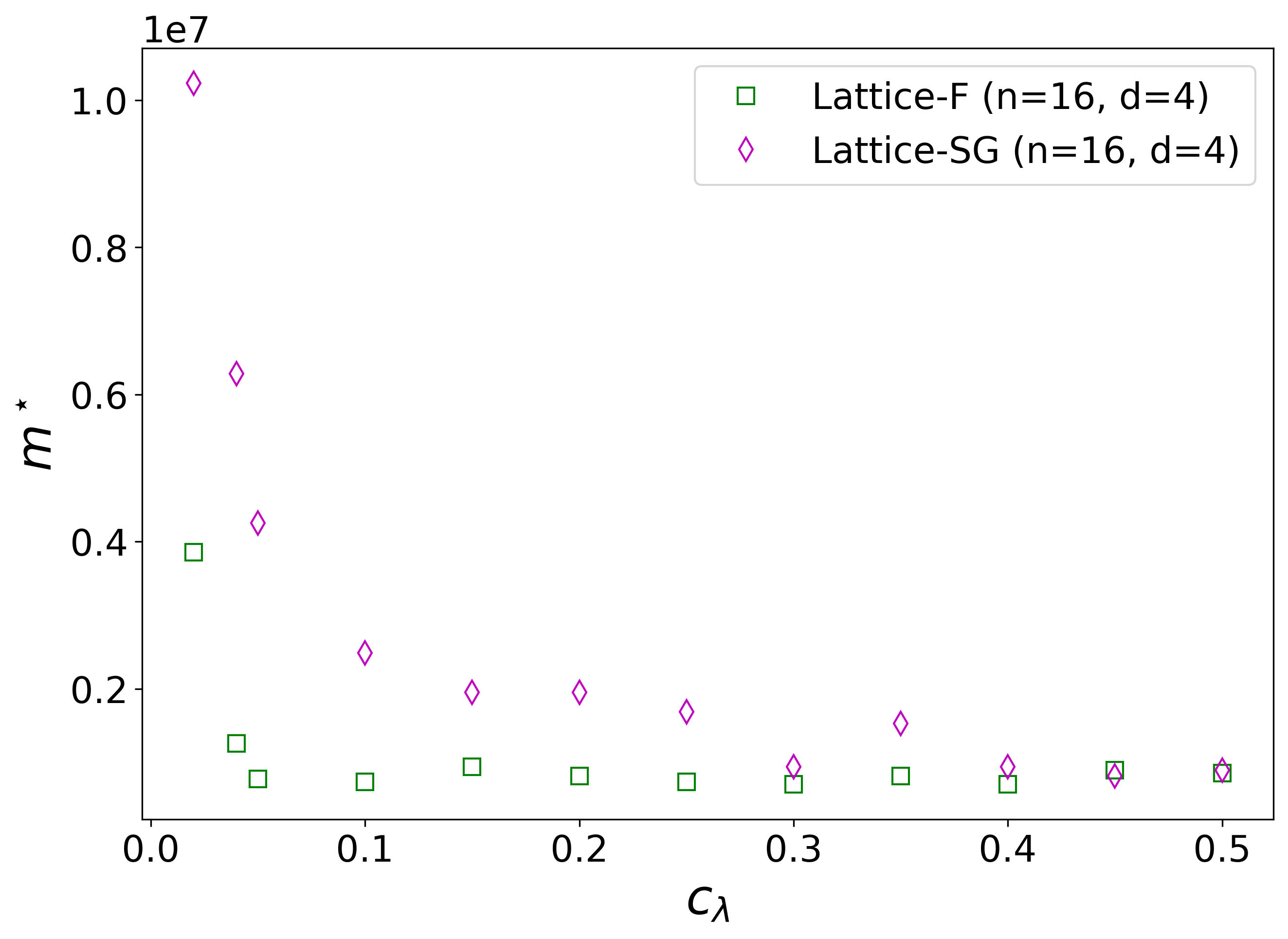} \qquad
\includegraphics[scale=0.3]{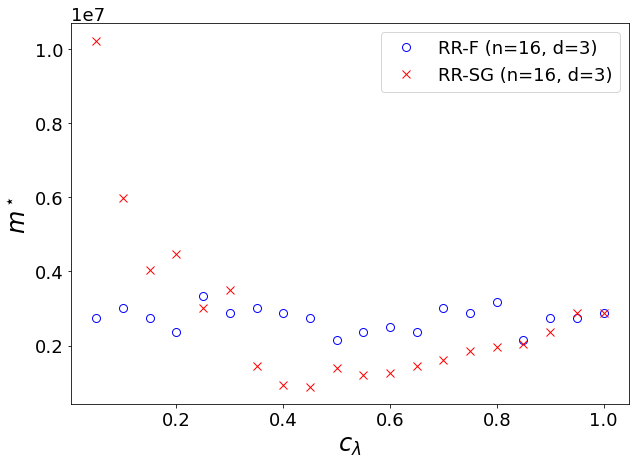}
\label{fig:emp_selection_clambda_DRISE_Tregime}\vspace{-0.03in}
}
\\
\subfloat[D-RPLE: Value of $\alpha=0.4$ for all the graphs. Value of $\beta=0.8$ on Lattice-F and $\beta=1.5$ on Lattice-SG. Value of $\beta=1.2$ on RR-F and $\beta=1.8$ on RR-SG.]{\includegraphics[scale=0.3]{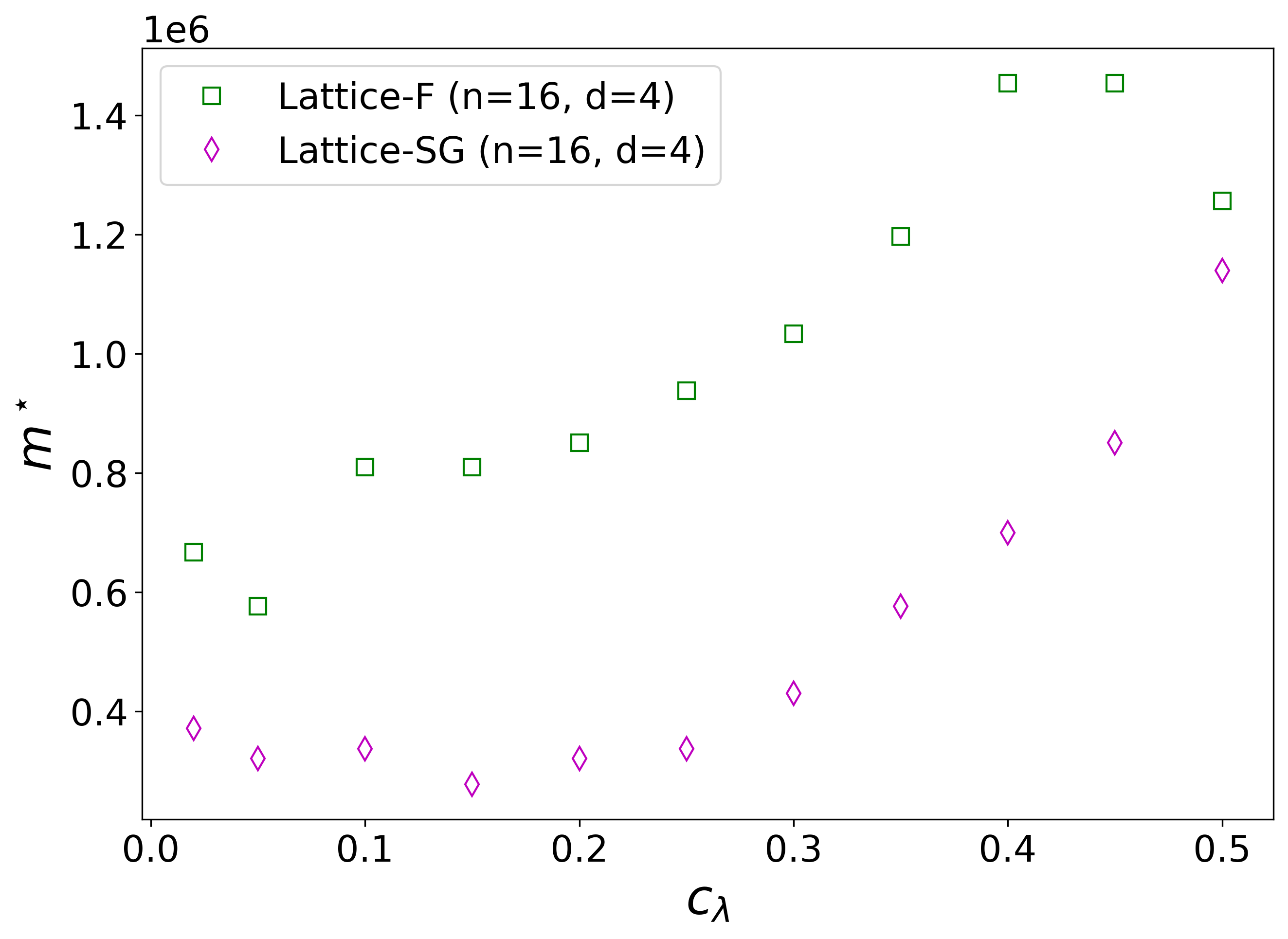} \qquad
\includegraphics[scale=0.3]{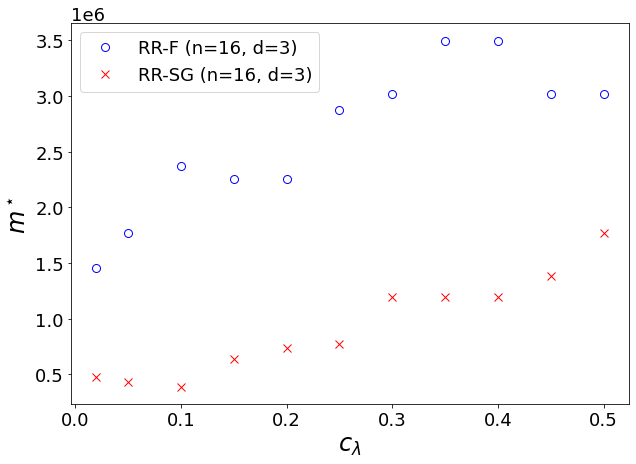}
\label{fig:emp_selection_clambda_DRPLE_Tregime}\vspace{-0.03in}
}
\caption{\textbf{Empirical selection of $c_\lambda$ in \treg{}} We assess the dependence of the number of samples $m^\star$ on the regularization coefficient $c_\lambda$ for the estimators of D-RISE and D-RPLE for successful structure reconstruction of Ising models of size $n=16$. The different Ising model topologies considered are: (Lattice-F) ferromagnetic model on a periodic lattice as in Figure~\ref{fig:sample_complexity_beta_scaling_T_regime}A, (Lattice-SG) spin glass model on a periodic lattice as in  as in Figure~\ref{fig:sample_complexity_beta_scaling_T_regime}C, (RR-F) ferromagnetic model on a random regular graph as in Figure~\ref{fig:sample_complexity_beta_scaling_T_regime}, and (RR-SG) spin glass model on a random regular graph as in Figure~\ref{fig:sample_complexity_beta_scaling_T_regime}D.}
\label{fig:emp_selection_clambda_DRISE_DRPLE_Tregime}
\end{figure*}

\begin{figure*}[!h]
\centering
\subfloat[D-RISE: Value of $\alpha=0.4$ for all the graphs. Value of $\beta=1.5$ on lattices and random regular graphs.]{
\includegraphics[scale=0.3]{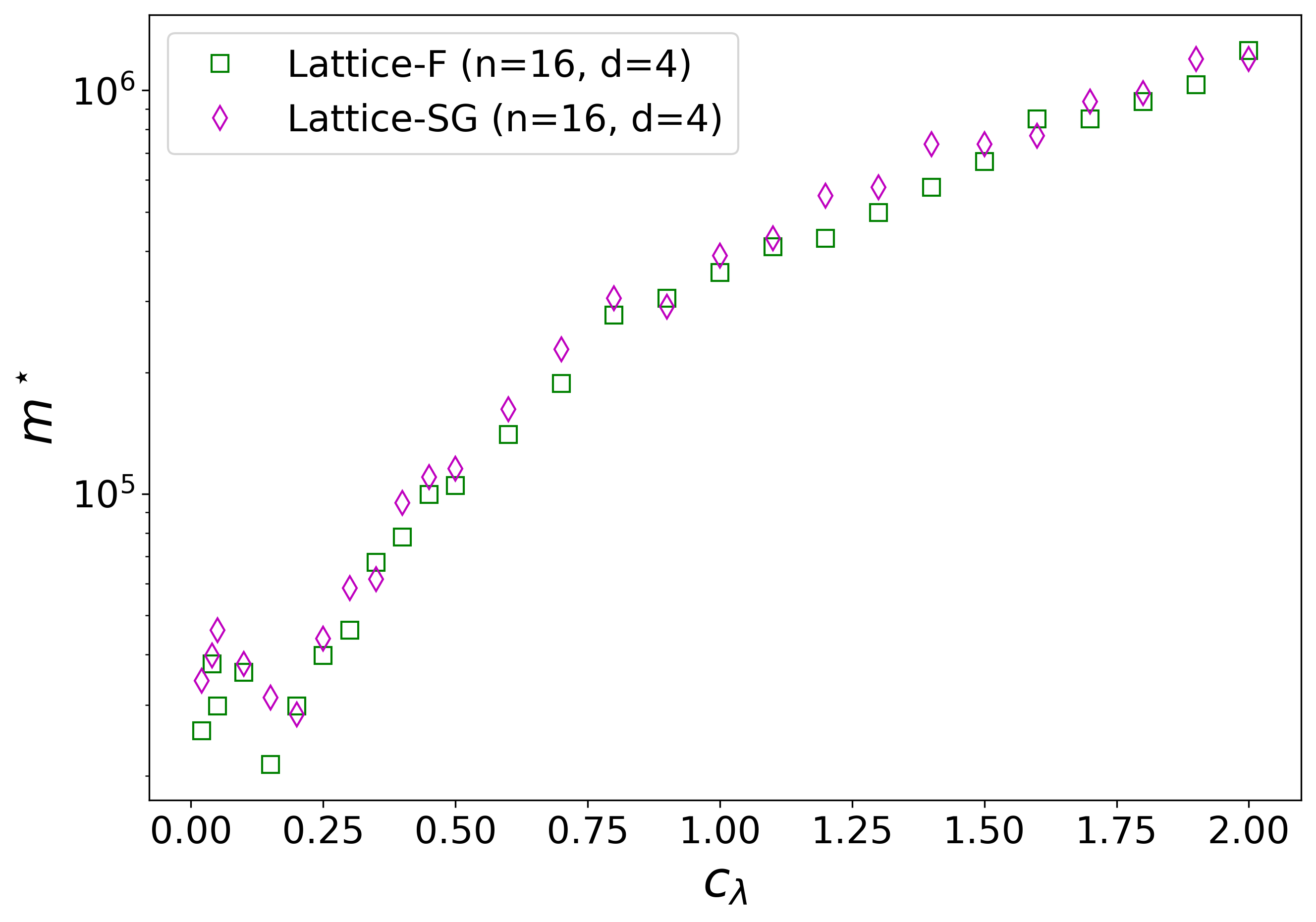} \qquad
\includegraphics[scale=0.3]{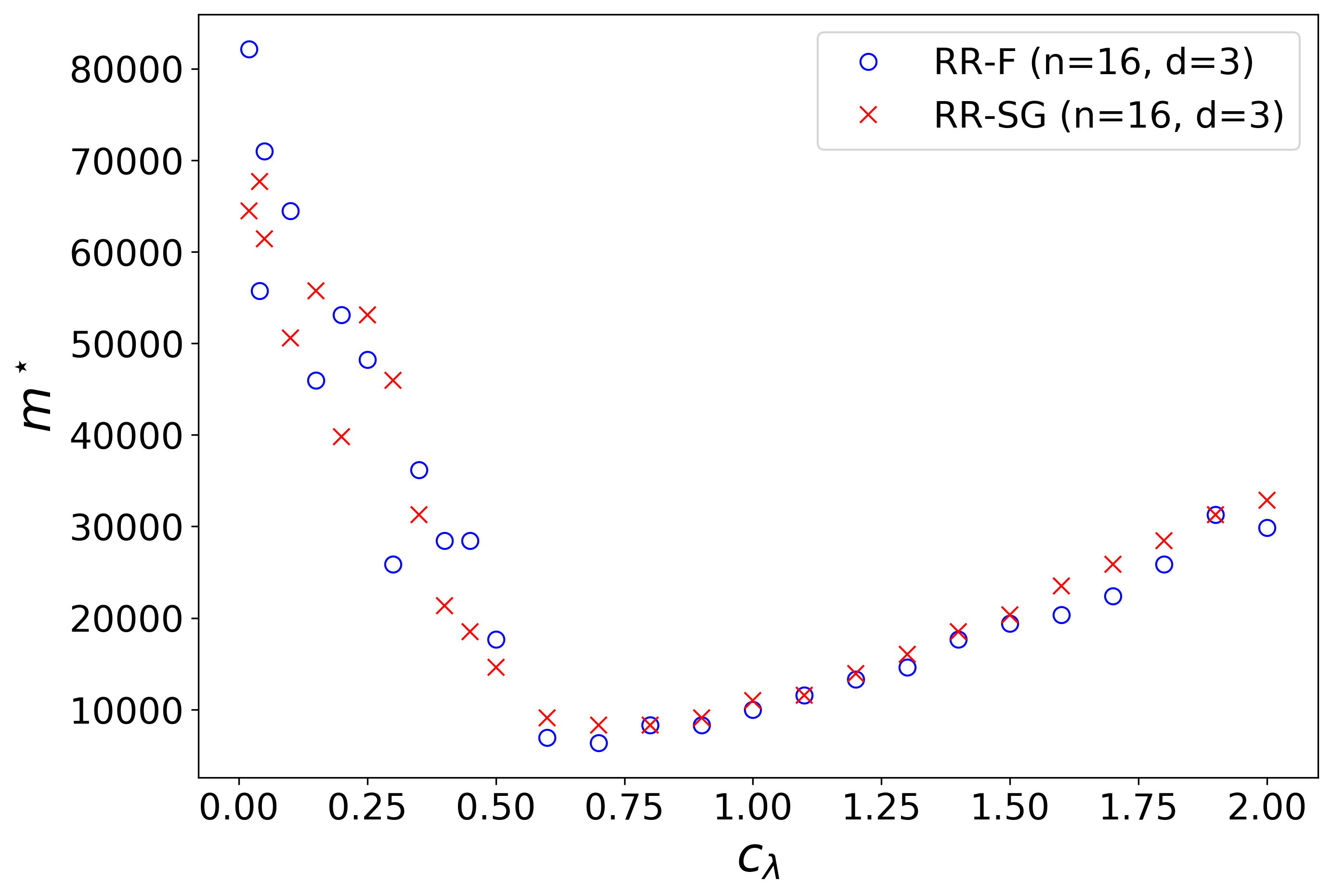}
\label{fig:emp_selection_clambda_DRISE_Mregime}\vspace{-0.03in}
}
\\
\subfloat[D-RPLE: Value of $\alpha=0.4$ for all the graphs. Value of $\beta=1.5$ on lattices and $\beta=2.6$ on random regular graphs.]{\includegraphics[scale=0.3]{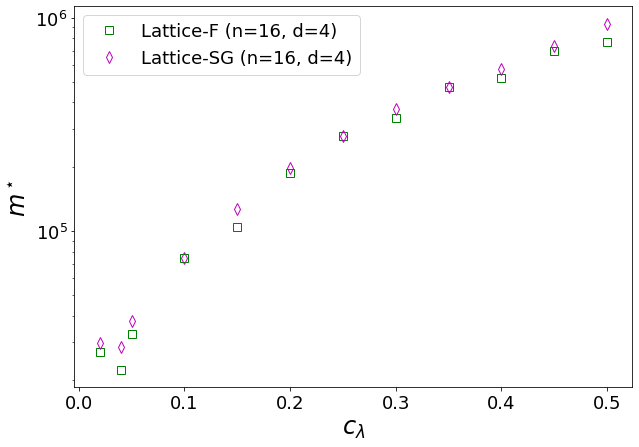} \qquad
\includegraphics[scale=0.3]{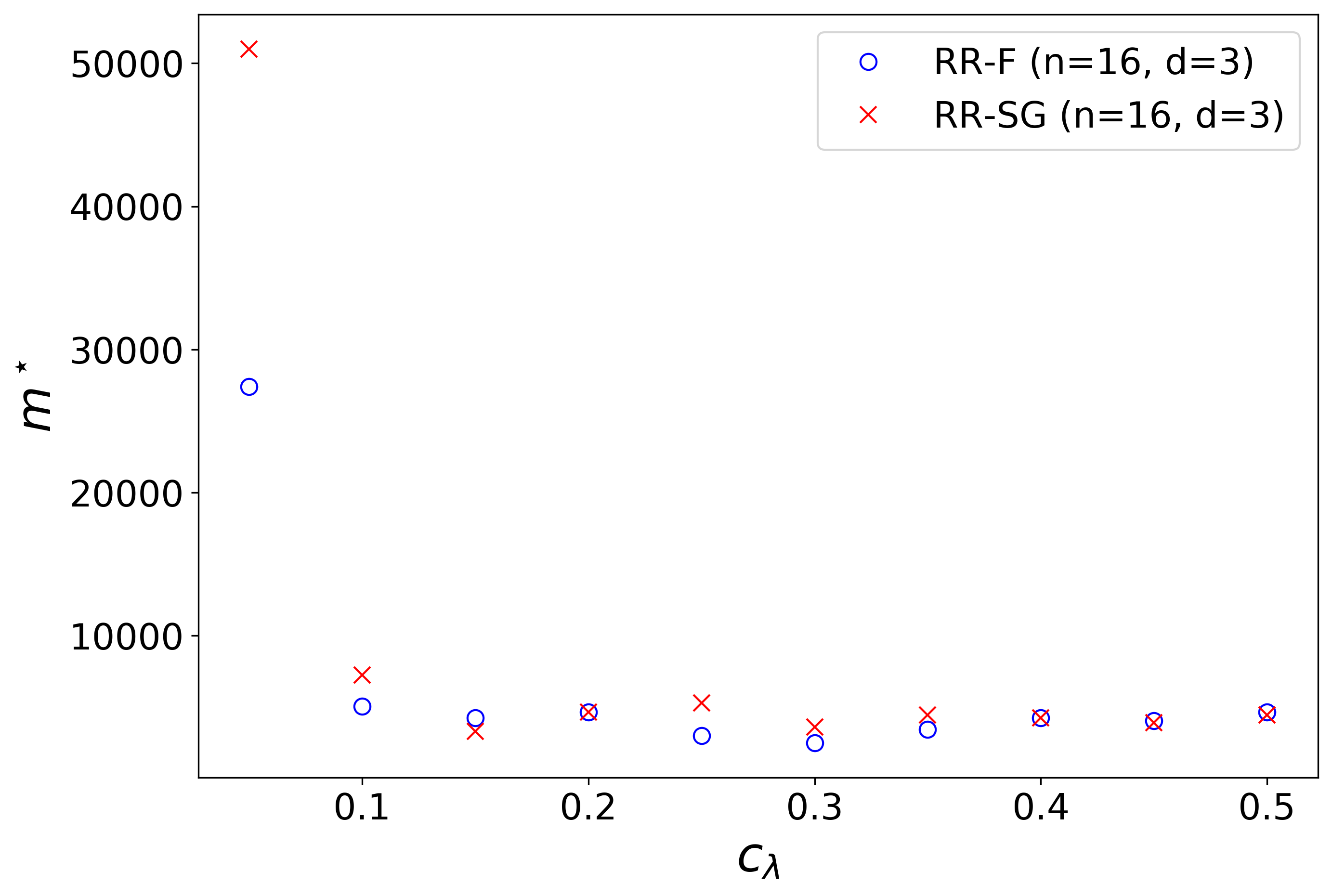}
\label{fig:emp_selection_clambda_DRPLE_Mregime}\vspace{-0.03in}
}
\caption{\textbf{Empirical selection of $c_\lambda$ in \mreg{}} We assess the dependence of the number of samples $m^\star$ on the regularization coefficient $c_\lambda$ for the estimators of D-RISE and D-RPLE for successful structure reconstruction of Ising models of size $n=16$. The different Ising model topologies considered are: (Lattice-F) ferromagnetic model on a periodic lattice as in Figure~\ref{fig:sample_complexity_beta_scaling_M_regime}A, (Lattice-SG) spin glass model on a periodic lattice as in  as in Figure~\ref{fig:sample_complexity_beta_scaling_M_regime}C, (RR-F) ferromagnetic model on a random regular graph as in Figure~\ref{fig:sample_complexity_beta_scaling_M_regime}, and (RR-SG) spin glass model on a random regular graph as in Figure~\ref{fig:sample_complexity_beta_scaling_M_regime}D.}
\label{fig:emp_selection_clambda_DRISE_DRPLE_Mregime}
\end{figure*}

\section{Learning Random Regular Graphs in the \mreg{}} \label{sec:mreg_learning_rr_graphs}
Here we discuss how structure learning in the \mreg{} can result in a sample complexity independent of $\beta = \max_{(i,j) \in E} |J_{ij}|$. The central object of our study is the conditional probability distribution in Eq.~\eqref{eq:glauber_dynamics_conditional_prob_update}. For simplicity we consider the situation where the magnetic field is zero and we can rewrite this conditional distribution as
\begin{equation}
    p(\sigma_i^{t+1}|\underline{\sigma}^t) = \frac{1 + \sigma_i^{t+1} \tanh \beta \left(\sum_{j \neq i} x_{j} \sigma_j^t \right)}{2},
    \label{eq:gd_conditional_prob}
\end{equation}
where $x_{j} = J_{ij} / \beta$ if $j \in \partial i$ and $x_{j}=0$ otherwise. In order to analyse the learning problem when $\beta$ is large, we look at the distribution in Eq.~\eqref{eq:gd_conditional_prob} in the limit where $\beta$ goes to infinity and we obtain the following expression,
\begin{equation}
    \lim_{\beta \rightarrow \infty} p(\sigma_i^{t+1}|\underline{\sigma}^t) = \frac{1 + \sigma_i^{t+1} \sign \left(\sum_{j \neq i} x_{j} \sigma_j^t\right)}{2}.
    \label{eq:gd_conditional_prob_limit}
\end{equation}
The form of the conditional distribution in Eq.~\eqref{eq:gd_conditional_prob_limit} implies that the update of $\sigma_i^{t+1}$ is with probability one equal to $\sign \left(\sum_{j \neq i} x_{j} \sigma_j^t\right)$ whenever $\sum_{j \neq i} x_{j} \sigma_j^t \neq 0$, otherwise $\sigma_i^{t+1}$ is updated to $-1$ or $1$ with equal probabilities. We see in the limit of large $\beta$ that the structure learning problem transforms essentially into the so-called noiseless one-bit compressive sensing problem \cite{Boufounos2008onebit}. In noiseless one-bit compressive sensing, we receive $t\in [1,m]$ observations of signs $y\in \mathbb{R}^m$ of the components of an unknown $d$-sparse vector $\underline{x} \in \mathbb{R}^{n}$ transformed by a known sensing matrix $A \in \mathbb{R}^{m\times n}$ i.e. $\underline{y}= \sign(A \underline{x})$. The objective in one-bit compressive sensing is to recover the support of $\underline{x}$ just like in our structure learning problem. In order to recover the support of $\underline{x}$, the number of observations (and the rank of $A$) has to be at least $m=\Omega(d^2\ln n / \ln d)$, see \cite{Acharya2017improvedbounds}. 
The difference between compressive sensing and our structure learning problem lies in the choice of the sensing matrix. While in compressive sensing, the design of the sensing matrix is left to operators, in our case it is imposed by the distribution $p(\underline{\sigma}^t)$ as the rows of our sensing matrix correspond to i.i.d. samples of spin configurations $A_{tj} = \sigma_j^t$.

In the \treg{} we expect the Glauber dynamics to mix rapidly and generate samples from the equilibrium distribution of the graphical model. In the limit $\beta \rightarrow \infty$, the equilibrium distribution is supported only by spin configurations whose energies are minimal. The number of such configurations is typically constant with respect to the system size. For example, ferromagnetic models have no more than two states of minimal energy regardless of the number of spins. Therefore, our sensing matrix only contains a fixed number of independent rows and the one-bit compressive sensing problem cannot be solved perfectly as $\text{rank}(A)=O(1)$. It implies that for large $\beta$, structure learning with a fixed number of i.i.d. samples from the equilibrium distributions cannot be accomplished.

In the \mreg{}, however, we carry out one step of Glauber dynamics and the distribution $p(\underline{\sigma}^0)$ is uniform. Our sensing matrix turns out to be generated from a Bernoulli ensemble with entries equal to $-1$ and $+1$ at random. Such matrices are known to have a rank $m$ with high probability \cite{kahn1995probability}. This renders possible the inversion of the one-bit compressed sensing problem and consequently the possibility to solve the structure learning problem with a number of observations independent of $\beta$ for $\beta$ large. However, matrices with random signs do not necessarily lead to an invertible one-bit compressive sensing problem and the invertibility of the problem depends on the hidden vector $\underline{x}$. For instance, consider the simple four dimensional vectors with different support $u =(1,1,1,\epsilon)$ where $|\epsilon|\in (0,1)$ and $v =(1,1,1,0)$. It is easy to see that for any configurations of $\sigma_i \in \{-1,1\}$ we have that $\sign\left(\sum_i \sigma_i u_i\right) = \sign\left(\sum_i \sigma_i v_i\right)$ for the quantity $\sigma_4 u_4=\epsilon \sigma_4$ has always a smaller magnitude than $|\sigma_1 + \sigma_2 + \sigma_3|\geq 1$. Therefore, structure learning with a fixed number of samples for large $\beta$ cannot be done for neighborhood of size four and couplings equal to $(\beta, \beta, \beta, \alpha)$, which explains the exponential scaling seen for the lattice instance in Fig.\ref{fig:sample_complexity_beta_scaling_M_regime}.
We have a different story when we consider the three dimensional vector $w= (1,1,\epsilon)$ with $|\epsilon|\in (0,1)$. If we take a sensing matrix equal to  $A=\begin{pmatrix}
1 & -1 & 1\\
-1 & 1 & 1\\
1 & 1 & -1
\end{pmatrix}$ when $\epsilon>0$ or $A=\begin{pmatrix}
1 & -1 & -1\\
-1 & 1 & -1\\
1 & 1 & 1
\end{pmatrix}$ when $\epsilon<0$, we see that the only three dimensional vectors $x$ that satisfy the equation $\sign(A x) = \sign(A w)$ are those for which $x_1>0$, $x_2>0$ and $x_3 \cdot \sign(\epsilon)>0$. This means that the (signed) support of $w$ is recoverable with a sensing matrix from the Bernoulli ensemble. It implies that structure learning with a fixed number of samples for large and even infinite $\beta$ is possible for neighborhood of size three and couplings equal to $(\beta, \beta, \alpha)$ which explains the flat curves seen for the three-regular graphs in Fig.\ref{fig:sample_complexity_beta_scaling_M_regime}.
Based on these considerations, we can extrapolate this behavior in a straightforward manner to graphs with odd and even degree $d$ having $d-1$ identical couplings.

\section{Learning dynamics from neural spike trains} \label{sec:neural_dataset_prep}

In this section, we describe how the relevant dataset from \cite{prentice2016error} is processed to be used for learning Ising dynamics with D-RISE/D-RPLE in this study. Then, we discuss the statistics of learned model parameters, and compare results obtained with D-RISE and D-RPLE. Finally, we provide details on computation of data moments (such as correlations) from the learned Ising model to assess performance against the dataset.

\subsection{Preparation of neural dataset}
The dataset contains spike trains from $152$ salamander retinal ganglion cells in response to a non-repeated natural movie stimulus, of which we select spike trains for $n=42$ neurons over $24$s for our application. To obtain a time series of spin configurations over the neurons, we bin the spike trains into 20 ms time bins. The spin $\sigma_i^{(t)}$ of a neuron $i$ in time bin $t$ is set to $1$ if it fires at least once in this time bin and $-1$ otherwise. We thus produce a sequence of $1.2 \times 10^5$ spin configurations (also called binary spike words). A segment of the sequence is shown as a spike raster in Figure~\ref{fig:spike_train_neural_data}. However, this can't be used directly for learning an Ising model using D-RISE or D-RPLE.

% How did we process this dataset given the assumptions for our learning algorithm?
Our learning algorithms require information about the identity of nodes being updated which isn't directly available from the data recorded. The identity of the node being updated at time $t$ is however known when there is some $l \in [n]$ for which the spin of node/neuron $l$ flips in sign i.e., $\sigma_l^{(t+1)} = - \sigma_l^{(t)}$. There maybe more than one such node at time $t$. However, in Glauber dynamics, only one node is selected for update at time $t$. We thus only consider samples of spin configurations $\{(\underline{\sigma}^{t},\underline{\sigma}^{t+1},I^{t+1})\}$ for time bins $t$ where $I^{(t+1)}$ can be directly inferred by searching for nodes which flipped its spin and there is only one such node. The resulting set of samples are time ordered but samples from consecutive time bins may not be chosen. Thus, it is convenient to represent the samples as $\{{(\underline{\sigma}^{0}}^{(k)},{\underline{\sigma}^{1}}^{(k)},{I^{1}}^{(k)})\}_{k \in [m]}$ where $k$ is now used to index the samples and $t_k$ corresponds to the time $t$ of the $k$-th time bin chosen. After this processing, we end up with a set of $3.2 \times 10^4$ samples corresponding to the \mreg{} with an unknown distribution over the initial spin configurations.

Let us briefly comment on some challenges associated with the dataset preparation procedure that we used, and point out to directions for overcoming these limitations. As we require the identity of the updated nodes for D-RISE/D-RPLE, this requires us to only select samples where a node is observed to be flipped. Otherwise, we wouldn't know a node has been updated or not. The model fit would improve if the node identities weren't required and thus samples were no flips are observed could also be used for estimation. This however requires deriving estimators beyond D-RISE/D-RPLE that is outside the scope of this work, but would be an interesting direction for future studies. The bin size was chosen to be $20$ms as this is the expected time scale of persistence of modes in the neural spike trains \cite{prentice2016error} and was reused. Reducing the bin size decreases the probability of observing a spike or a node activation in that time bin leading to more samples with no updates. Increasing the bin size increases the probability of observing a spike or node activation in the time bin but there may also be more than one spike in the time interval for a given node which is counted only once. Choosing the bin size appropriately ensures that there is enough samples and that the Poisson rate of observation matches closely with the Poisson rate of node updates. We would expect the model fit to overfit if the bin size is too high and be more computationally expensive in the case where bin size is small. 

Finally, it would be interesting to go beyond the assumption of Glauber dynamics for constructing an effective model of the data, making use of all available samples that can be costly to get. Possible extensions of our framework include: (i) considering samples with only spin history and thus not requiring identities of updated nodes, (ii) accounting for multiple nodes being updated at the same time e.g., akin to block Gibbs sampling, and (iii) considering more general Markov chain dynamics. We leave these extensions to future work.

\subsection{Statistics of learned model parameters}

Ising model parameters learned using D-RISE on the set of samples prepared as discussed in the previous section are shown in Figure~\ref{fig:learned_model_neural_data}. If the task is to learn an effective Ising model for explaining the dynamics, these parameter estimates can be used directly. However, if the task is that of a model selection, i.e. learning the network structure of the model, then the following procedure can be used. In the histogram over $\hat{J}_{ij}$, we observe gaps separating a group of estimated couplings in the vicinity of zero from those with higher intensities in absolute value. We choose the threshold $\delta$ to correspond to the first symmetrical gap around zero ($\delta_{-}=-\delta$ and $\delta_{+}=\delta$) that separates these groups. All the coupling parameters $|\hat{J}_{ij}| < \delta$ obtained after structure learning and shown in red in the histogram of Figure~\ref{fig:learned_model_neural_data} would be set to zero. Observing these gaps and ability to choose a clear threshold indicates that the number of samples is sufficient for structure learning.

The resulting $\hat{J}_{ij}$ estimates indicate a low value of $\beta$ and thus a high effective temperature of the model. Most of the couplings are weak with few strong couplings. This is in agreement with other such studies. The effective Ising model that we learn here from Glauber dynamics can be used for predicting higher order moments of the data and understanding the behavior of this population of neurons. The difference in correlation matrices computed from data assuming the samples are i.i.d. and that respecting time (presented in the Main Text) explain the difference in parameter estimates $\hat{J}_{ij}$ obtained through RISE and D-RISE as visualized in Figure~\ref{fig:neural_comparison_drise_rise_estimate}. This once again highlights the importance of respecting dynamics and hence time correlations in the data when learning an effective Ising model.

\begin{figure}[h!]
    \centering
    \includegraphics[scale=0.45]{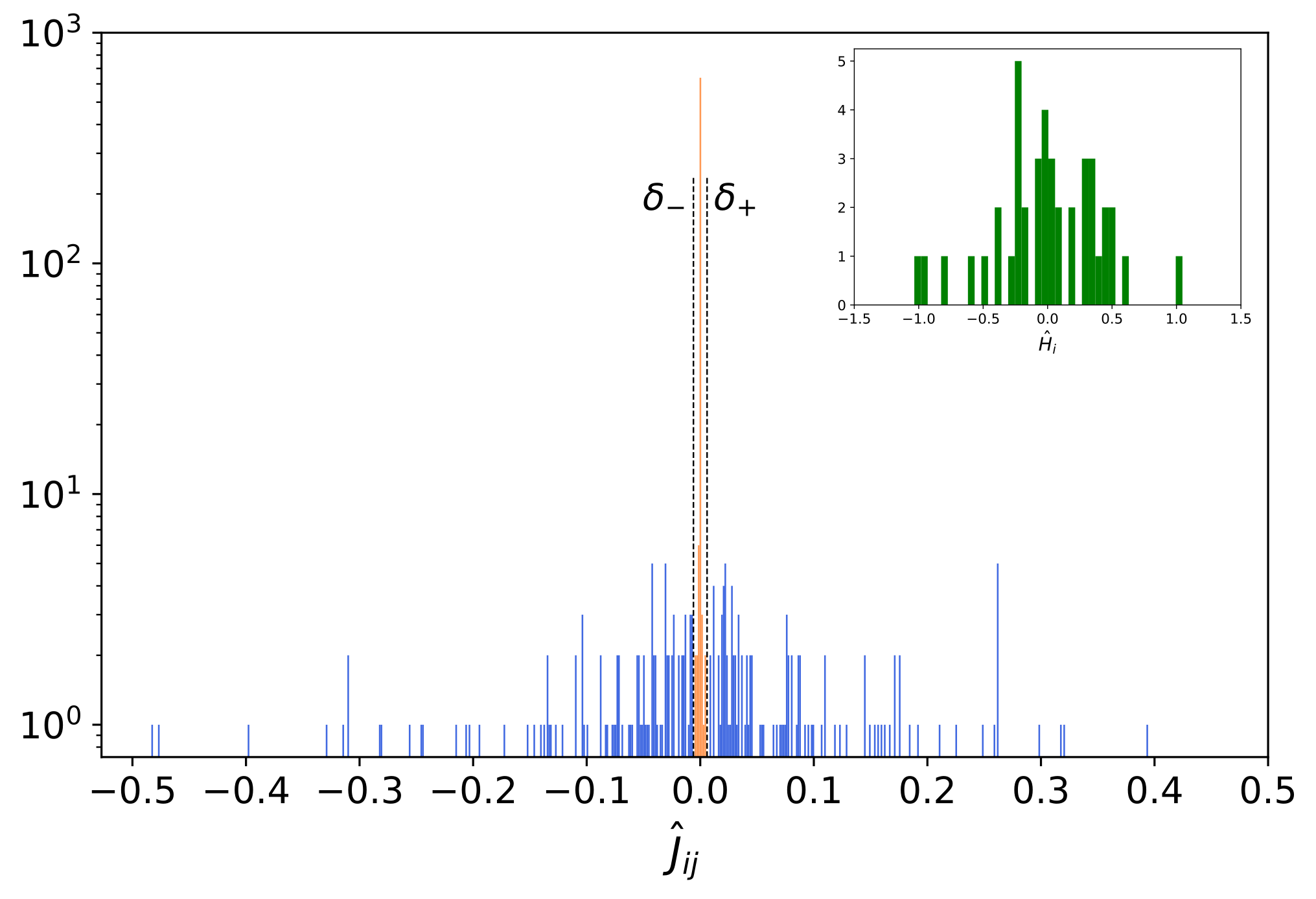}
	\caption{Ising model parameters learned from spike trains over $42$ neurons using $3.2 \times 10^4$ samples. Significant couplings are in blue and thresholded couplings are in red. Reconstructed fields are in green in a separated histogram.}
	\label{fig:learned_model_neural_data}
	\vspace{-0.1in}
\end{figure}
\begin{figure}[h!]
    \centering
    \includegraphics[scale=0.35]{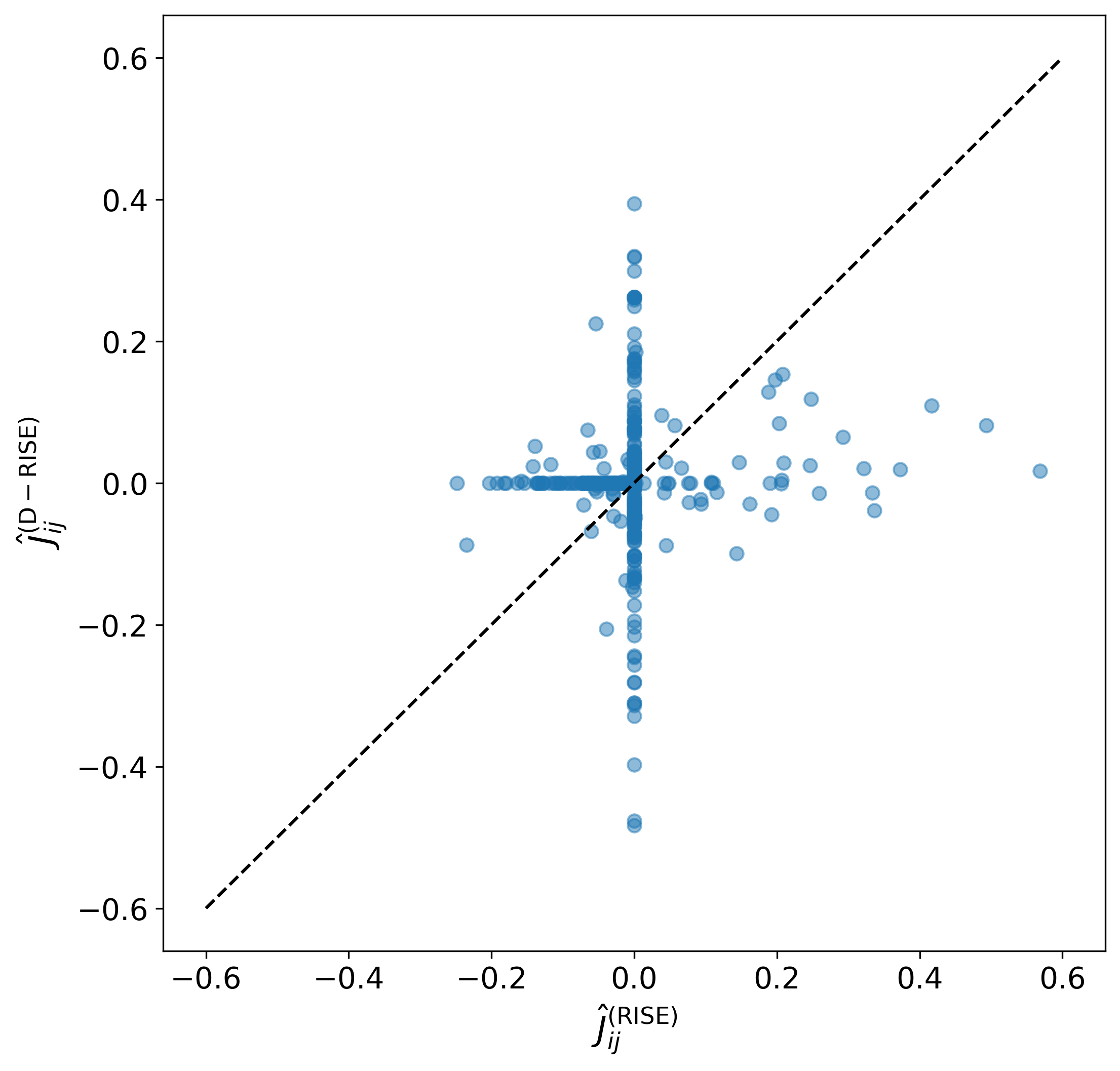}
	\caption{Comparison of Ising model parameter estimates $\hat{J}_{ij}$ obtained through RISE and D-RISE}
	\label{fig:neural_comparison_drise_rise_estimate}
	\vspace{-0.2in}
\end{figure}

\subsection{Computation of correlations}
To assess the performance of the learned Ising model, it is common to compare predicted correlations from the model against the data. Here, we obtain predicted correlations by computing correlations on a dataset simulated using the learned Ising model under the M-regime and running Glauber dynamics. We now explain how this simulated dataset is constructed.

We note that the covariance of interest (including the corresponding probabilities) is given by
\begin{equation}
    \mathrm{Cov}(\sigma_i^0, \sigma_j^1) = \expectation_{p(\sigma_j^1|\sigma_i^0,I^1)p(\sigma_i^0)p(I^1)}[\sigma_i^0 \sigma_j^1] - \expectation_{p(\sigma_i^0)}[\sigma_i^0] \expectation_{p(\sigma_j^1|\sigma_i^0,I^1)}[\sigma_j^1]
\end{equation}
We note from the above expression that the resulting covariance not only depends on the Ising model dynamics which explains $p(\sigma_j^1|\sigma_i^0,I^1)$ but also the initial distribution over spin configurations $p(\sigma_i^0)$ and the probability of a node being chosen for update $p(I^1)$. In order to ensure we can compare the predicted covariance from the Ising model against that from data, we use the same $p(\sigma_i^0)$ and $p(I^1)$ as in the experimental dataset. Additionally, we note that the spin configurations in the experimental dataset only contains flipped spin configurations which also needs to be respsected.

Thus, we construct the simulated dataset by running Glauber dynamics on the M-regime using the same $p(\sigma^0)$ and $p(I^1)$ as from the experimental dataset, and only including those $\sigma^1$ where there is a flip in the sign of node given by $I^1$. Moments are then computed on this simulated dataset using the usual estimators for population means, covariances, correlations, etc., from the samples.

%% file: main.bbl
\begin{thebibliography}{61}
\providecommand{\natexlab}[1]{#1}
\providecommand{\url}[1]{\texttt{#1}}
\expandafter\ifx\csname urlstyle\endcsname\relax
  \providecommand{\doi}[1]{doi: #1}\else
  \providecommand{\doi}{doi: \begingroup \urlstyle{rm}\Url}\fi

\bibitem[{Acharya} et~al.(2017){Acharya}, {Bhattacharyya}, and
  {Kamath}]{Acharya2017improvedbounds}
{Acharya}, J., {Bhattacharyya}, A., and {Kamath}, P.
\newblock Improved bounds for universal one-bit compressive sensing.
\newblock In \emph{2017 IEEE International Symposium on Information Theory
  (ISIT)}, pp.\  2353--2357, 2017.
\newblock \doi{10.1109/ISIT.2017.8006950}.

\bibitem[Bachschmid-Romano \& Opper(2015)Bachschmid-Romano and
  Opper]{bachschmid2015learning}
Bachschmid-Romano, L. and Opper, M.
\newblock Learning of couplings for random asymmetric kinetic ising models
  revisited: random correlation matrices and learning curves.
\newblock \emph{Journal of Statistical Mechanics: Theory and Experiment},
  2015\penalty0 (9):\penalty0 P09016, 2015.

\bibitem[Beck \& Teboulle(2003)Beck and Teboulle]{beck2003mirror}
Beck, A. and Teboulle, M.
\newblock Mirror descent and nonlinear projected subgradient methods for convex
  optimization.
\newblock \emph{Operations Research Letters}, 31\penalty0 (3):\penalty0
  167--175, 2003.

\bibitem[Ben-Tal et~al.(2001)Ben-Tal, Margalit, and Nemirovski]{ben2001ordered}
Ben-Tal, A., Margalit, T., and Nemirovski, A.
\newblock The ordered subsets mirror descent optimization method with
  applications to tomography.
\newblock \emph{SIAM Journal on Optimization}, 12\penalty0 (1):\penalty0
  79--108, 2001.

\bibitem[Berry et~al.(1997)Berry, Warland, and Meister]{berry1997structure}
Berry, M.~J., Warland, D.~K., and Meister, M.
\newblock The structure and precision of retinal spike trains.
\newblock \emph{Proceedings of the National Academy of Sciences}, 94\penalty0
  (10):\penalty0 5411--5416, 1997.

\bibitem[Besag(1975)]{besag1975statistical}
Besag, J.
\newblock Statistical analysis of non-lattice data.
\newblock \emph{Journal of the Royal Statistical Society: Series D (The
  Statistician)}, 24\penalty0 (3):\penalty0 179--195, 1975.

\bibitem[Biegler \& Zavala(2009)Biegler and Zavala]{biegler2009large}
Biegler, L.~T. and Zavala, V.~M.
\newblock Large-scale nonlinear programming using ipopt: An integrating
  framework for enterprise-wide dynamic optimization.
\newblock \emph{Computers \& Chemical Engineering}, 33\penalty0 (3):\penalty0
  575--582, 2009.

\bibitem[{Boufounos} \& {Baraniuk}(2008){Boufounos} and
  {Baraniuk}]{Boufounos2008onebit}
{Boufounos}, P.~T. and {Baraniuk}, R.~G.
\newblock 1-bit compressive sensing.
\newblock In \emph{2008 42nd Annual Conference on Information Sciences and
  Systems}, pp.\  16--21, 2008.
\newblock \doi{10.1109/CISS.2008.4558487}.

\bibitem[Bresler(2015)]{Bresler2015}
Bresler, G.
\newblock Efficiently learning {I}sing models on arbitrary graphs.
\newblock In \emph{Proceedings of the Forty-Seventh Annual ACM on Symposium on
  Theory of Computing}, pp.\  771--782. ACM, 2015.

\bibitem[Bresler et~al.(2017)Bresler, Gamarnik, and Shah]{bresler2017learning}
Bresler, G., Gamarnik, D., and Shah, D.
\newblock Learning graphical models from the glauber dynamics.
\newblock \emph{IEEE Transactions on Information Theory}, 64\penalty0
  (6):\penalty0 4072--4080, 2017.

\bibitem[Broderick et~al.(2007)Broderick, Dudik, Tkacik, Schapire, and
  Bialek]{broderick2007faster}
Broderick, T., Dudik, M., Tkacik, G., Schapire, R.~E., and Bialek, W.
\newblock Faster solutions of the inverse pairwise ising problem.
\newblock \emph{arXiv preprint arXiv:0712.2437}, 2007.

\bibitem[{Buczak} \& {Guven}(2016){Buczak} and {Guven}]{Buczak2016datamining}
{Buczak}, A.~L. and {Guven}, E.
\newblock A survey of data mining and machine learning methods for cyber
  security intrusion detection.
\newblock \emph{IEEE Communications Surveys Tutorials}, 18\penalty0
  (2):\penalty0 1153--1176, 2016.

\bibitem[Chaudhuri et~al.(2015)Chaudhuri, Kakade, Netrapalli, and
  Sanghavi]{chaudhuri2015convergence}
Chaudhuri, K., Kakade, S.~M., Netrapalli, P., and Sanghavi, S.
\newblock Convergence rates of active learning for maximum likelihood
  estimation.
\newblock In \emph{Advances in Neural Information Processing Systems}, pp.\
  1090--1098, 2015.

\bibitem[Chaves et~al.(2015)Chaves, Majenz, and Gross]{chaves2015information}
Chaves, R., Majenz, C., and Gross, D.
\newblock Information--theoretic implications of quantum causal structures.
\newblock \emph{Nature communications}, 6:\penalty0 5766, 2015.

\bibitem[{Chow} \& {Liu}(1968){Chow} and {Liu}]{chowliu68}
{Chow}, C. and {Liu}, C.
\newblock Approximating discrete probability distributions with dependence
  trees.
\newblock \emph{IEEE Transactions on Information Theory}, 14\penalty0
  (3):\penalty0 462--467, May 1968.
\newblock ISSN 1557-9654.
\newblock \doi{10.1109/TIT.1968.1054142}.

\bibitem[Cocco et~al.(2009)Cocco, Leibler, and Monasson]{cocco2009neuronal}
Cocco, S., Leibler, S., and Monasson, R.
\newblock Neuronal couplings between retinal ganglion cells inferred by
  efficient inverse statistical physics methods.
\newblock \emph{Proceedings of the National Academy of Sciences}, 106\penalty0
  (33):\penalty0 14058--14062, 2009.

\bibitem[Constantinou et~al.(2016)Constantinou, Fenton, Marsh, and
  Radlinski]{constantinou2016complex}
Constantinou, A.~C., Fenton, N., Marsh, W., and Radlinski, L.
\newblock From complex questionnaire and interviewing data to intelligent
  bayesian network models for medical decision support.
\newblock \emph{Artificial intelligence in medicine}, 67:\penalty0 75--93,
  2016.

\bibitem[Decelle \& Zhang(2015)Decelle and Zhang]{decelle2015inference}
Decelle, A. and Zhang, P.
\newblock Inference of the sparse kinetic ising model using the decimation
  method.
\newblock \emph{Physical Review E}, 91\penalty0 (5):\penalty0 052136, 2015.

\bibitem[Decelle et~al.(2016)Decelle, Ricci-Tersenghi, and
  Zhang]{decelle2016data}
Decelle, A., Ricci-Tersenghi, F., and Zhang, P.
\newblock Data quality for the inverse lsing problem.
\newblock \emph{Journal of Physics A: Mathematical and Theoretical},
  49\penalty0 (38):\penalty0 384001, 2016.

\bibitem[Eagle et~al.(2009)Eagle, Pentland, and Lazer]{EaglePentlandLazer2009}
Eagle, N., Pentland, A.~S., and Lazer, D.
\newblock Inferring friendship network structure by using mobile phone data.
\newblock \emph{Proceedings of the National Academy of Sciences}, 106\penalty0
  (36):\penalty0 15274--15278, 2009.
\newblock \doi{10.1073/pnas.0900282106}.

\bibitem[Glauber(1963)]{glauber1963time}
Glauber, R.~J.
\newblock Time-dependent statistics of the ising model.
\newblock \emph{Journal of mathematical physics}, 4\penalty0 (2):\penalty0
  294--307, 1963.

\bibitem[Gotovos et~al.(2015)Gotovos, Hassani, and Krause]{Hassani2015sampling}
Gotovos, A., Hassani, H., and Krause, A.
\newblock Sampling from probabilistic submodular models.
\newblock In Cortes, C., Lawrence, N.~D., Lee, D.~D., Sugiyama, M., and
  Garnett, R. (eds.), \emph{Advances in Neural Information Processing Systems
  28}, pp.\  1945--1953. Curran Associates, Inc., 2015.

\bibitem[Hamilton et~al.(2017)Hamilton, Koehler, and Moitra]{Ankur2017nips}
Hamilton, L., Koehler, F., and Moitra, A.
\newblock Information theoretic properties of markov random fields, and their
  algorithmic applications.
\newblock In Guyon, I., Luxburg, U.~V., Bengio, S., Wallach, H., Fergus, R.,
  Vishwanathan, S., and Garnett, R. (eds.), \emph{Advances in Neural
  Information Processing Systems 30}, pp.\  2463--2472. Curran Associates,
  Inc., 2017.

\bibitem[He \& Zhang(2011)He and Zhang]{HeZhang2011}
He, M. and Zhang, J.
\newblock A dependency graph approach for fault detection and localization
  towards secure smart grid.
\newblock \emph{IEEE Transactions on Smart Grid}, 2\penalty0 (2):\penalty0
  342--351, June 2011.
\newblock ISSN 1949-3053.
\newblock \doi{10.1109/TSG.2011.2129544}.

\bibitem[Hertz et~al.(2011)Hertz, Roudi, and Tyrcha]{hertz2011ising}
Hertz, J., Roudi, Y., and Tyrcha, J.
\newblock Ising models for inferring network structure from spike data.
\newblock \emph{arXiv preprint arXiv:1106.1752}, 2011.

\bibitem[Jansen et~al.()Jansen, Yu, Greenbaum, Kluger, Krogan, Chung, Emili,
  Snyder, Greenblatt, and Gerstein]{Jansen2003}
Jansen, R., Yu, H., Greenbaum, D., Kluger, Y., Krogan, N.~J., Chung, S., Emili,
  A., Snyder, M., Greenblatt, J.~F., and Gerstein, M.
\newblock A bayesian networks approach for predicting protein-protein
  interactions from genomic data.
\newblock 302\penalty0 (5644):\penalty0 449--453.

\bibitem[Kahn et~al.(1995)Kahn, Koml{\'o}s, and
  Szemer{\'e}di]{kahn1995probability}
Kahn, J., Koml{\'o}s, J., and Szemer{\'e}di, E.
\newblock On the probability that a random$\pm$1-matrix is singular.
\newblock \emph{Journal of the American Mathematical Society}, 8\penalty0
  (1):\penalty0 223--240, 1995.

\bibitem[Klivans \& Meka(2017)Klivans and Meka]{Klivans2017}
Klivans, A. and Meka, R.
\newblock Learning graphical models using multiplicative weights.
\newblock In \emph{2017 IEEE 58th Annual Symposium on Foundations of Computer
  Science (FOCS)}, pp.\  343--354, Oct 2017.

\bibitem[Levin \& Peres(2017)Levin and Peres]{levin2017markov}
Levin, D.~A. and Peres, Y.
\newblock \emph{Markov chains and mixing times}, volume 107.
\newblock American Mathematical Soc., 2017.

\bibitem[Lokhov et~al.(2018)Lokhov, Vuffray, Misra, and
  Chertkov]{lokhov2018optimal}
Lokhov, A.~Y., Vuffray, M., Misra, S., and Chertkov, M.
\newblock Optimal structure and parameter learning of ising models.
\newblock \emph{Science advances}, 4\penalty0 (3):\penalty0 e1700791, 2018.

\bibitem[Marbach et~al.(2012)Marbach, Costello, Kuffner, Vega, Prill, Camacho,
  Allison, Kellis, Collins, and Stolovitzky]{MarbachCostelloKuffnerEtAl2012}
Marbach, D., Costello, J.~C., Kuffner, R., Vega, N.~M., Prill, R.~J., Camacho,
  D.~M., Allison, K.~R., Kellis, M., Collins, J.~J., and Stolovitzky, G.
\newblock Wisdom of crowds for robust gene network inference.
\newblock \emph{Nat Meth}, 9\penalty0 (8):\penalty0 796--804, Aug 2012.
\newblock ISSN 1548-7091.
\newblock \doi{10.1038/nmeth.2016}.

\bibitem[Marre et~al.(2009)Marre, El~Boustani, Fr{\'e}gnac, and
  Destexhe]{marre2009prediction}
Marre, O., El~Boustani, S., Fr{\'e}gnac, Y., and Destexhe, A.
\newblock Prediction of spatiotemporal patterns of neural activity from
  pairwise correlations.
\newblock \emph{Physical review letters}, 102\penalty0 (13):\penalty0 138101,
  2009.

\bibitem[Martinelli \& Olivieri(1994)Martinelli and
  Olivieri]{martinelli1994approach}
Martinelli, F. and Olivieri, E.
\newblock Approach to equilibrium of glauber dynamics in the one phase region.
\newblock \emph{Communications in Mathematical Physics}, 161\penalty0
  (3):\penalty0 447--486, 1994.

\bibitem[M{\'e}zard \& Sakellariou(2011)M{\'e}zard and
  Sakellariou]{mezard2011exact}
M{\'e}zard, M. and Sakellariou, J.
\newblock Exact mean-field inference in asymmetric kinetic ising systems.
\newblock \emph{Journal of Statistical Mechanics: Theory and Experiment},
  2011\penalty0 (07):\penalty0 L07001, 2011.

\bibitem[Morcos et~al.(2011)Morcos, Pagnani, Lunt, Bertolino, Marks, Sander,
  Zecchina, Onuchic, Hwa, and Weigt]{MorcosPagnaniLuntEtAl2011}
Morcos, F., Pagnani, A., Lunt, B., Bertolino, A., Marks, D.~S., Sander, C.,
  Zecchina, R., Onuchic, J.~N., Hwa, T., and Weigt, M.
\newblock Direct-coupling analysis of residue coevolution captures native
  contacts across many protein families.
\newblock \emph{Proceedings of the National Academy of Sciences}, 108\penalty0
  (49):\penalty0 E1293--E1301, 2011.
\newblock \doi{10.1073/pnas.1111471108}.

\bibitem[Negahban et~al.(2009)Negahban, Yu, Wainwright, and
  Ravikumar]{negahban2009unified}
Negahban, S., Yu, B., Wainwright, M.~J., and Ravikumar, P.~K.
\newblock A unified framework for high-dimensional analysis of $ m $-estimators
  with decomposable regularizers.
\newblock In \emph{Advances in neural information processing systems}, pp.\
  1348--1356. Citeseer, 2009.

\bibitem[Nirenberg \& Victor(2007)Nirenberg and Victor]{nirenberg2007analyzing}
Nirenberg, S.~H. and Victor, J.~D.
\newblock Analyzing the activity of large populations of neurons: how tractable
  is the problem?
\newblock \emph{Current opinion in neurobiology}, 17\penalty0 (4):\penalty0
  397--400, 2007.

\bibitem[Pillow et~al.(2008)Pillow, Shlens, Paninski, Sher, Litke,
  Chichilnisky, and Simoncelli]{pillow2008spatio}
Pillow, J.~W., Shlens, J., Paninski, L., Sher, A., Litke, A.~M., Chichilnisky,
  E., and Simoncelli, E.~P.
\newblock Spatio-temporal correlations and visual signalling in a complete
  neuronal population.
\newblock \emph{Nature}, 454\penalty0 (7207):\penalty0 995--999, 2008.

\bibitem[Prentice et~al.(2016)Prentice, Marre, Ioffe, Loback, Tka{\v{c}}ik, and
  Berry]{prentice2016error}
Prentice, J.~S., Marre, O., Ioffe, M.~L., Loback, A.~R., Tka{\v{c}}ik, G., and
  Berry, M.~J.
\newblock Error-robust modes of the retinal population code.
\newblock \emph{PLoS computational biology}, 12\penalty0 (11):\penalty0
  e1005148, 2016.

\bibitem[Ravikumar et~al.(2010)Ravikumar, Wainwright, Lafferty,
  et~al.]{ravikumar2010high}
Ravikumar, P., Wainwright, M.~J., Lafferty, J.~D., et~al.
\newblock High-dimensional ising model selection using $\ell_1$-regularized
  logistic regression.
\newblock \emph{The Annals of Statistics}, 38\penalty0 (3):\penalty0
  1287--1319, 2010.

\bibitem[Rieke et~al.(1999)Rieke, Warland, Van~Steveninck, Bialek,
  et~al.]{rieke1999spikes}
Rieke, F., Warland, D., Van~Steveninck, R. D.~R., Bialek, W.~S., et~al.
\newblock \emph{Spikes: exploring the neural code}, volume~7.
\newblock MIT press Cambridge, 1999.

\bibitem[Roth \& Black(2005)Roth and Black]{RothBlack2005}
Roth, S. and Black, M.~J.
\newblock Fields of experts: a framework for learning image priors.
\newblock In \emph{Computer Vision and Pattern Recognition, 2005. CVPR 2005.
  IEEE Computer Society Conference on}, volume~2, pp.\  860--867 vol. 2, June
  2005.
\newblock \doi{10.1109/CVPR.2005.160}.

\bibitem[Roudi \& Hertz(2011)Roudi and Hertz]{roudi2011mean}
Roudi, Y. and Hertz, J.
\newblock Mean field theory for nonequilibrium network reconstruction.
\newblock \emph{Physical review letters}, 106\penalty0 (4):\penalty0 048702,
  2011.

\bibitem[Roudi et~al.(2009{\natexlab{a}})Roudi, Nirenberg, and
  Latham]{roudi2009pairwise}
Roudi, Y., Nirenberg, S., and Latham, P.~E.
\newblock Pairwise maximum entropy models for studying large biological
  systems: when they can work and when they can't.
\newblock \emph{PLoS Comput Biol}, 5\penalty0 (5):\penalty0 e1000380,
  2009{\natexlab{a}}.

\bibitem[Roudi et~al.(2009{\natexlab{b}})Roudi, Tyrcha, and
  Hertz]{roudi2009ising}
Roudi, Y., Tyrcha, J., and Hertz, J.
\newblock Ising model for neural data: model quality and approximate methods
  for extracting functional connectivity.
\newblock \emph{Physical Review E}, 79\penalty0 (5):\penalty0 051915,
  2009{\natexlab{b}}.

\bibitem[Santhanam \& Wainwright(2012)Santhanam and
  Wainwright]{santhanam2012information}
Santhanam, N.~P. and Wainwright, M.~J.
\newblock Information-theoretic limits of selecting binary graphical models in
  high dimensions.
\newblock \emph{IEEE Transactions on Information Theory}, 58\penalty0
  (7):\penalty0 4117--4134, 2012.

\bibitem[Schneidman et~al.(2006{\natexlab{a}})Schneidman, Berry, Segev, and
  Bialek]{SchneidmanBerrySegevEtAl2006}
Schneidman, E., Berry, M.~J., Segev, R., and Bialek, W.
\newblock Weak pairwise correlations imply strongly correlated network states
  in a neural population.
\newblock \emph{Nature}, 440\penalty0 (7087):\penalty0 1007--1012, Apr
  2006{\natexlab{a}}.
\newblock ISSN 0028-0836.
\newblock \doi{10.1038/nature04701}.

\bibitem[Schneidman et~al.(2006{\natexlab{b}})Schneidman, Berry, Segev, and
  Bialek]{schneidman2006weak}
Schneidman, E., Berry, M.~J., Segev, R., and Bialek, W.
\newblock Weak pairwise correlations imply strongly correlated network states
  in a neural population.
\newblock \emph{Nature}, 440\penalty0 (7087):\penalty0 1007--1012,
  2006{\natexlab{b}}.

\bibitem[Settles(2009)]{settles2009active}
Settles, B.
\newblock Active learning literature survey.
\newblock Technical report, University of Wisconsin-Madison Department of
  Computer Sciences, 2009.

\bibitem[Sly \& Sun(2012)Sly and Sun]{sly2012computational}
Sly, A. and Sun, N.
\newblock The computational hardness of counting in two-spin models on
  d-regular graphs.
\newblock In \emph{2012 IEEE 53rd Annual Symposium on Foundations of Computer
  Science}, pp.\  361--369. IEEE, 2012.

\bibitem[Sourati et~al.(2017{\natexlab{a}})Sourati, Akcakaya, Erdogmus, Leen,
  and Dy]{sourati2017probabilistic}
Sourati, J., Akcakaya, M., Erdogmus, D., Leen, T.~K., and Dy, J.~G.
\newblock A probabilistic active learning algorithm based on fisher information
  ratio.
\newblock \emph{IEEE transactions on pattern analysis and machine
  intelligence}, 40\penalty0 (8):\penalty0 2023--2029, 2017{\natexlab{a}}.

\bibitem[Sourati et~al.(2017{\natexlab{b}})Sourati, Akcakaya, Leen, Erdogmus,
  and Dy]{sourati2017asymptotic}
Sourati, J., Akcakaya, M., Leen, T.~K., Erdogmus, D., and Dy, J.~G.
\newblock Asymptotic analysis of objectives based on fisher information in
  active learning.
\newblock \emph{The Journal of Machine Learning Research}, 18\penalty0
  (1):\penalty0 1123--1163, 2017{\natexlab{b}}.

\bibitem[Tkacik et~al.(2009)Tkacik, Schneidman, Berry~II, and
  Bialek]{tkacik2009spin}
Tkacik, G., Schneidman, E., Berry~II, M.~J., and Bialek, W.
\newblock Spin glass models for a network of real neurons.
\newblock \emph{arXiv preprint arXiv:0912.5409}, 2009.

\bibitem[Tyrcha et~al.(2013)Tyrcha, Roudi, Marsili, and
  Hertz]{tyrcha2013effect}
Tyrcha, J., Roudi, Y., Marsili, M., and Hertz, J.
\newblock The effect of nonstationarity on models inferred from neural data.
\newblock \emph{Journal of Statistical Mechanics: Theory and Experiment},
  2013\penalty0 (03):\penalty0 P03005, 2013.

\bibitem[Vuffray et~al.(2016)Vuffray, Misra, Lokhov, and
  Chertkov]{vuffray2016interaction}
Vuffray, M., Misra, S., Lokhov, A., and Chertkov, M.
\newblock Interaction screening: Efficient and sample-optimal learning of ising
  models.
\newblock In \emph{Advances in Neural Information Processing Systems}, pp.\
  2595--2603, 2016.

\bibitem[Vuffray et~al.(2019)Vuffray, Misra, and Lokhov]{vuffray2019efficient}
Vuffray, M., Misra, S., and Lokhov, A.~Y.
\newblock Efficient learning of discrete graphical models.
\newblock \emph{arXiv preprint arXiv:1902.00600}, 2019.

\bibitem[Wang et~al.(2013)Wang, Komodakis, and Paragios]{WANG20131610}
Wang, C., Komodakis, N., and Paragios, N.
\newblock Markov random field modeling, inference \& learning in computer
  vision \& image understanding: A survey.
\newblock \emph{Computer Vision and Image Understanding}, 117\penalty0
  (11):\penalty0 1610 -- 1627, 2013.
\newblock ISSN 1077-3142.
\newblock \doi{https://doi.org/10.1016/j.cviu.2013.07.004}.

\bibitem[Wei et~al.(2015)Wei, Iyer, and Bilmes]{wei2015submodularity}
Wei, K., Iyer, R., and Bilmes, J.
\newblock Submodularity in data subset selection and active learning.
\newblock In \emph{International Conference on Machine Learning}, pp.\
  1954--1963, 2015.

\bibitem[Zeng et~al.(2011)Zeng, Aurell, Alava, and Mahmoudi]{zeng2011network}
Zeng, H.-L., Aurell, E., Alava, M., and Mahmoudi, H.
\newblock Network inference using asynchronously updated kinetic ising model.
\newblock \emph{Physical Review E}, 83\penalty0 (4):\penalty0 041135, 2011.

\bibitem[Zeng et~al.(2013)Zeng, Alava, Aurell, Hertz, and
  Roudi]{zeng2013maximum}
Zeng, H.-L., Alava, M., Aurell, E., Hertz, J., and Roudi, Y.
\newblock Maximum likelihood reconstruction for ising models with asynchronous
  updates.
\newblock \emph{Physical review letters}, 110\penalty0 (21):\penalty0 210601,
  2013.

\bibitem[Zhang(2012)]{zhang2012inference}
Zhang, P.
\newblock Inference of kinetic ising model on sparse graphs.
\newblock \emph{Journal of Statistical Physics}, 148\penalty0 (3):\penalty0
  502--512, 2012.

\end{thebibliography}
